\documentclass[11pt,reqno]{amsart}
\usepackage{graphicx,verbatim,lineno,titletoc}
\usepackage{amssymb,mathrsfs}
\usepackage{amsmath}
\usepackage{adjustbox}
\usepackage{calc}
\usepackage{ytableau}
\usepackage{array}
\usepackage{tikz}
\usetikzlibrary{shapes.multipart,patterns,arrows}
\usepackage[font=footnotesize]{caption}
\usepackage[T1]{fontenc} 
\usepackage{fourier} 
\usepackage[english]{babel} 
\usepackage{appendix}
\usepackage[all]{xy}
\usepackage{enumerate}
\usepackage[english]{babel}
\usepackage{multirow}
\usepackage{float}
\usepackage{enumitem}
\usepackage{accents}
\usepackage[numbers]{natbib}
\usepackage{hyperref}
\hypersetup{colorlinks}
\usepackage{algorithm}
\usepackage[noend]{algpseudocode}
\usepackage{mathtools}
\mathtoolsset{showonlyrefs}
\usepackage{wrapfig}
\usepackage{blkarray}
\UseRawInputEncoding

\usepackage[top=1.0in, left=1.1in, right=1.1in, bottom=1.0in]{geometry}
\hypersetup{colorlinks,linkcolor={red},citecolor={olive},urlcolor={red}}

\newcolumntype{M}[1]{>{\centering\arraybackslash}m{#1}}
\newcolumntype{N}{@{}m{0pt}@{}}

\makeatletter
\newcommand{\vast}{\bBigg@{3}}
\newcommand{\Vast}{\bBigg@{5}}
\makeatother

\newtheorem{theorem}{Theorem}[section]

\newtheorem{definition}[theorem]{Definition}
\newtheorem{lemma}[theorem]{Lemma}

\theoremstyle{definition}
\newtheorem{remark}[theorem]{Remark}

\def\bZ{\mathbf{Z}}
\def\Q{\mathbf{Q}}
\def\R{\mathbb{R}}
\def\C{\mathbb{C}}
\def\E{\mathbb{E}}
\def\P{\mathbb{P}}

\def\L{\mathcal{L}}

\def\H{\mathbf{H}}

\def\eps{\varepsilon}

\def\X{\mathbf{X}}
\def\x{\mathbf{x}}
\def\v{\mathbf{v}}
\def\u{\mathbf{u}}
\def\W{\mathbf{W}}
\def\a{\mathbf{a}}
\def\A{\mathbf{A}}
\def\B{\mathbf{B}}
\def\C{\mathbf{C}}
\def\K{\mathbf{K}}
\def\var{\textup{Var}}

\def\param{\boldsymbol{\theta}}
\def\Param{\boldsymbol{\Theta}}
\def\I{\mathbf{I}}
\def\h{\mathbf{h}}
\def\bSigma{\boldsymbol{\Sigma}}

\def\V{\mathbf{V}}
\def\U{\mathbf{U}}

\def\bR{\mathbf{R}}
\def\w{\mathbf{w}}
\def\blambda{\boldsymbol{\lambda}}

\def\beps{\boldsymbol{\eps}}
\def\rank{\textup{rank}}
\def\bSigma{\boldsymbol{\Sigma}}
\def\bphi{\boldsymbol{\phi}}
\def\bpsi{\boldsymbol{\psi}}
\def\bPhi{\boldsymbol{\Phi}}

\def\L{\mathcal{L}}

\usepackage{pifont}

\definecolor{hancolor}{rgb}{0.1, 0.0, 0.9}

\DeclareMathOperator{\diag}{diag}

\DeclareMathOperator*{\argmin}{arg\,min}
\DeclareMathOperator{\vect}{vec}

\newcommand{\tr}{\textup{tr}}
\newcommand{\p}{\mathbf{p}}

\def\R{\mathbb{R}}
\def\E{\mathbb{E}}

\def\P{\mathbb{P}}

\def\L{\mathcal{L}}

\def\r{\mathtt{r}}
\def\H{\mathbf{H}}

\def\eps{\varepsilon}

\def\X{\mathbf{X}}
\def\x{\mathbf{x}}
\def\g{\mathbf{g}}
\def\W{\mathbf{W}}
\def\Y{\mathbf{Y}}
\def\bP{\mathbf{P}}
\def\H{\mathbf{H}}
\def\Beta{\boldsymbol{\beta}}
\def\bGamma{\boldsymbol{\Gamma}}
\def\A{\mathbf{A}}
\def\B{\mathbf{B}}
\def\var{\textup{Var}}

\def\y{\mathbf{y}}
\def\u{\mathbf{u}}

\newlength\myindent
\newtheorem{innercustomassumption}{}
\newenvironment{customassumption}[1]
{\renewcommand\theinnercustomassumption{#1}\innercustomassumption}
{\endinnercustomassumption}

\theoremstyle{definition}

\allowdisplaybreaks

\makeatletter
\newcommand{\addresseshere}{%
	\enddoc@text\let\enddoc@text\relax
}
\makeatother

\usepackage{color}
\def\blue{\color{blue}}
\def\boxit#1{\vbox{\hrule\hbox{\vrule\kern6pt\vbox{\kern6pt#1\kern6pt}\kern6pt\vrule}\hrule}}

\begin{document}

	\title[Supervised Dictionary Learning with Auxiliary Covariates ]{Supervised Dictionary Learning with Auxiliary Covariates}

	\author{Joowon Lee }
	\address{Joowon Lee, Department of Statistics, University of Wisconsin - Madison, WI 53709, USA}
	\email{\texttt{jlee2256@wisc.edu}}

	\author{Hanbaek Lyu}
	\address{Hanbaek Lyu, Department of Mathematics, University of Wisconsin - Madison, WI 53709, USA}
	\email{\texttt{hlyu@math.wisc.edu}}
	
	\author{Weixin Yao}
	\address{Weixin Yao, Department of Applied Statistics, University of California, Riverside, CA 92521, USA}
	\email{\texttt{weixin.yao@math.ucla.edu}}



	\begin{abstract}
		Supervised dictionary learning (SDL) is a classical machine learning method that simultaneously seeks feature extraction and classification tasks, which are not necessarily a priori aligned objectives. The goal of SDL is to learn a class-discriminative dictionary, which is a set of latent feature vectors that can well-explain both the features as well as labels of observed data. In this paper, we provide a systematic study of SDL, including the theory, algorithm, and applications of SDL. First, we provide a novel framework that `lifts' SDL as a convex problem in a combined factor space and propose a low-rank projected gradient descent algorithm that converges exponentially to the global minimizer of the objective. We also formulate generative models of SDL and provide global estimation guarantees of the true parameters depending on the hyperparameter regime. Second, viewed as a nonconvex constrained optimization problem, we provided an efficient block coordinate descent algorithm for SDL that is guaranteed to find an $\eps$-stationary point of the objective in $O(\eps^{-1}(\log \eps^{-1})^{2})$ iterations. For the corresponding generative model, we establish a novel non-asymptotic local consistency result for constrained and regularized maximum likelihood estimation problems, which may be of independent interest. Third, we apply SDL for imbalanced document classification by supervised topic modeling and also for pneumonia detection from chest X-ray images. We also provide simulation studies to demonstrate that SDL becomes more effective when there is a discrepancy between the best reconstructive and the best discriminative dictionaries.  
	\end{abstract}
	
	${}$
	\vspace{-0.5cm}
	${}$
	\maketitle

	\section{Introduction}

	Classification and feature extraction are arguably the two most fundamental tasks in machine learning and statistical inference problems. In classical classification models such as logistic regression, 
	the conditional class-generating probability distribution is modeled as a simple function of the observed feature with unknown parameters to be trained. However, the raw observed features may be high-dimensional and most of them might be uninformative and hard to interpret (e.g., pixel values of an image), so it would be desirable to extract more informative and interpretable low-dimensional features prior to the classification task. For instance, the multi-layer perception, or the feed-forward neural network in general \cite{bishop1995neural, bishop2006pattern}, uses additional feature extraction layers prior to the logistic regression layer so that the model itself learns the most effective feature extraction mechanism as well as the association of the extracted features with class labels at the same time. In this view, one may say that feed-forward neural nets perform supervised feature extraction. 
	
	A classical unsupervised feature extraction framework is called \textit{dictionary learning} (DL), a machine-learning technique that learns latent structures of complex data sets and is applied regularly in the data analysis of text and images \cite{elad2006image, mairal2007sparse, peyre2009sparse}. Various matrix factorization models provide fundamental tools for DL tasks such as singular value decomposition (SVD), principal component analysis (PCA), and nonnegative matrix factorization (NMF) \cite{golub1971singular, wall2003singular, abdi2010principal, lee1999learning}.  In particular, NMF seeks to approximately factorize a data matrix into the product of two nonnegative matrices, a dictionary matrix containing unknown features/basis and a coding matrix that provides a compressed representation of the data over that basis. Such an additive decomposition of data often results in highly interpretable features, which has been one of the main attractions of NMF as a fundamental tool for numerous applications such as topic modeling, image reconstruction, bioinformatics for protein-protein interaction networks, to name a few \cite{sitek2002correction, berry2005email, berry2007algorithms, chen2011phoenix, taslaman2012framework, boutchko2015clustering, ren2018non}. 
	
	There has been extensive research on making dictionary learning models adapted to also perform classification tasks by supervising the dictionary learning process using additional class labels.  Note that dictionary learning and classification are not necessarily aligned objectives, so some degree of trade-off is necessary when one seeks to achieve both goals at the same time. \textit{Supervised dictionary learning} (SDL) provides systematic approaches for such a multi-objective task. The general framework of SDL was introduced in  \cite{mairal2008supervised}.  A stochastic formulation of SDL was proposed as `task-driven dictionary learning' in \cite{mairal2011task}. A similar SDL-type framework of discriminative K-SVD was proposed for face recognition  \cite{zhang2010discriminative}. SDL has also found numerous applications in various other problem domains including speech and emotion recognition \cite{gangeh2014multiview}, music genre classification \cite{zhao2015supervised}, concurrent brain network inference \cite{zhao2015supervised}, and structure-aware clustering \cite{yankelevsky2017structure}. More recently, supervised variants of NMF, as well as PCA, were proposed in \cite{austin2018fully, leuschner2019supervised, ritchie2020supervised}. See also the survey \cite{gangeh2015supervised} on SDL.

	In spite the extensive literature on supervised dictionary learning, to our best knowledge, there is not much work on computational and statistical guarantees of algorithms and models of SDL. This is mainly due to the fact training an SDL model amounts to solving a nonconvex and possibly constrained optimization problem. In this paper, we provide extensive theoretical investigations of various SDL models and algorithms, establishing exponential convergence to global optimum, sublinear convergence to stationary points, and global and local estimation guarantee under the generative model assumption, depending on the hyperparameter regimes and the structure of the model. 
	In addition, we also investigate how to incorporate auxiliary information to the task of SDL, which is not studied in the literature as far as we know. We illustrate our results through various simulation and application experiments. One of our main applications is document classification for fake job postings by learning supervised topics as well as using auxiliary covariates such as the existence of company logos and websites.


	\subsection{Contribution} 
	
	In this paper, we provide a systematic study of supervised dictionary learning (SDL), including  theories, algorithms, and applications of SDL. 
	In addition, we also consider an extended SDL model where an auxiliary covariate can be used for improving classification accuracy. This extension is practically motivated for the cases where the input data for classification is of mixed-type in the sense that some part of it is high-dimensional and subject to reduced-dimensional feature extraction via dictionary learning, but some other part is already low-dimensional and can be used as a complementary covariate for improved classification performance. We consider two particular classes of such extended SDL models: 1) filter-based SDL models and 2) feature-based SDL models, depending on the type of reduced-dimensional covariates used for classification tasks, categorizing known SDL models in the literature. We provide extensive theoretical analysis for these two classes of extended SDL models, which we summarize below:
	\begin{description}[itemsep=0.1cm]
		\item[1.] (\textit{Convex approach for weakly constrained  SDL.}) If the SDL model parameters are unconstrained or weakly constrained (see \ref{assumption:A1}), then we show that we can `lift' the original nonconvex SDL problem into a convex problem in a larger space with low-rank constraints. We further propose a projected gradient descent (PGD) type algorithm for SDL operating in this larger space and establish that it converges exponentially fast to a global minimizer of the objective in an explicit hyperparameter regime (see Theorems \ref{thm:SDL_LPGD} and \ref{thm:SDL_LPGD_feat}). For the corresponding generative model, we obtain a strong statistical estimation guarantee in this case (see Theorems \ref{thm:SDL_LPGD_STAT} and \ref{thm:SDL_LPGD_STAT_feat}). 
		
		\item[2.] (\textit{Nonconvex approach for strongly constrained SDL.}) For the cases where the SDL problem cannot be lifted as a convex problem due to strong constraints (e.g., supervised NMF),  we propose an efficient block coordinate descent (BCD) algorithm for SDL that is guaranteed to find an $\eps$-stationary point (see Section \ref{sec:stationary_points} for definition) of the objective in $O(\eps^{-1} (\log \eps^{-1})^{2} )$ iterations (see Theorem \ref{thm:SDL_BCD}). For the corresponding generative model, we obtain a non-asymptotic local consistency result (see Theorem \ref{thm:SDL_BCD_STAT_filt}), which may be of an independent interest in other statistical estimation settings.
		
		\item[3.] (\textit{Comparison between filter-based and feature-based SDL.}) \, We find an interesting difference in theoretical stability of filter-based and feature-based SDL models, which has not been reported before. Namely, filter-based SDL may enjoy exponential convergence to the global optimum of the corresponding optimization problem without any additional $L_{2}$-regularization, but the feature-based SDL requires $L_{2}$-regularization. In a statistical estimation setting, this implies that the maximum likelihood estimator (MLE) for the generative feature-based SDL model can be computed exponentially fast but it may be a constant order away from the true parameter. However, generative filter-based SDL models admit $\sqrt{n}$-consistent MLE that can be computed exponentially fast. 
		
		\item[4.] (\textit{Applications and simulations}) We apply SDL as a supervised topic learning method and demonstrate how it learns topics that are relevant for document classification, where supervision works as a means of auto-correcting imbalance in datasets. Also, we use SDL for chest X-ray image analysis for pneumonia detection by learning latent shapes and their association with pneumonia. We also provide simulation studies to demonstrate that SDL becomes more effective when there is a significant discrepancy between the best reconstructive and the best discriminative dictionaries.  
	\end{description}

	\subsection{Related works}
	
	It is standard in the literature of SDL to propose an optimization or probabilistic framework of SDL model geared for some particular application, and derive an iterative optimization algorithm (mostly in the form of block coordinate descent, see, e.g., \cite{wright2015coordinate}) with some experimental results. However, convergence analysis or statistical estimation bounds often are missing in the existing literature. As we will discuss shortly, training an SDL model amounts to solving a nonconvex optimization problem possibly under some convex constraints on individual factors (parameters) of the model.  Moreover, even the special case of matrix factorization does not have a unique global minimizer. Such difficulties partly explain why SDL models and algorithms lack much theoretical analysis while enjoying numerous successful applications \cite{zhang2010discriminative, gangeh2014multiview,zhao2015supervised,yankelevsky2017structure}. However, we remark that Mairal et al. provided a rigorous justification of the differentiability of a feature-based SDL model formulated as a stochastic optimization problem     \cite{mairal2011task}, based on a similar analysis used for analyzing an online NMF algorithm \cite{mairal2010online}. 
	
	
	In this article, we propose both nonconvex and convex types of algorithms for training SDL and provide their convergence analysis and estimation properties. Algorithm \ref{algorithm:SDL} of former type is based on block coordinate descent with diminishing radius \cite{lyu2020convergence}. On the other hand, Algorithms \ref{alg:SDL_filt_LPGD} and \ref{alg:SDL_feat_LPGD} are special instances of the low-rank projected gradient descent in Algorithm \ref{algorithm:LPGD}, which is inspired by the singular value projection for low-rank matrix completion \cite{jain2010guaranteed}. Our convergence analysis of Algorithm \ref{algorithm:LPGD} is inspired by the analysis of an initialization algorithm for a low-rank matrix estimation problem in \cite{wang2017unified}. 
	
	In establishing Theorems \ref{thm:SDL_LPGD} and \ref{thm:SDL_LPGD_feat}. We use a `double-lifting' technique that converts a nonconvex SDL problem into a low-rank factored estimation and then into a  convex low-rank matrix estimation problem. This is reminiscent of the tight relation between a convex low-rank matrix estimation and a nonconvex factored estimation problem, which has been actively employed in a body of works in statistics and optimization \cite{agarwal2010fast, ravikumar2011high, negahban2011estimation, zheng2015convergent, tu2016low, wang2017unified,park2017non, park2018finding, tong2021accelerating}.

	One of our main results for non-asymptotic consistency of constrained and regularized MLE (Theorem \ref{thm:MLE_local_consistency}), which is a critical ingredient in establishing local consistency of SDL in the general case (Theorem \ref{thm:SDL_BCD_STAT_filt}), is inspired by the seminal work on local consistency guarantee for nonconcave penalized MLE in \cite{fan2001variable}.

	We consider both the constrained and unconstrained SDL, depending on whether we confine the dictionary matrix into an additional convex constraint set (e.g., nonnegative entries). The original SDL in \cite{mairal2008supervised} in this sense is an unconstrained  SDL and the supervised NMF in \cite{austin2018fully, leuschner2019supervised} belongs to a constrained SDL. The supervised PCA in \cite{ritchie2020supervised} uses the nonconvex (Grassmannian) constraint on the dictionary, which we do not consider in this present work.

	\subsection{Notations}
	\label{subsection:notation}
	
	Throughout this paper, we denote by $\R^{p}$ the ambient space for data equipped with standard inner project $\langle\cdot, \cdot \rangle$ that induces the Euclidean norm $\lVert \cdot \rVert$. We denote by  $\{ 0,1,\dots,\kappa \}$ the space of class labels with $\kappa+1$ classes.   For a convex subset $\Param$ in a Euclidean space, we denote $\Pi_{\Param}$ the projection operator onto $\Param$. For an integer $r\ge 1$, we denote by $\Pi_{r}$ the rank-$r$ projection operator for matrices. More precisely, for $\x\in \R^{p}$ and $\X\in \R^{m\times n}$, 
	\begin{align}
		\Pi_{\Param}(\x) \in \argmin_{\x'\in \Param} \lVert \x'-\x \rVert,\, 	\qquad \Pi_{r}(\X) \in \argmin_{\X'\in \R^{m\times n} ,\, \rank(\X')\le r} \lVert \X'-\X \rVert_{F}.
	\end{align}
	For a matrix $\A=(a_{ij})_{ij}\in \R^{m\times n}$, we denote its Frobenius, operator (2-), and supremum norm by 
	\begin{align}
		\lVert \A \rVert_{F} := \left(\sum_{1\le i\le m,\, 1\le j \le n} a_{ij}^{2}\right)^{1/2}, \quad \lVert \A \rVert_{2} := \sup_{\x\in \R^{n},\, \lVert \x \rVert=1} \, \lVert \A\x \rVert, \qquad \lVert \A \rVert_{\infty}:= \max_{1\le i\le m, \, 1\le j \le n } |a_{ij}|, 
	\end{align}
	respectively. For each $1\le i \le m$ and $1\le j \le n$, we denote $\A[i,:]$ and $\A[:,j]$ for the $i$th row and the $j$th column of $\A$, respectively (adopting python notation). For each integer $n\ge 1$, $\I_{n}$ denotes the $n\times n$ identity matrix. For square symmetric matrices  $\A,\B\in \R^{n\times n}$, we denote $\A\preceq  \B$ if $\v^{T}\A\v \le \v^{T}\B\v $ for all unit vectors $\v\in \R^{n}$.	For two elements $\bZ=[\X,\bGamma]$ and  $\bZ'=[\X', \bGamma']$ in the product space $\R^{d_{1}\times d_{2}}\times \R^{d_{3}\times d_{4}}$, we define their Frobenius distance  $\lVert \bZ - \bZ' \rVert_{F}$ by
	\begin{align}\label{eq:def_F_distance_gen}
		\lVert \bZ - \bZ' \rVert_{F}^{2} := \lVert \vect(\bZ) - \vect(\bZ') \rVert_{2}^{2}= \lVert  \X- \X'\rVert_{F}^{2} + \lVert \bGamma-\bGamma' \rVert_{F}^{2},
	\end{align}
	where $\vect(\cdot)$ is a vectorization operator that maps $\R^{d_{1}\times d_{2}}\times \R^{d_{3}\times d_{4}}$ to $\R^{d_{1}d_{2}+d_{3}d_{4}}$ by an arbitrary but fixed ordering of coordinates. We say 
	$$y\sim \textup{Multinomial}(1, (p_0,\ldots,p_\kappa ))$$ if $y$ can only take values from $0,1,\ldots,\kappa$ with $P(Y=j)=p_j,j=0,\ldots,\kappa,$ where $\sum_{j=0}^\kappa p_j=1$.
	
	\section{Problem formulation and background}

	\subsection{Supervised Dictionary Learning}\label{subsection:SDL1}
	
	Suppose we are given with $n$ labeled signals $(\x_{i}, y_{i})$ for $i=1,\dots, n$, where $\x_{i}\in \R^{p}$ is a $p$-dimensional signal and $y_{i} \in \{ 0,1,\dots,\kappa \}$ is its label, where $\kappa\ge 1$ is a fixed integer.  In the classical \textit{dictionary learning} (DL) literature \cite{mairal2010online, mairal2013optimization,mairal2013stochastic}, one seeks to find a dictionary $\W=[\w_{1},\dots,\w_{r}]\in \R^{p\times r}$ ($r\ll p$) that is \textit{reconstructive} in the sense that the observed signals $\x_{i}$ can be effectively reconstructed as (or approximated by)  the linear transform $\W\h_{i}$ of the `atoms' $\w_{1},\dots,\w_{r}\in \R^{p}$ for some suitable (sparse) `code' $\h_{i}\in \R^{r}$. However, the best reconstructive dictionary $\W$ may not be very effective for the classification tasks. In the \textit{supervised dictionary learning} (SDL) literature 	\cite{mairal2008supervised}, one desires dictionary that is reconstructive as well as \textit{discriminative} in that such a compressed representation of signals is adapted to predicting the class labels $y_{i}$. 
	
	More precisely, consider the following probability distribution $\g(\a)=(g_{0}(\a),\dots,g_{\kappa}(\a))$ on $\{0,1,\dots,\kappa\}$ with \textit{activation} $\a=(a_{1},\dots,a_{\kappa})$ given by 
	\begin{align}\label{eq:prediction_model_g}
		g_{0}(\a) = \frac{1}{1+\sum_{c=1}^{\kappa} h(a_{c})  },\qquad 	g_{j}(\a) = \frac{h(a_{j})}{1+\sum_{c=1}^{\kappa} h( a_{c} )} \quad \textup{for $j=1,\dots,\kappa$},
	\end{align}
	where $h:\R\rightarrow [0,\infty)$ is a fixed \textit{score function}. For instance, taking $h(\cdot)=\exp(\cdot)$ results in multinomial logistic regression (see Section \ref{sec:MLR} in the appendix for more details). We then model the given training data $(\x_{i}, y_{i})$ as 
	\begin{align}\label{eq:SDL_model}
		\x_{i} = \W\h_{i} \quad \text{and} \quad y_{i}\,|\, \x_{i}\sim \textup{Multinomial}(1, \g(\a_i  )) 
		\quad \text{for $i=1,\dots,n$},
	\end{align}
	where we allow the activation $\a_i$ to depend on the signal $\x_{i}$, latent factors $\W$ and $\h_{i}$, and an additional model parameter $\Beta$ through some functional relation.  
	
	

	As we seek to balance the tasks of dictionary learning and classification, the objective of SDL can naturally be formulated as a multi-objective optimization problem as below:
	\begin{align}\label{eq:SDL_1}
		&\min_{\W, \H, \Beta} \quad  {L(\W,\H,\Beta):=\left( \sum_{i=1}^{n} \ell(y_{i}, \g( \a(\x_{i},\W,\h_{i},\Beta))) \right)  +   \xi \lVert \X_{\textup{data}} - \W\H\rVert_{F}^{2}} \\
		&\textup{subject to:} \quad \textup{Constraints on $\W\in \R^{p\times r}$, $\H\in \R^{r\times n}$, and $\Beta\in \R^{r\times \kappa}$}
	\end{align}
	where $\X_{\textup{data}}=[\x_{1},\dots,\x_{n}]\in \R^{p\times n}$, $\H=[\h_{1},\dots,\h_{n}]\in \R^{r\times n}$, and $\ell(\cdot)$ is a classification loss and is usually taken as	 the negative log likelihood 
	\begin{align}\label{eq:ell_log_likelihood}
		\ell(y_{i},\g( \a(\x_{i},\W,\h_{i},\Beta)))
		:=-\sum_{j=0}^{\kappa} \mathbf{1}(y_{i}=j)  \log \left\{g_j( \a(\x_{i}, \W, \h_{i}, \Beta))\right\}.
	\end{align}
	Here, the \textit{tuning parameter} $\xi$ controls the trade-off between the two objectives of classification and dictionary learning. 	 We allow to put desired constraints on the parameters $\{\W,\H,\Beta\}$. In particular, we will consider nonnegativity constraints on $\W$ and $\H$ as in the supervised nonnegative matrix factorization (SNMF) model \cite{austin2018fully, leuschner2019supervised} to enjoy the nice interpretability of NMF in the supervised setting. 
	
	Various models of the form \eqref{eq:SDL_1} have been proposed in the past two decades. We divide them into two categories, depending on whether the classification model $\g$ is either `feature-based' or `filter-based'. The classification function $\g$ in \textit{feature-based SDL} (SDL-feat) makes use of the code $\h_{i}$, which is a $r$-dimensional feature of $\x_{i}$ extracted by the dictionary $\W$. On the other hand, \textit{filter-based SDL} (SDL-filt) uses the $r$-dimensional filtered input $\W^{T}\x_{i}$ instead of the code $\h_{i}$ in the classification/prediction step. In this work, we consider the following two types of multinomial prediction models: 
	\begin{description}
		\item[\quad (SDL-feat) ] Feature-based classification: $\a(\x_{i},\W,\h_{i},\Beta)=\Beta^{T}\h_{i}$;
		\vspace{0.1cm}
		\item[\quad (SDL-filt)] Filter-based classification: $\a(\x_{i},\W,\h_{i},\Beta)= \Beta^{T} \W^{T}\x_{i}$.
	\end{description} 
	Feature-based models include the classical ones by Mairal et al. (see, e.g., \cite{mairal2008supervised, mairal2011task}) as well as the more recent model of 
	Convolutional Matrix Factorization by Kim et al. \cite{kim2016convolutional} for a contextual text recommendation system. One of the downsides of SDL-feat for classification tasks is that for a new test signal $\x$, its correct code representation $\h$ may need to be learned in a supervised fashion by using an unknown true label $y$ of $\x$. Since $y$ can assume $\kappa+1$ different class labels, one can solve $\kappa$ instances of `supervised sparse coding' to make a prediction for test signals. 
	\begin{figure}[h!]
		\centering
		\includegraphics[width=0.65\linewidth]{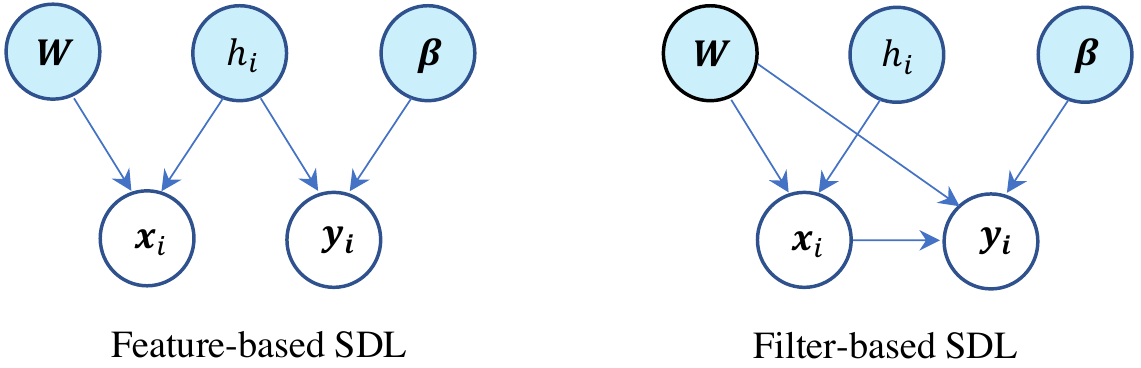} 
		\caption{Graphical models for the feature-based and the filter-based SDL. $\x_{i}$ and $\y_{i}$ denote the feature and the label of the $i$th training data, whereas $W$ denotes $p\times r$ dictionary matrix, $h_{i}$ denotes $r\times 1$ code of data $\x_{i}$, and $\beta_{i}$ denotes parameters for classification.}
		\label{fig:graphical_models}
	\end{figure}
	
	On the other hand, filter-based models have been studied more recently in the supervised matrix factorization literature, most notably from supervised nonnegative matrix factorization \cite{austin2018fully, leuschner2019supervised} and supervised PCA \cite{ritchie2020supervised}. Compared to the prediction step in SDL-feat, the pipeline for the filter-based models is more streamlined, as there is no need to compute supervised sparse code $\h_{i}$ as before. This is because the model now learns to predict directly from the feature-extraction filter $\W$, rather than from the extracted and possibly supervised feature $\h_{i}$. 
	
	\subsection{SDL with auxiliary variable} 
	
	Consider the case where we have additional covariate data $\X_{\textup{aux}}=[\x'_{1},\dots,\x'_{n}]\in \R^{q\times n}$ along with the original data $\X_{\textup{data}}=[\x_{1},\dots,\x_{n}]\in \R^{p\times n}$ (assume $q\ll p$) and labels  $\Y_{\textup{label}}=(y_{1},\dots, y_{n}) \in \{ 0,1,\dots,\kappa \}^{n} $. While $\X_{\textup{data}}$ is subject to a dimension reduction by dictionary learning, $\X_{\textup{aux}}$ will only be used as an auxiliary information for the classification task and is usually low dimensional with possibly discrete variables. In this case, we propose extending the SDL model \eqref{eq:SDL_1} with the types of multi-class classification model specified: 
	\begin{align}\label{eq:ASDL_1}
		&\min_{\W, \H, \Beta,\bGamma} \quad  L(\W,\H,\Beta,\bGamma):=\left(  -\sum_{i=1}^{n} \sum_{j=0}^{\kappa}  \mathbf{1}(y_{i}=j) \log g_{j}( \a(\x_{i},\x'_{i}, \W, \h_{i}, \Beta,\bGamma)) \right)  +   \xi \lVert \X_{\textup{data}} - \W\H\rVert_{F}^{2}  \\
		&\hspace{2cm} \a = \a(\x_{i},\x'_{i}, \W, \h_{i}, \Beta,\bGamma) = \begin{cases}
			\Beta^{T} \h_{i} + \bGamma^{T} \x'_{i} & \textup{feature-based}\\
			\Beta^{T} \W^{T} \x_{i} + \bGamma^{T} \x'_{i} & \textup{filter-based}
		\end{cases}
		\in \R^{\kappa} \\
		&\textup{subject to:} \quad \textup{Constraints on}\,\, \W\in \R^{p\times r}, \H\in \R^{r\times n}, \Beta\in \R^{r\times \kappa }, \bGamma\in \R^{q\times \kappa}.
	\end{align}
	The first term in the bracket in the right hand side of \eqref{eq:ASDL_1} equals the negative log likelihood of observing labels $(y_{1},\dots,y_{n})$ given the input $(\X_{\textup{data}}, \X_{\textup{aux}})$. 
	
	
	Note that when predicting $\y_{i}$, the auxiliary covariate $\x_{i}'$ together with the corresponding auxiliary coefficient $\bGamma$ is used, but the dictionary learning part is unchanged compared to the existing SDL model. For a vivid context, think of $\x_{i}$ as the X-ray image of a patient and $\x_{i}'$ denoting some biological measurement, gender, smoking status, and body mass index (BMI). While it may be desired to compress the image $\x_{i}$ and extract reconstructive and discriminative dictionary atoms from it, it would be more natural to use the additional covariate $\x_{i}'$ as-is for the prediction purpose. 
	
	
	\subsection{Constrained and Augmented Low-rank Estimation}
	
	Next, we introduce another problem class that turns out to be very closely related to SDL \eqref{eq:ASDL_1}, albeit the connection may not seem obvious at first glance. Fix a  function $f:\R^{d_{1}\times d_{2}}\times \R^{d_{3}\times d_{4}}\rightarrow \R$, which takes the input of a $d_{1}\times d_{2}$ matrix and an augmented variable in $\R^{d_{3}\times d_{4}}$.  Consider the following  \textit{constrained and augmented   low-rank estimation}  (CALE) problem 
	\begin{align}\label{eq:CALE}
		\hspace{-3cm} (\textbf{CALE}) \hspace{2cm}	&\min_{\bZ=[\X, \bGamma] \in \subseteq  \R^{d_{1}\times d_{2}} \times \R^{d_{3}\times d_{4}}} \,f(\bZ),\qquad \textup{subject to $\bZ\in \Param$ and $\rank(\X)\le r$},
	\end{align}
	where $\Param$ is a convex subset of $\R^{d_{1}\times d_{2}}\times \R^{d_{3}\times d_{4}}$. Here, we seek to find a global minimizer $\bZ^{\star}=[\X^{\star}, \bGamma^{\star}]$ of the objective function $f$ over the convex set $\Param$, consisting of a low-rank matrix component $\X^{\star}\in \R^{d_{1}\times d_{2}}$ and an auxiliary variable $\bGamma^{\star}\in \R^{d_{3}\times d_{4}}$. In a statistical inference setting, the loss function $f=f_{n}$ may be based on $n$ noisy observations according to a probabilistic model, and the true parameter $\bZ^{*}$ to be estimated may approximately minimize $f$ over the constraint set $\Param$, with some statistical error $\eps(n)$ depending on the sample size $n$. In this case, a global minimizer $\bZ^{\star}\in \argmin_{\Param} f$ serves as an estimate of the true parameter $\bZ^{*}$.  
	The matrix completion and low-rank matrix estimation problem \cite{meka2009guaranteed, recht2010guaranteed} can be considered as special cases of \eqref{eq:CALE} without constraint $\Param$ and the auxiliary variable $\bGamma$. 
	This problem setting has been one of the most important research topics in the machine learning and statistics literature for the past few decades. 
	
	On the other hand, one can reformulate \eqref{eq:CALE} as the following nonconvex problem, where one parameterizes the low-rank matrix variable $\X$ with product $\U\V^{T}$ of two matrices, which we call the \textit{constrained and augmented factored estimation} (CAFE) problem:
	\begin{align}\label{eq:CAFE}
		\hspace{-3.3cm} (\textbf{CAFE}) \hspace{2cm}	&\min_{ \U\in \R^{d_{1}\times r}, \V\in \R^{d_{2}\times r}, \bGamma\in \R^{d_{3}\times d_{4}} } \,f(\U\V^{T}, \bGamma) ,\qquad \textup{subject to $[\U\V^{T},\bGamma]\in \Param$}.
	\end{align}
	Note that a solution to \eqref{eq:CAFE} gives a solution to \eqref{eq:CALE}. Conversely, for \eqref{eq:CALE} without constraint on the first matrix component, singular value decomposition of the first matrix component easily shows that a solution to \eqref{eq:CALE} is also a solution to \eqref{eq:CAFE}. Recently, there has been a surge of progress in global guarantees of solving the factored problem \eqref{eq:CAFE} using various nonconvex optimization methods \cite{jain2010guaranteed, jain2013low, zhao2015nonconvex, zheng2015convergent, tu2016low, park2017non, wang2017unified, park2016provable, park2018finding}. Most of the work considers \eqref{eq:CAFE} without the auxiliary variable and constraints, some with a particular type of constraints (e.g., matrix norm bound), but not general convex constraints. 
	
	It is common that the non-convex factored problem \eqref{eq:CAFE} is introduced as a more efficient formulation of the convex problem \eqref{eq:CALE}. Interestingly, in the present work, we will reformulate the four-factor nonconvex problem of SDL in \eqref{eq:ASDL_1} as a three-factor nonconvex CAFE problem in \eqref{eq:CAFE} and then realize it as a single-factor convex CALE problem in \eqref{eq:CALE}. We illustrate this connection briefly in the following section and in more detail in Section \ref{subsection:SDL_convex_alg}.

	\subsection{A preliminary connection with CALE and SDL}
	\label{subsection:SDL_CALE_prelim}
	
	In this subsection, we give a preliminary discussion on how SDL problems can be formulated as a CALE problem. For simplicity,  we consider the following linear regression version of SDL, where we seek to solve  matrix factorization and linear regression problems simultaneously for data matrix $\X_{\textup{data}}\in \R^{p\times n}$ and response variable $\Y_{\textup{label}}\in \R^{1\times n}$:
	\begin{align}\label{eq:SDL_regression_H1}
		\min_{\W\in \R^{p\times r},\,  \Beta\in \R^{r\times 1} ,\, \H\in \R^{r\times n} } \lVert \Y_{\textup{label}} - \Beta^{T}\H  \rVert_{F}^{2} + \xi  \lVert \X_{\textup{data}} - \W \H  \rVert_{F}^{2}.
	\end{align}
	As in the SDL problem \eqref{eq:SDL_1}, this is a three-block optimization problem involving three factors $\W,\H,$ and $\Beta$. However, by suitably stacking up the matrices, we can reformulate it as the following single matrix factorization problem, which is an instance of the CAFE problem \eqref{eq:CAFE}:
	\begin{align}\label{eq:SDL_regression_H2}
		\min_{\W\in \R^{p\times r},\,  \Beta\in \R^{r\times 1} ,\, \H\in \R^{r\times n} }
		\left[ f\left( \begin{bmatrix} \Beta^{T} \\  \W  \end{bmatrix} \H\right)
		:=
		\left\lVert 
		\begin{bmatrix}
			\Y_{\textup{label}} \\
			\sqrt{\xi} \X_{\textup{data}}
		\end{bmatrix}
		- 
		\begin{bmatrix}
			\Beta^{T}\\
			\sqrt{\xi} 	\W 
		\end{bmatrix}
		\H \right\rVert_{F}^{2}  
		\right].
	\end{align}
	Indeed, we now seek to find \textit{two} decoupled matrices (instead of three), one for $\Beta^{T}$ and $\W$ stacked vertically, and the other for $\H$. The same idea of matrix stacking was used in 
	\cite{zhang2010discriminative} for discriminative K-SVD. Proceeding one step further, another important observation is that it is also equivalent to finding a \textit{single} matrix $\bZ:=\begin{bmatrix} \Beta^{T}\H \\ \W\H \end{bmatrix}\in \R^{(1+p)\times n}$ of rank  at most $r$ that minimizes the function $f$ in  \eqref{eq:SDL_regression_H2}. Thus we can view \eqref{eq:SDL_regression_H1} as a low-rank matrix estimation problem, a special case of CALE \eqref{eq:CALE}. This simple yet instructive example illustrates our two-step lifting strategy for analyzing SDL problems. 
	
	In view of the discussion in Subsection \ref{subsection:SDL1}, \eqref{eq:SDL_regression_H1} can be regarded as using a feature-based regression model.  An analogous filter-based regression model would be the following: 
	\begin{align}\label{eq:SDL_regression_W1}
		\min_{\W\in \R^{p\times r},\,  \Beta\in \R^{r\times 1} ,\, \H\in \R^{r\times n} } \left[  f\left( \W \begin{bmatrix} \Beta,\,  \H  \end{bmatrix} \right):= \lVert \Y_{\textup{label}} -  \Beta^{T}  \W^{T}\X_{\textup{data}} \rVert_{F}^{2} + \xi  \lVert \X_{\textup{data}} - \W \H  \rVert_{F}^{2}\right].
	\end{align}
	Here, matrix stacking as in \eqref{eq:SDL_regression_H2} is not available. However, a simple but important observation we make is that the objective in the right hand side of \eqref{eq:SDL_regression_W1} depends only on the product $\W[\Beta, \H]$  
	and hence we can still view it as an instance of CAFE problem \eqref{eq:CAFE}. Then we may further lift it as a CALE problem \eqref{eq:CALE}, seeking a single matrix $\bZ:=[\W\Beta,\, \W\H]\in \R^{p\times (1+n)}$ of rank at most $r$ that solves \eqref{eq:SDL_regression_W1}. This observation will be used crucially to double-lifting SDL problems \eqref{eq:ASDL_1} as a CALE problem in section \ref{subsection:SDL_convex_alg}.
	
	\section{Algorithms}

	\subsection{Convex algorithms for weakly constrained SDL}
	\label{subsection:SDL_convex_alg}
	
	In this subsection, we consider an instance of SDL \eqref{eq:ASDL_1} when it can be converted to the CALE formulation \eqref{eq:CALE}, and then propose convex algorithms to effectively find global optimality. Inspired by the  observation in Subsection \ref{subsection:SDL_CALE_prelim}, we rewrite the objective function of the filter-based aSDL model \eqref{eq:ASDL_1}  as the following CAFE \eqref{eq:CAFE}  problem: 
	{\small
		\begin{align}\label{eq:SDL_filt_1}
			&			\hspace{-0.5cm}\min_{\W,\H,\Beta,\bGamma } \,\, \left[  f_{\textup{SDL-filt}}\left( \W [\Beta \,\, \H],\, \bGamma
			\right):=    \left( \sum_{i=1}^{n}  \ell(y_{i},\g(\Beta^{T}\W^{T}\x_{i} +\bGamma^{T}\x_{i}')) \right) +   \xi \lVert \X_{\textup{data}} - \W\H\rVert_{F}^{2} + \nu \left( \lVert \Beta \W^{T} \rVert_{F}^{2} + \lVert \bGamma \rVert_{F}^{2} \right)    \right]\\
			&\textup{subject to:}\quad \textup{Constraints on $\W\in \R^{p\times r}$, $\H\in \R^{r\times n}$, $\Beta\in \R^{r \times \kappa}$, $\bGamma\in \R^{q\times \kappa}$}.
		\end{align}
	}
	Note that we have added a $L_{2}$-regularization term for $\W\Beta $ and $ \bGamma$ with coefficient $\nu\ge 0$. This term will play a crucial role in well-conditioning \eqref{eq:SDL_filt_1}.  As before, it is important to notice that the objective function in \eqref{eq:SDL_filt_1} depends only on the products $\W\Beta$ and $\W\H$ as well as the auxiliary variable $\bGamma$. By stacking the matrices $\W\Beta$ and $\W\H$ and imposing a low-rank constraint on the stacked matrix, we can reformulate \eqref{eq:SDL_filt_1} as a CALE problem \eqref{eq:CALE} as below:
	\begin{align}\label{eq:SDL_filt_CALE}
		&\min_{\A,\B, \bGamma } \,\, \left[  f_{\textup{SDL-filt}}	\left([\A, \, \B],\, \bGamma \right)
		=   \left( \sum_{i=1}^{n} \ell(y_{i}, \g(\A^{T} \x_{i} + \bGamma^{T}\x_{i}' )) \right) +   \xi \lVert  \X_{\textup{data}}  -\B\rVert_{F}^{2} + \nu \left( \lVert \A \rVert_{F} + \lVert \bGamma \rVert_{F}  \right)^{2}    \right] \\
		&\quad \textup{subject to:} \quad \textup{Constraints on $\X:=[\A,\B]\in \R^{p\times (\kappa+n)}$, $\bGamma\in \R^{q\times \kappa}$, and $\rank\left( \X \right) \le r$}.
	\end{align}
	Note that the CALE formulation \eqref{eq:SDL_filt_CALE} assumes that the constraints we use in \eqref{eq:SDL_filt_1} for $\W,\H,\Beta$ can be translated into a constraint on the low-rank matrix $\X$ in \eqref{eq:SDL_filt_CALE}. We can such constraints `\textit{weak constraints}', which we formally introduce below:
	\begin{customassumption}{A1}(Weakly constrained SDL)\label{assumption:A1} 
		The constraints on $[\W,\H,\Beta,\bGamma]$ in \eqref{eq:SDL_filt_1} are `weak' in the sense that it can be written as a convex constraint $\Param\subseteq \R^{p\times (\kappa+n)}\times \R^{q\times \kappa}$  on $(\X=[\W\Beta^{T}, \W\H],\bGamma)$ in \eqref{eq:SDL_filt_CALE}. 
		Similarly, the constraints on $[\W,\H,\Beta,\bGamma]$ in \eqref{eq:SDL_feat_1} can be written as a convex constraint $\Param\subseteq \R^{(\kappa+p)\times n}\times \R^{q\times \kappa}$  on $\left(\X=\begin{bmatrix}\Beta \H \\  \W\H\end{bmatrix},\bGamma\right)$ in \eqref{eq:SDL_feat_CALE}. 
	\end{customassumption}
	\noindent In particular, when $\Param$ in \ref{assumption:A1} equals the whole space, we call the corresponding SDL problem \eqref{eq:ASDL_1} \textit{unconstrained}. If the constraints on the parameters are not weak in the sense of \ref{assumption:A1} (e.g., nonnegativity on $\W$ and $\H$), we call the corresponding SDL problem \textit{strongly constrained} and 
	will need to directly solve the nonconvex formulation \eqref{eq:SDL_filt_1} using Algorithm \ref{algorithm:SDL}.

	In order to solve \eqref{eq:SDL_filt_CALE}, we propose a projected gradient descent (PGD) type algorithm, inspired by the singular value projection in \cite{jain2010guaranteed} as well as the initialization algorithm in \cite{wang2017unified}. Namely, we iterate gradient descent followed by projecting onto the convex constraint set of the combined factor $\mathbf{X}=[\A,\B]$ and then perform rank-$r$  projection via truncated SVD until convergence. (See \eqref{eq:LRPGD_iterate0} and Algorithm \ref{algorithm:LPGD}). Once we have a solution $[\X^{\star},\bGamma^{\star}]$ to \eqref{eq:SDL_filt_CALE}, we can use SVD of $\X^{\star}$ to obtain a solution to \eqref{eq:SDL_filt_1}. Namely,  let $\X^{\star} = \Q_{\U} \bSigma \Q_{\V}^{T}$ denote the SVD of $\X$. Since $\rank(\X^{\star})\le r$, we may assume that $\bSigma$ is an $r\times r$ diagonal matrix of singular values of $\X$. Then $\Q_{\U}\in \R^{m\times r}$ and $\Q_{\V}\in \R^{n\times r}$ are semi-orthonormal matrices, that is, $\Q_{\U}^{T}\Q_{\U} = \Q_{\V}^{T}\Q_{\V}=\I_{r}$. Then $\X^{\star} = \U\V^{T}$ where 	$\U:= \Q_{\U} \bSigma^{1/2}$ and  $\V := \Q_{\V} \bSigma^{1/2}$. Consequently, we can take $\W^{\star}=\U$ and $\begin{bmatrix} (\Beta^{\star})^{T}, \,\, \H^{\star} \end{bmatrix} = \V$. Then $[\W^{\star}, \H^{\star}, \Beta^{\star}, \bGamma^{\star}]$ is a solution to \eqref{eq:SDL_filt_1} under the compatibility of constraints stated in \eqref{assumption:A1}. We summarize this CALE approach of solving \eqref{eq:SDL_filt_1} in the following algorithm. 
	Below, $\textup{SVD}_{r}$ denotes rank-$r$ truncated SVD and the projection operators $\Pi_{\Param}$ and $\Pi_{r}$ are defined in Subsection \ref{subsection:notation}.
	\begin{algorithm}[H]
		\caption{SDL-conv-filt}
		\label{alg:SDL_filt_LPGD}
		\begin{algorithmic}[1]
			\State \textbf{Input:} $\X_{\textup{data}}\in \R^{p\times n}$ (Data matrix);\,\,$\X'_{\textup{aux}}\in \R^{q\times n}$ (Auxiliary covariate matrix);\,\, $\Y_{\textup{label}}\in \{0,\dots,\kappa\}^{1 \times n}$ (Label matrix)
			\vspace{0.1cm}
			\State \textbf{Constraints}: Convex set $\Param\subseteq \R^{p\times (\kappa+n)}$ 
			\State \textbf{Parameters}: $\tau>0$ (Stepsize parameter);\,\, $N\in \mathbb{N}$ (number of iterations); $r\ge 1$ (rank parameter)
			\State \qquad  Initialize $\A_{0}\in \R^{p\times \kappa}$, $\B_{0}\in \R^{p\times n}$, $\bGamma_{0}\in \R^{q\times \kappa}$
			\State \qquad \textbf{For $k=1,2,\dots,N$ do:}
			\State \qquad \qquad  $\displaystyle
			\begin{bmatrix}
				\A_{k},\,\, \B_{k},\,\,\bGamma_{k} 
			\end{bmatrix}
			\leftarrow \Pi_{\Param \times \R^{q\times \kappa}}\left(	\begin{bmatrix}
				\A_{k-1},\,\,
				\B_{k-1},\,\,
				\bGamma_{k-1}
			\end{bmatrix}
			- \tau \nabla f_{\textup{SDL-filt}}\left( 		\begin{bmatrix}
				\A_{k-1},\,\, \B_{k-1},\,\, \bGamma_{k-1}
			\end{bmatrix} \right)\right)
			$
			\State\qquad \qquad $\displaystyle 	\begin{bmatrix}
				\A_{k},\,\,
				\B_{k}
			\end{bmatrix} \leftarrow \Pi_{r}\left(  	\begin{bmatrix}
				\A_{k} ,\,\, 
				\B_{k}
			\end{bmatrix}  \right)$	  \qquad ($\triangleright$ rank-$r$ projection)
			\State \qquad\textbf{End for} 
			\State \qquad $[ \U, \bSigma, \V ] \leftarrow \textup{SVD}_{r}\left(\begin{bmatrix}
				\A_{N},\,\,
				\B_{N}
			\end{bmatrix} \right)$ \qquad ($\triangleright$ rank-$r$ SVD)
			\State \qquad $\displaystyle \W_{N} \leftarrow \U \bSigma^{1/2}$, $[\Beta_{N},\,\, \H_{N}] \leftarrow(\bSigma)^{1/2}\V^{T}$
			
			\State \textbf{Output:} $(\W_{N}, \H_{N}, \Beta_{N}, \bGamma_{N})$ and  $(\A_{N}, \B_{N}, \bGamma_{N} )$ 
		\end{algorithmic}
	\end{algorithm}	

Similarly, we can write the objective function of the feature-based augmented SDL model \eqref{eq:ASDL_1} as the following CAFE \eqref{eq:CAFE} problem:
	{\small
		\begin{align}\label{eq:SDL_feat_1}
			&\hspace{-0.4cm}\min_{\W,\H,\Beta,\bGamma } \,\, \left[  f_{\textup{SDL-feat}}\left( 	\begin{bmatrix}
				\Beta^{T}\\
				\W 
			\end{bmatrix}
			\H,\, \bGamma  \right):=    \left( \sum_{i=1}^{n} \ell(y_{i}, \g(\Beta^{T} \h_{i} +\bGamma^{T} \x'_{i} )) \right) +  \xi \lVert  \X_{\textup{data}}  - \W\H\rVert_{F}^{2} + \nu\left( \lVert  \Beta \H \rVert_{F}^{2} + \lVert \bGamma \rVert_{F}^{2}\right)    \right], \\
			&\textup{subject to:}\quad \textup{Constraints on $\W\in \R^{p\times r}$, $\H\in \R^{r\times n}$, $\Beta\in \R^{r\times \kappa}$, $\bGamma\in \R^{q\times \kappa}$}.
		\end{align}
	}
We can reformulate \eqref{eq:SDL_feat_1} as the following CALE problem:
	\begin{align}\label{eq:SDL_feat_CALE}
		&\hspace{-0.5cm}	\min_{\A,\B, \bGamma } \,\, \left[  f_{\textup{SDL-feat}}\left( \X:=	\begin{bmatrix}
			\A\\
			\B
		\end{bmatrix} ,\, \bGamma \right)
		=   \left( \sum_{i=1}^{n} \ell(y_{i},g(\A_{i} + \bGamma^{T}\x_{i}' ) )\right) +   \xi \lVert  \X_{\textup{data}}  -\B\rVert_{F}^{2} +  \nu \left(\lVert \A \rVert_{F}^{2} + \lVert \bGamma \rVert_{F}^{2} \right) \right] \\
		&\quad \textup{subject to:} \quad \textup{Constraints on $\X\in \R^{(\kappa+p)\times n}$, $\bGamma\in \R^{q\times \kappa}$, and $\rank\left( \X \right) \le r$}.
	\end{align}
	As before, this CALE formulation assumes the compatibility of constraints in the two settings (see the weak constraints condition \ref{assumption:A1}), and solutions to \eqref{eq:SDL_feat_CALE} can be transformed into solutions to \eqref{eq:SDL_feat_1} by using SVD. The analog of Algorithm \ref{alg:SDL_feat_LPGD} is stated in Algorithm \ref{alg:SDL_feat_LPGD}.

	\begin{algorithm}[H]
		\caption{SDL-conv-feat}
		\label{alg:SDL_feat_LPGD}
		\begin{algorithmic}[1]
			\State \textbf{Input:} $\X_{\textup{data}}\in \R^{p\times n}$ (Data matrix);\,\,$\X'_{\textup{aux}}\in \R^{q\times n}$ (Auxiliary covariate matrix);\,\, $\Y_{\textup{label}}\in \{0,\dots,\kappa\}^{1 \times n}$ (Label matrix)
			\vspace{0.1cm}
			\State \textbf{Constraints}: Convex set $\Param\subseteq \R^{(p+\kappa)\times n}$ 
			\State \textbf{Parameters}: $\tau>0$ (Stepsize parameter);\,\, $T\in \mathbb{N}$ (number of iterations);
			\State \qquad  Initialize $\A_{0}\in \R^{\kappa\times n}$, $\B_{0}\in \R^{p\times n}$, $\bGamma_{0}\in \R^{q\times \kappa}$
			\State \qquad \textbf{For $k=1,2,\dots,N$ do:}
			\State \qquad \qquad  $\displaystyle
			\begin{bmatrix}
				\A_{k},\,\,
				\B_{k},\,\,
				\bGamma_{k} 
			\end{bmatrix}
			\leftarrow  \Pi_{\Param}	\left( \begin{bmatrix}
				\A_{k-1},\,\, \B_{k-1},\,\, \bGamma_{k-1}
			\end{bmatrix} 
			- \tau \nabla f_{\textup{SDL-feat}}\left( 		\begin{bmatrix}
				\A_{k-1},\,\, \B_{k-1},\,\, \bGamma_{k-1}
			\end{bmatrix}  \right)\right)
			$
			\State\qquad \qquad $\displaystyle 	\begin{bmatrix}
				\A_{k} \\
				\B_{k}
			\end{bmatrix} \leftarrow \Pi_{r}\left(  	\begin{bmatrix}
				\A_{k}  \\
				\B_{k}
			\end{bmatrix}  \right)$	 \qquad ($\triangleright$ rank-$r$ projection)
			\State \qquad\textbf{End for} 
			\State \qquad $[ \U , \bSigma, \V ] \leftarrow \textup{SVD}_{r}\left(\begin{bmatrix}
				\A_{T} \\
				\B_{T}
			\end{bmatrix} \right)$ \qquad ($\triangleright$ rank-$r$ SVD)
			\State \qquad $\displaystyle \begin{bmatrix} (\Beta_{T})^{T} \\ \W_{T} \end{bmatrix} \leftarrow \U \bSigma^{1/2}$, $\H_{T} \leftarrow \bSigma^{1/2} \V^{T}$
			
			\State \textbf{Output:} $(\W_{T}, \H_{T},\Beta_{T}, \bGamma_{T})$
		\end{algorithmic}
	\end{algorithm}
	
	We provide the formulas for the gradients $\nabla f_{\textup{SDL-filt}}$ and $\nabla f_{\textup{SDL-feat}}$ in the appendix, see \eqref{eq:SDL_filt_gradients} and \eqref{eq:SDL_feat_gradients}, respectively.

	\subsection{Nonconvex algorithm for strongly constrained SDL}
	\label{subsection:SDL_BCD_alg}
	
	In this subsection, we provide an algorithm for iteratively solving the strongly constrained SDL problem. More precisely, here we consider both the filter- and the feature-based SDL models in \eqref{eq:ASDL_1} where the constraints on parameters $\W,\H,\Beta$, and $\bGamma$ do not satisfy the weak constraint assumption in \ref{assumption:A1}. In general, this means that we impose separate convex constraints on each of the four parameters. One primary case of interest is using nonnegativity constraints on both $\W$ and $\H$, so that the SDL model \eqref{eq:ASDL_1} combines nonnegative matrix factorization (NMF) together with multinomial logistic regression in two different ways.

	NMF has been a popular dictionary learning model in the literature for various applications, mainly due to the interpretability of dictionary atoms learned under the nonnegativity constraint \cite{lee1999learning, lee2000algorithms}. Analogously, we propose a nonnegative variant of the augmented SDL model \eqref{eq:ASDL_1}, where we impose nonnegativity constraints on the factor matrices $\W$ and $\H$. The special case of filter-based SDL without auxiliary covariates has been studied empirically recently in \cite{austin2018fully, leuschner2019supervised} under the name of `supervised NMF' but without theoretical guarantee of their algorithms. 
	
	In case of strong constraints for the SDL model, we cannot directly use the lifting technique to relate the nonconvex SDL problem to some convex problem in a larger dimensional space as we discussed for the weakly constrained case (see Subsection \ref{subsection:SDL_convex_alg}). We make a key observation that the SDL loss function $L$ in \eqref{eq:ASDL_1}, while being nonconvex, is in fact \textit{multiconvex}. That is, it is convex in each of the four matrix parameters (blocks) $\W,\H,\Beta$, and $\bGamma$ while the other three are held fixed. This fact can be justified by directly computing the Hessian of the loss function $L$ in \eqref{eq:ASDL_1}. (See Lemmas \ref{lem:SDL_filt_BCD_derivatives} and  \ref{lem:SDL_BCD_smooth} in the appendix.) Hence, in order to solve the strongly constrained SDL problems, we may use coordinate descent (BCD) type algorithms \cite{bertsekas1997nonlinear}. The idea is simply to iteratively optimize over one block of parameters while all other blocks are fixed, cycling through all blocks. Such algorithms have been widely used for nonnegative matrix and tensor factorization problems recently \cite{lee1999learning, lee2001algorithms, kolda2009tensor, kim2014algorithms}. Our algorithm for solving the strongly constrained SDL problem \eqref{eq:ASDL_1} uses a similar idea and is stated in Algorithm \ref{algorithm:SDL}. Since each of the four factors $\W,\H,\Beta,$ and $\bGamma$ is iteratively updated by solving a convex sub-problem, it can handle possible constraints for the individual factors such as nonnegativity over $\W$ and $\H$.

	\begin{algorithm}[H]
		\caption{BCD-DR for SDL}
		\label{algorithm:SDL}
		\begin{algorithmic}[1]
			\State \textbf{Input:} $\X_{\textup{data}}\in \R^{p\times n}$ (Data);\,\, $\X_{\textup{aux}}\in \R^{q\times n}$ (Auxiliary covariate);\,\,  $ \Y_{\textup{label}}\in \{0,\dots,\kappa\}^{1 \times n}$  
			\vspace{0.1cm}
			\State \textbf{Constraints}: Convex subsets $\mathcal{C}^{\textup{dict}}\subseteq \R^{p\times r}$,\, $\mathcal{C}^{\textup{code}}\subseteq \R^{r\times n}$,\, $\mathcal{C}^{\textup{beta}}\subseteq \R^{r\times \kappa}$,\, $\mathcal{C}^{\textup{aux}}\subseteq \R^{q\times \kappa}$
			\State \textbf{Parameters}: $\xi\ge 0$ (Tuning parameter);\, $T\in \mathbb{N}$ (number of iterations);\,\, $(r_{k})_{k\ge 1}$ radii in $(0,1]$;
			\State Initialize $\W_{0}\in \mathcal{C}^{\textup{dict}}$, $\H_{0}\in \mathcal{C}^{\textup{code}}$,  $\Beta_{0}\in \mathcal{C}^{\textup{beta}}$,\, $\Beta_{0}'\in \mathcal{C}^{\textup{aux}}$
			\State \textbf{For $k=1,2,\dots,T$ do:}
			\begin{align}
				& \W_{k}\leftarrow  \argmin_{\W \in \mathcal{C}^{\textup{dict}},\, \lVert \W - \W_{k-1} \rVert_{F}\le r_{k}} \,\, L(\W,\H_{k-1},\Beta_{k-1}, \bGamma_{k-1})      \\
				& \Beta_{k}\leftarrow   \argmin_{\Beta\in \mathcal{C}^{\textup{beta}},\,\lVert \Beta - \Beta_{k-1} \rVert_{F}\le r_{k}} \,\, L(\W_{k},\H_{k-1},\Beta,\bGamma_{k-1})     \\
				& \bGamma_{k} \leftarrow   \argmin_{\bGamma\in \mathcal{C}^{\textup{aux}},\,\lVert \bGamma - \bGamma_{k-1} \rVert_{F}\le r_{k}} \,\, L(\W_{k},\H_{k-1},\Beta_{k},\bGamma)     \\
				& \H_{k}\leftarrow \argmin_{\H\in \mathcal{C}^{\textup{code}},\, \lVert \H - \H_{k-1} \rVert_{F}\le r_{k}} \,\, L(\W_{k},\H, \Beta_{k}, \bGamma_{k})
			\end{align}
		
			\State \textbf{End for}
			\State \textbf{Output:} $(\W_{T}, \H_{T}, \Beta_{T}, \bGamma_{T})$ 
		\end{algorithmic}
	\end{algorithm}
	
	For the problems we consider, the radius $r_{k}=O(1/k)$ seems to work well. In most of experiments we perform in this paper we choose the convex constraint sets to be $\mathcal{C}^{\textup{dict}}= \{ \W\in \R^{p\times r}_{\ge 0}\,|\, \lVert \W\rVert_{F}\le 1  \}$, $\mathcal{C}^{\textup{code}}= \{ \H\in \R^{r\times n}_{\ge 0}\,|\, \lVert \H\rVert_{F}\le C_{1} \}$, $\mathcal{C}^{\textup{beta}}= \{ \Beta\in \R^{r\times \kappa}\,|\, \lVert \Beta\rVert_{F}\le C_{2} \}$,  and $\mathcal{C}^{\textup{aux}}= \{ \bGamma\in \R^{q\times \kappa}\,|\, \lVert \bGamma\rVert_{F}\le C_{3} \}$,  where $C_{1},C_{2},C_{3}>0$ are large enough  constants. For example, we may choose them to be a multiple of $\lVert \X_{\textup{data}} \rVert_{F} + \lVert \Y_{\textup{label}} \rVert_{F} )$. 
	
	Each convex optimization sub-problems in Algorithm \ref{algorithm:SDL} can be solved by standard projected gradient descent algorithms, where the optimality gap decays sub-exponentially for convex sub-problems and exponentially if the restricted objectives are strongly convex (see, e.g., \cite[Thm. 10.29]{beck2017first}). In practice, we use only $O(1)$ sub-iterations for solving each convex sub-problems. In fact, the main result in \cite{lyu2020convergence} implies that when each convex sub-problems are strongly convex (which can be enforced by adding an $L_{2}$-regularization term), then it is enough to iteratively solve it at iteration $k$ to the accuracy of $O(k^{-2})$ (in optimality gap of function values), which is achieved by $O(\log k)$ sub-iterations. We provide theoretical convergence guarantees of Algorithm \ref{algorithm:SDL} in Sections \ref{subsection:SDL_BCD} and \ref{subsection:statistical_estimation_BCD}.
	
	Here we provide computations for the derivatives of the loss function $L$ in Lemma \ref{lem:SDL_filt_BCD_derivatives} which may be useful in executing projected gradient descent algorithm to solve each convex sub-problems in Algorithm \ref{algorithm:SDL}. Namely, let $L(\bZ)$ denote the objective of the feature-based SDL in \eqref{eq:ASDL_1}. For the filter-based model in \eqref{eq:ASDL_1}, recall that the activation $\a_{s}:=\Beta^{T}\W^{T}\x_{s}+\bGamma^{T}\x'_{s}$ for predicting $y_{s}$ given $\x_{s}$  and $\x_{s}'$. Define the vector $\dot{\h}(y,\a) \in \R^{\kappa}$  as 
	\begin{align}\label{eq:hddot_def}
		&\dot{\h}(y,\a) :=(\dot{h}_{1},\dots, \dot{h}_{\kappa})^{T}\in \R^{\kappa}, \qquad \dot{h}_{j}=\dot{h}_{j}(y,\a) :=\left( \frac{h'(a_{j})}{1+\sum_{c=1}^{\kappa} h(a_{c})} - \mathbf{1}(y=j)\frac{h'(a_{j})}{h(a_{j})} \right). 
	\end{align} 
	Denote $\K:=[\dot{\h}(y_{1},\a_{1}),\dots, \dot{\h}(y_{1},\a_{1})]\in \R^{\kappa\times n}$. Then the gradients of the objective $L(\bZ)$ can be computed as 
	\begin{align}\label{eq:SDL_filt_BCD_derivatives1_main}
		\textup{(SDL-filt)}\quad 
		\begin{cases} 
			\nabla_{\W} \,  L(\bZ) &= \X_{\textup{data}} \K^{T} \Beta^{T} +  2\xi(\W\H-\X_{\textup{data}})\H^{T},\qquad
			\nabla_{\Beta} \, L(\bZ)  = \W^{T} \X_{\textup{data}}  \K^{T} \\
			\nabla_{\bGamma} \, L(\bZ) &=  \X_{\textup{aux}}\K^{T}, \qquad  \nabla_{\H} \, L(\bZ) =  2\xi \W^{T}(\W\H-\X_{\textup{data}}).
		\end{cases}
	\end{align}
	On the other hand, for the feature-based model in \eqref{eq:ASDL_1}, we use the activation $\a_{s}:=\Beta^{T}\h_{s}+\bGamma^{T}\x'_{s}$. Then the gradients of the objective $L(\bZ)$ is given by 
	\begin{align}\label{eq:SDL_feat_BCD_derivatives1_main}
		\textup{(SDL-feat)}\quad
		\begin{cases}
			\nabla_{\W} \,  L(\bZ) &= 2\xi(\W\H-\X_{\textup{data}})\H^{T},\qquad
			\nabla_{\Beta} \, L(\bZ)  = \H  \K^{T} \\
			\nabla_{\bGamma} \, L(\bZ) &=  \X_{\textup{aux}}\K^{T}, \qquad  \nabla_{\H} \, L(\bZ) =  \Beta \K +   2\xi \W^{T}(\W\H-\X_{\textup{data}}).
		\end{cases}
	\end{align}
	Details for these computations as well as the Hessian computation are given in the appendix, see Lemma \ref{lem:SDL_filt_BCD_derivatives} and Remark \eqref{rmk:feat_SDL_gradient}. 
	
	\subsection{Comparison of algorithms}
	\label{sec:alg_complexity}
	
	In this subsection, we compare our nonconvex (Alg. \ref{algorithm:SDL}) and convex (Alg. \ref{alg:SDL_filt_LPGD} and \ref{alg:SDL_feat_LPGD}) algorithms in terms of their types, asymptotic convergence guarantee, rate of convergence, and computational complexity. Justifications and a more detailed discussion of the asymptotic convergence and complexity bounds we briefly state in Table \ref{table:comparison_algs} are given in the following section, see Theorems \ref{thm:SDL_BCD} and \ref{thm:SDL_LPGD}.

	We give some remarks on the computational complexity of the algorithms. For algorithm \ref{algorithm:SDL} based on BCD-DR, the per-iteration cost is proportional to the cost of computing gradients of the objective in each block variable (e.g., $\W,\H,\Beta,\H$) and the number of projected gradient descent steps (sub-iterations) used in each iteration. According to the discussion in the previous subsection, the number of sub-iterations at iteration $k$ of Algorithm \ref{algorithm:SDL} is at most $O(\log k)$ for theoretical convergence, but in practice, we use $O(1)$ sub-iterations. Hence we simply report the cost of computing gradients for the per-iteration cost of Algorithm \ref{algorithm:SDL} in Table \ref{table:comparison_algs}, which is $O((pr+q)n)$ for both the filter-based and feature-based cases. However, while they have the same asymptotic order, computing gradients for the filter-based model are multiple constant factors more expensive than that for the feature-based model, which can be seen by comparing the gradient formulas in \eqref{eq:SDL_filt_BCD_derivatives1_main} and \eqref{eq:SDL_feat_BCD_derivatives1_main}. Namely, SDL-filter computes additional $\X_{\textup{data}}\K^{T} \Beta^{T}$ for the gradient of $\W$ and the gradient of $\Beta$ uses more expensive matrix multiplication $\W^{T}\X_{\textup{data}}\K^{T}$ depending on $p$ instead of $\H\K^{\kappa}$ for SDL-feature independent of $p$.

	\begin{table}[htbp]
		\centering
		\begin{tabular}{c||cccccc}
			\hline 
			Algorithm &  Type  & \textup{Constraints}  &  $\begin{matrix} \textup{Asymptotic} \\ \textup{Convergence} \end{matrix}$ & $\begin{matrix} \textup{Per-iteration} \\ \textup{cost} \end{matrix}$ & $\begin{matrix} \textup{Iteration} \\ \textup{Complexity} \end{matrix}$  
			\\
			\hline
			Alg. \ref{alg:SDL_filt_LPGD} and \ref{alg:SDL_feat_LPGD}  & GD+rSVD  &  $\begin{matrix} \textup{product of} \\ \textup{factors} \end{matrix}$  & Global minimizer &  $O( \min(pn^{2}, np^{2}))$  & $O(\log \eps^{-1})$  
			\\[10pt] 
			Alg. \ref{algorithm:SDL} & BCD & $\begin{matrix} \textup{individual} \\ \textup{factors} \end{matrix}$  &Stationary points &  $O((pr+q)n)$  &  $O(\eps^{-1}(\log \eps^{-1})^{2})$  
			\\
			\hline
		\end{tabular}
		\caption{Overview of computational aspects of the various SDL algorithms.  `GD' stands for gradient descent, `rSVD'  stands for rank-$r$ SVD, and `BCD' stands for block coordinate descent. Iteration complexity means the worst-case number of iterations to achieve an $\eps$-accurate first-order optimal solution. Under the hypothesis of Theorems \ref{thm:SDL_LPGD} and \ref{thm:SDL_LPGD_feat}, first-order optimality for Algorithms \ref{alg:SDL_filt_LPGD} and \ref{alg:SDL_feat_LPGD} implies global optimality.  We assume $\kappa=O(1)$ in this table.} 
		\label{table:comparison_algs}
	\end{table}
	
	On the other hand, the per-iteration computational cost for Algorithms \ref{alg:SDL_filt_LPGD} and \ref{alg:SDL_feat_LPGD} are dominated by the cost of computing SVD of $p\times (\kappa+n)$ and $(p+q)\times n$ matrix, respectively. Assuming $q=O(p)$, this is of order $O( \min(pn^{2}, np^{2}))$. However, since we only need rank-$r$ truncated SVD, we may employ a much more efficient truncated SVD using random projection \cite{halko2011finding}, which approximately computes rank-$r$ truncated SVD. Using this heuristic, one can reduce the per-iteration computational cost to $O((p+q)rn)$, which is essentially the same order for the nonconvex method in Algorithm \ref{algorithm:SDL}. However, our convergence guarantee in Theorems \ref{thm:SDL_LPGD} and \ref{thm:SDL_LPGD_feat} does not apply in this case. Also in practice, we find the convex methods in Algorithms \ref{alg:SDL_filt_LPGD} and \ref{alg:SDL_feat_LPGD} depend more sensitively on the choice of hyperparameters (e.g., stepsize) than the nonconvex method in Algorithm \ref{algorithm:SDL}.


	\section{Statement of Results}

	\subsection{First-order optimality measures }
	\label{sec:stationary_points}

	In order to make the notion of approximate solutions, we first recall the definition of stationary points for constrained optimization problems. Namely, consider the problem of minimizing a function $f:\R^{p}\rightarrow \R$ over a convex set $\Param\subset \R^{p}$. Then
	$\param^{*}\in \Param$ is a \textit{stationary point} of $f$  over $\Param$ if $\inf_{\param\in \Param} \, \langle \nabla f(\param^{*}) ,\, \param - \param^{*} \rangle \ge 0$. This is equivalent to saying that $-\nabla f(\param^{*})$ is in the normal cone of $\Param$ at $\param^{*}$. Every local minimum of $f$ over $\Param$ is a stationary point. Relaxing this notion, for each $\eps\ge 0$, we define $\param^{*}\in  \Param$ to be an \textit{$\eps$-stationary point} of $f$  over $\Param$ if 
	\begin{align}\label{eq:stationary_approximate}
		-\inf_{\param\in \Param} \, \left\langle \nabla f(\param^{*}) ,\,  \frac{\param - \param^{*}}{\lVert \param - \param^{*} \rVert_{F}} \right\rangle  \le \sqrt{\eps}.
	\end{align}	
	In order to explain this definition, suppose $\param^{*}$ lies in the interior of $\Param$. In this case, \eqref{eq:stationary_approximate} is equivalent to $\lVert \nabla f(\param^{*}) \rVert_{F}^{2}\le \eps$. When $f$ is differentiable, $\param^{*}$ is a stationary point of $f$ over $\Param$ if and only if $\lVert \nabla f(\param^{*}) \rVert_{F}^{2}=0$. Moreover, it is standard that a rate of convergence to stationary points in the interior of the constraint set $\Param$ (the whole parameter space for unconstrained problems) is measured in terms of gradient norm squared  \cite{sun2018markov, xu2019convergence, ward2019adagrad, khamaru2018convergence}. We also remark that the above notion of $\eps$-approximate solution is  also equivalent to a similar notion in \cite[Def. 1]{nesterov2013gradient}, which is stated for non-smooth objectives using subdifferentials instead of gradients as in \eqref{eq:stationary_approximate}. 	
	
	
	\subsection{Exponential convergence  of Low-rank PGD}
	
	In order to solve the CALE problem \eqref{eq:CALE}, consider the following \textit{Low-rank Projected Gradient Descent} (LPGD) algorithm: (See Algorithm \ref{algorithm:LPGD})
	\begin{align}\label{eq:LRPGD_iterate0}
		\bZ_{t} \leftarrow \Pi_{r}\left( \Pi_{\Param} \left(\bZ_{t-1} - \tau \nabla f(\bZ_{t-1}) \right) \right), 
	\end{align}
	where $\tau$ is a stepsize parameter, $\Pi_{\Param}$ denotes projection onto the convex constraint set $\Param\subseteq \R^{d_{1}\times d_{2}} \times \R^{d_{3}\times d_{4}}$, and $\Pi_{r}$ denotes the projection of the first matrix component onto matrices of rank at most $r$ in $\R^{d_{1}\times d_{2}}$. More precisely, let $\bZ = [\X,\bGamma]$. Then $\Pi_{r}(\bZ):=[ \Pi_{r}(\X), \bGamma ]$. It is well-known that the rank-$r$ projection above can be explicitly computed by the singular value decomposition (SVD). Namely, $\Pi_{r}(\X)= \U\bSigma\V^{T}$, where $\bSigma$ is the $r\times r$ diagonal matrix of the top $r$ singular values of $\X$ and  $\U\in \R^{d_{1}\times r}$, $\V\in \R^{d_{2}\times r}$ are semi-orthonormal matrices (i.e., $\U^{T}\U = \V^{T}\V = \I_{r}$). 
	
	\begin{algorithm}[H]
		\caption{Low-rank Projected Gradient Descent (LPGD)} 
		\label{algorithm:LPGD}
		\begin{algorithmic}[1]
			\State \textbf{Input:}  $f:\R^{d_{1}\times d_{1}}\times \R^{d_{3}\times d_{4}}\rightarrow \R$ (Objective function);\,  $\Param\subseteq \R^{d_{1}\times d_{1}}\times \R^{d_{3}\times d_{4}}\rightarrow \R$ (Convex constraint);\, $r\in \mathbb{N}$ (Rank parameter); 
			\State \textbf{Parameters}: $\tau>0$ (Stepsize parameter);\,\, $N\in \mathbb{N}$ (number of iterations);
			\State \qquad  Initialize $\bZ_{0}\in \Param$
			\State \qquad \textbf{For $t=1,2,\dots,T$ do:}
			\State \qquad \qquad  $\displaystyle
			\bZ_{t} \leftarrow \Pi_{r}\left(  \Pi_{\Param}(\bZ_{t-1} - \tau \nabla f(\bZ_{t-1}) )  \right) 
			$
			\State \qquad\textbf{End for} 
			\State \textbf{Output:} $\bZ_{T}$
		\end{algorithmic}
	\end{algorithm}
	
	Note that Algorithm \ref{algorithm:LPGD} resembles the standard projected gradient descent (PGD) in the optimization literature, as a gradient descent step is followed first by projecting onto the convex constraint set $\Param$ and then by the rank-$r$ projection. 
	 It is also worth noting the similarity of \eqref{eq:LRPGD_iterate0} to the `lift-and-project' algorithm in \cite{chu2003structured} for structured low-rank approximation problem, which proceeds by alternatively applying the projections $\Pi_{\Param}$ and $\Pi_{r}$ to a given matrix until convergence.  
	 
	 In Theorem \ref{thm:CALE_LPGD}, we show that the iterate $\bZ_{t}$ of Algorithm \ref{algorithm:LPGD} converges exponentially to a low-rank approximation of the global minimizer of the objective $f$ over $\Param$, given that the objective $f$ satisfies the following restricted strong convexity (RSC) and restricted smoothness (RSM) properties in Definition \ref{def:RSC}. These properties were first used in  \cite{agarwal2010fast, ravikumar2011high,negahban2011estimation} for a class of matrix estimation problems and have found a number of applications in optimization and machine learning literature \cite{wang2017unified,park2018finding, tong2021accelerating}.

	\begin{definition}(Restricted Strong Convexity and Smoothness)\label{def:RSC}
		A function $f:\R^{d_{1}\times d_{2}} \times \R^{d_{3}\times d_{4}} \rightarrow \R$ is \textit{$r$-restricted strongly convex and smooth} with parameters $\mu,L>0$ if for all $\X,\Y\in \R^{d_{1}\times d_{2}}\times \R^{d_{3}\times d_{4}}$  whose matrix coordinates are of rank $\le r$, 
		\begin{align}\label{eq:def_RSC_RSM}
			\frac{\mu}{2} \lVert \vect(\X) - \vect(\Y) \rVert_{2}^{2} \overset{\textup{(RSC)}}{\le} 	f(\Y) - f(\X) - \langle \nabla f(\X),\, \Y-\X \rangle  \overset{\textup{(RSM)}}{\le} \frac{L}{2} \lVert \vect(\X) - \vect(\Y) \rVert_{2}^{2}.
		\end{align}
	\end{definition}
	
	Recall that the CALE \eqref{eq:CALE} problem considers a constrained optimization problem, where the global minimizer of the objective function $f$ over the constraint set $\Param$ need not be a critical point of $f$, but only a stationary point when it is at the boundary of $\Param$. In order to measure the rate of convergence of an algorithm to a stationary point, we use gradient mapping \cite{nesterov2013gradient, beck2017first} as a measure of the degree at which a point $\bZ^{\star}$ in $\Param$ fails to be a stationary point, which is particularly well-suited for projected gradient descent type algorithms. Namely, for the CALE problem in \eqref{eq:CALE}, we define a map $G:\Param\times (0,\infty)\rightarrow \R$ by 
	\begin{align}\label{eq:def_grad_mapping}
		G(\bZ, \tau) := \frac{1}{\tau}(\bZ - \Pi_{\Param}(\bZ - \tau \nabla f(\bZ))).
	\end{align}
	We call $G$ the \textit{gradient mapping} associated with problem \eqref{eq:CALE}. In the special cases when $\Param$ is the whole space or when $\bZ$ is in the interior of $\Param$,  if $\tau$ is sufficiently small (so that $\bZ-\tau\nabla f(\bZ)\in \Param$), then $\lVert G(\bZ, \tau)\rVert_{F} = \lVert \nabla f(\bZ) \rVert_{F}$, which is the standard measure of first-order optimality of $\bZ$ for minimizing the objective $f$. In general, it holds that $\lVert G(\bZ,\tau)\rVert_{F}\le \lVert \nabla f(\bZ) \rVert_{F}$ (see Lemma \ref{lem:gradient_mapping}). 
	
	In order to better motivate the definition, fix $\bZ\in \Param$ and decompose it as 
	\begin{align}
		\bZ &=\Pi_{\Param}(\bZ - \tau \nabla f(\bZ))   + \left( \bZ- \Pi_{\Param}(\bZ - \tau \nabla f(\bZ))   \right) \\
		&=\Pi_{\Param}(\bZ - \tau \nabla f(\bZ))  +  \tau G(\bZ, \tau).
	\end{align}
	Namely, the first term above is a one-step update of a projected gradient descent at $\bZ$ over $\Param$ with stepsize $\tau$, and the second term above is the error term. If $\bZ$ is a stationary point of $f$ over $\Param$, then $-\nabla f(\bZ)$ lies in the normal cone of $\Param$ at $\bZ$, so $\bZ$ is invariant under the projected gradient descent and the error term above is zero. If $\bZ$ is only approximately stationary, then the error above is nonzero. In fact, $ G(\bZ, \tau)=0$ if and only if $\bZ$ is a stationary point of $f$ over $\Param$ (see \cite[Thm 10.7]{beck2017first}). Therefore, we may use the size of $ G(\bZ,\tau)$ (measured using an appropriate norm) as a measure of first-order optimality of $\bZ$ for the problem \eqref{eq:CALE}.

	Now we state our main result concerning exponential convergence of Algorithm \ref{algorithm:LPGD} for CALE \eqref{eq:CALE}.
	
	\begin{theorem}(Exponential convergence  of LPGD)\label{thm:CALE_LPGD}
		Let $f:\R^{d_{1}\times d_{2}} \times \R^{d_{3}\times d_{4}} \rightarrow \R$ be $r$-restricted strongly convex and smooth with parameters $\mu$ and $L$, respectively, with $L/\mu<3$.  Let $(\bZ_{t})_{t\ge 0}$ be the iterates generated by Algorithm \ref{algorithm:LPGD}. Suppose $\Param\subseteq \R^{d_{1}\times d_{2}} \times \R^{d_{3}\times d_{4}}$ is a convex subset and fix a stepsize $\tau\in ( \frac{1}{2\mu}, \frac{3}{2L})$. Then $\rho:=2 \max( |1-\tau\mu|,\, |1-\tau L| ) \in (0,1)$ and  the followings hold:
		
		\begin{description}
			\item[(i)] Let $\bZ^{\star}=[\X^{\star}, \bGamma^{\star}]\in \Param$ be arbitrary with $\rank(\X^{\star})\le r$. Write $ G(\bZ^{\star}, \tau) =[ \Delta\X ^{\star}, \Delta \bGamma^{\star}]$. Then for $t\ge 1$, 
			\begin{align}\label{eq:CALE_LPGD_thm}
				\lVert \bZ_{t} - \bZ^{\star}  \rVert_{F}  \le  \rho^{t}  \, \lVert  \bZ_{0} - \bZ^{\star}\rVert_{F} + \frac{\tau}{1-\rho}  \left( \sqrt{3r} \lVert \Delta \X^{\star} \rVert_{2} + \lVert \Delta \bGamma^{\star}  \rVert_{F} \right).
			\end{align}

			\item[(ii)] Suppose $\bZ^{\star}=[\X^{\star}, \bGamma^{\star}]$ is any stationary point of $f$ over $\Param$ whose matrix factor $\X^{\star}$ has rank $\le r$. Then $\lVert \bZ_{t} - \bZ^{\star}  \rVert_{F}\le  \rho^{t}  \, \lVert  \bZ_{0} - \bZ^{\star}\rVert_{F}$.  Furthermore, if $\nabla f$ is Lipschitz continuous over $\Param$, then for $t\ge 1$, 
			\begin{align}\label{eq:PSGD_linear_conv2}
				f\left( \bZ_{t}
				\right)  -  f(\bZ^{\star}) \le \left( \lVert \nabla f(\bZ^{\star}) \rVert + L \rho^{t} \right) \rho^{t} \lVert \bZ_{0}-\bZ^{\star}\rVert_{F}.
			\end{align}
		\end{description}
	\end{theorem}

	In particular, if $\bZ^{\star}$ in the above theorem is a stationary point of $f$ over $\Param$, then $\Delta \X^{\star}=O$ and $\Delta \bGamma^{\star}=\mathbf{0}$, so the above theorem implies that the iterate $\bZ_{t}$ converges to $\bZ^{\star}$ at a geometric rate with contraction constant $\le \rho$. In particular, it implies that there is a unique global minimizer of $f$ over $\Param$ under the hypothesis of Theorem \ref{thm:CALE_LPGD}. In a statistical estimation setting, the true parameter $\bZ^{*}$ to be estimated may only be approximately stationary. In that case, the second term on the right-hand side of {\blue \eqref{eq:CALE_LPGD_thm}} gives a bound on the statistical error, whereas the first term shows the algorithmic error decays geometrically. See Section \ref{subsection:statistical_estimation} for implications of Theorem \ref{thm:CALE_LPGD} in the context of statistical estimation for SDL.

	\subsection{Exponential convergence  of weakly constrained SDL} 
	In Section \ref{subsection:SDL_convex_alg}, we have discussed that weakly constrained SDL problem \eqref{eq:ASDL_1} can be reformulated as a CAFE problem \eqref{eq:CAFE}, which itself is a factored reformulation of the CALE problem \eqref{eq:CALE}.  Note that CAFE problems \eqref{eq:CAFE} in general does not have unique minimizer due to the `rotation invariance'. Namely, let $\bR$ be any $r\times r$ orthonormal (rotation) matrix (i.e., $\bR^{T}\bR=\bR \bR^{T}=\I_{r}$). Then 
	\begin{align}\label{eq:identifiability}
		f((\U\bR)(\V\bR)^{T}, \bGamma) =  f(\U\bR \bR^{T} \V^{T} , \bGamma) =   f(\U\V^{T} , \bGamma).
	\end{align}
	Thus $[\U, \V,\bGamma]$ and $[\U\bR, \V\bR, \param]$ give the same objective value. This implies that the best type of guarantee that we can hope for the CAFE problem \eqref{eq:CAFE} is the recovery of a global minimizer $[\U, \V,\bGamma]$ so that the product $\U\V^{T}$ is uniquely determined but not necessarily for the pair $(\U,\V)$. 
	
	
	In the context of the filter-based SDL problem \eqref{eq:SDL_filt_1}, we may seek to find a global minimizer $[\W,\H,\Beta,\bGamma]$  of the objective such that the products $\W\Beta^{T}$ and $\W\H$ are uniquely determined. Similarly, in the context of the feature-based SDL problem \eqref{eq:SDL_feat_1}, we may seek to find a global minimizer $[\W,\H,\Beta,\bGamma]$  of the objective such that the products $\Beta\H$ and $\W\H$ are uniquely determined. These goals can be achieved by using Algorithms \ref{alg:SDL_filt_LPGD} and \ref{alg:SDL_feat_LPGD}, respectively, as long as the hypothesis of Theorem \ref{thm:CALE_LPGD} is satisfied. In such a case, the convergence is exponential and we can also show that the optimality gap (in objective value) shrinks geometrically, as stated in Theorem \ref{thm:SDL_LPGD}. 
	
	
	We first introduce the following technical assumption (\ref{assumption:A2}-\ref{assumption:A4}) that are needed to quantify the restricted strong convexity and smoothness parameters for the SDL loss function in \eqref{eq:ASDL_1}. Namely, \ref{assumption:A2} limits the norm of the activation $\a$ as an input for the classification model in \eqref{eq:ASDL_1} is bounded. This is standard in the literature (see, e.g., \cite{negahban2011estimation}) in order to uniformly bound the eigenvalues of the Hessian of the (multinomial) logistic regression model; \ref{assumption:A3} introduces uniform bounds on the eigenvalues of the covariance matrix of the input data; \ref{assumption:A4} introduces uniform bounds on the eigenvalues of the $\kappa\times \kappa$ observed information as well as the first derivative of the predictive probability distribution (see \cite{bohning1992multinomial} and Appendix \ref{sec:MLR} for more details). For \ref{assumption:A2}, we remark that there are a number of known works that bound the eigenvalues of the averaged Gram matrix $n^{-1}\bPhi\bPhi^{T}$ with i.i.d. columns (see, e.g., \cite{yaskov2016controlling, lecue2017sparse}).

	\begin{customassumption}{A2}(Bounded activation)\label{assumption:A2}
		The activation $\a=\a(\x_{i},\x'_{i}, \W, \h_{i}, \Beta,\bGamma)\in \R^{\kappa}$ defined in \eqref{eq:ASDL_1} assume bounded norm, i.e., $\lVert \a \rVert\le M$ for some constant $M\in (0,\infty)$.
	\end{customassumption}
	
	\begin{customassumption}{A3}(Bounded eigenvalues of covariance matrix)\label{assumption:A3}
		Denote $\bPhi=[\bphi_{1},\dots,\bphi_{n}]\in \R^{(p+q)\times n}$, where $\bphi_{i} = [\x_{i}^{T}, (\x_{i}')^{T}]^{T}\in \R^{p+q}$. In other words, $\bPhi=[\X_{\textup{data}}^{T}, \X_{\textup{aux}}^{T}]^{T}$. Then there exists constants $c^{-},c^{+}>0$ such that for all $n\ge 1$, 
		\begin{align}
			\delta^{-} \le \lambda_{\min}( n^{-1} \bPhi \bPhi^{T}) \le \lambda_{\max}( n^{-1}\bPhi \bPhi^{T}) \le \delta^{+}.
		\end{align}
	\end{customassumption}
	
	\begin{customassumption}{A4}(Bounded stiffness and eigenvalues of observed information)\label{assumption:A4}
		The score function  $h:\R\rightarrow [0,\infty)$  in \eqref{eq:prediction_model_g} is twice continuously differentiable. Further, for $y\in \{0,1,\dots,\kappa\}$ and $\a=(a_{1},\dots,a_{\kappa})\in \R^{\kappa}$, define 
		the  symmetric matrix $\ddot{\H}(y,\a)\in \R^{\kappa\times \kappa}$ as 
		\begin{align}\label{eq:Hddot_def}
			&\ddot{\H}(y,\a) \in \R^{\kappa\times \kappa}, \quad \ddot{\H}(y,\a)_{ij}  := \begin{matrix} \left(  \frac{ h''(a_{j}) \mathbf{1}(i=j)  }{  1+\sum_{c=1}^{\kappa} h(a_{c})  } - \frac{ h'(a_{i}) h'(a_{j})   }{  \left( 1+\sum_{c=1}^{\kappa} h(a_{c}) \right)^{2} }  \right) - \mathbf{1}(y=i=j) \left( \frac{h''(a_{j})}{h(a_{j})} - \frac{\left( h'(a_{j}) \right)^{2} }{\left( h(a_{j}) \right)^{2}}   \right) \end{matrix}. 
		\end{align}
		Then for the constant $M>0$ in \ref{assumption:A2}, there exists constants $\gamma_{\max}, \alpha^{-},\alpha^{+}>0$ such that 
		\begin{align}
			&\hspace{3cm} \gamma_{\max}:=\sup_{\lVert \a \rVert<M} \max_{1\le s \le  n}\, \lVert \dot{\h}(y_{s},\a_{s}) \rVert_{\infty},\\
			&	\alpha^{-}:=\inf_{\lVert \a \rVert < M} \min_{1\le s \le  n}\, \lambda_{\min}(\ddot{\H}(y_{s},\a)) ,\quad \alpha^{+}:=\sup_{\lVert \a \rVert<M} \max_{1\le s \le  n}\, \lambda_{\max}(\ddot{\H}(y_{s},\a)),
		\end{align}
		where the vector $\dot{\h}(y,\a) \in \R^{\kappa}$ is defined in \eqref{eq:hddot_def}.
	\end{customassumption}
	
	Under \ref{assumption:A2} and the multinomial logistic regression model, one can derive \ref{assumption:A4} with a simple expression for the bounds $\alpha^{\pm}$, as discussed in the following remark. 
	
	\begin{remark}[Multinomial Logistic Classifier]
		\label{rmk:MNL_constants}
		In the special case of multinomial logistic model with the score function $h(\cdot)=\exp(\cdot)$, we have $h=h'=h''$ so the second term in \eqref{eq:Hddot_def} vanishes and we get 
		\begin{align}
			\dot{h}_{j}(y,\a)&= g_{j}(\a) - \mathbf{1}(y=j) \\
			\ddot{H}(y,\a)_{ij} &= g_{i}(\a) \left(\mathbf{1}(i=j)-g_{j}(\a) \right)  = \frac{ \exp( a_{i} )}{  1+\sum_{c=1}^{\kappa} \exp( a_{c})} \left( \mathbf{1}(i=j) - \frac{ \exp( a_{j} )}{  1+\sum_{c=1}^{\kappa} \exp( a_{c})}  \right),
		\end{align}
		where $g_{j}$ is the predictive probability of label $j$ given activation $\a$ (see \eqref{eq:prediction_model_g}). Under \ref{assumption:A2}, according to Lemma \ref{lem:MNL}, we can take 
		\begin{align}
			\gamma_{\max}=1+\frac{ e^{M} }{1+ e^{M} + (\kappa-1) e^{-M}}\le 2 ,\quad 	\alpha^{-} = \frac{e^{-M}}{1+e^{-M}+(\kappa-1) e^{M}},\quad \alpha^{+} = \frac{e^{M} \left(1+2(\kappa-1)e^{M} \right) }{\left( 1+e^{M}+(\kappa-1) e^{-M} \right)^{2}}.
		\end{align}
		For binary classification, $\kappa=1$, it also holds that $\alpha^{+}\le 1/4$. 
	\end{remark}

	We now state the main result in this section.

	\begin{theorem}(Exponential convergence  of LPGD for SDL-filt)\label{thm:SDL_LPGD}
		Let $\bZ_{t}:=[[\A_{t}, \B_{t}], \bGamma_{t}]$ denote the iterates of Algorithm \ref{alg:SDL_filt_LPGD} for the filter-based SDL problem \eqref{eq:SDL_filt_1}. Assume \ref{assumption:A1}-\ref{assumption:A4} hold. Fix any $\mu^{*}\le \delta^{-}\alpha^{-}$ and $L^{*}\ge \delta^{+}\alpha^{+}$. Let $\mu:=\min(2\xi,\,2\nu+n \mu^{*})$ and $L:=\max(2\xi,\, 2\nu+ n L^{*})$ and suppose that
		\begin{align}\label{eq:thm_SDL_LGPD_cond}
			\frac{L}{\mu}= \frac{\max( \xi,\, \nu + \frac{n L^{*}}{2} )}{\min(\xi, \, \nu +\frac{n \mu^{*}}{2})} < 3.
		\end{align}
		Fix any stepsize $\tau\in ( \frac{1}{2\mu}, \frac{3}{2L})$ and let $\rho:=2 \max( |1-\tau\mu|,\, |1-\tau L| ) \in (0,1)$. Let $\bZ^{\star}:=[[\A^{\star}, \B^{\star}], \bGamma^{\star}]\in \Param$ denote the unknown target parameters to be estimated such that $\rank([\A^{\star}, \B^{\star}])\le r$.  Denote $\rho:=2(1-\tau\mu)$. Then the followings hold:
		\begin{description}
			\item[(i)] Denote $[ \Delta\X ^{\star}, \Delta \bGamma^{\star}] :=\frac{1}{\tau}\left( \bZ^{\star} - \Pi_{\Param}(\bZ^{\star} - \tau \nabla f_{\textup{SDL-filt}}(\bZ^{\star})) \right)$ (see \eqref{eq:def_grad_mapping}). Then for $t\ge 1$, 
			\begin{align}\label{eq:LPGD_SDL_filt_thm}
				\lVert \bZ_{t} - \bZ^{\star}  \rVert_{F}  \le  \rho^{t}  \, \lVert  \bZ_{0} - \bZ^{\star}\rVert_{F} + \frac{\tau }{1-\rho}  \left( \sqrt{3r} \lVert \Delta \X^{\star} \rVert_{2} + \lVert \Delta \bGamma^{\star}  \rVert_{F} \right).
			\end{align}
			
			\item[(ii)] Suppose $\bZ^{*}=[\X^{*}, \bGamma^{*}]$ is any stationary point of $f_{\textup{SDL-filt}}$ over $\Param$ such that $\rank(\X^{*})\le r$. Then  for $t\ge 1$, we have 
			\begin{align}\label{eq:PSGD_linear_conv2}
				\lVert \bZ_{t} - \bZ^{*}  \rVert_{F}&\le  \rho^{t}  \, \lVert  \bZ_{0} - \bZ^{*}\rVert_{F}, \\
				f_{\textup{SDL-filt}}\left( [\A_{t},\B_{t}] ,\, \bGamma_{t}
				\right)  -  f_{\textup{SDL-filt}}(\bZ^{*}) 
				&\le \left( \lVert \nabla f_{\textup{SDL-filt}}(\bZ^{*}) \rVert + L \rho^{t} \right) \rho^{t} \lVert \bZ_{0}-\bZ^{*}\rVert_{F}.
			\end{align}
		\end{description}
	\end{theorem}

	Note that we may view the ratio $L/\mu$ that appears in Theorem \ref{thm:SDL_LPGD} as the condition number of the SDL problem in \eqref{eq:ASDL_1}, whereas the ratio $L^{*}/\mu^{*}$ as the condition number for the multinomial classification problem. These two condition numbers are closely related. First, note that for any given $\mu^{*}, L^{*}$ and sample size $n$, we can always make $L/\mu<3$ by choosing sufficiently large $\xi$ and $\nu$ so that Theorem \ref{thm:SDL_LPGD} holds. However, using large $L_{2}$-regularization coefficient $\nu$ may perturb the original SDL problem \ref{eq:ASDL_1} too much   that the converged solution may not be close to the optimal solution. Hence we may want to take $\nu$ as small as possible. For instance, setting $\nu=0$, the condition \eqref{eq:thm_SDL_LGPD_cond} reduces to 
	\begin{align}\label{eq:thm_SDL_LGPD_cond_filt1}
		0 < L^{*} < 3\mu^{*},\qquad \frac{L^{*}}{6}< \frac{\xi}{n} < \frac{3\mu^{*}}{2}.
	\end{align}
	That is, if the multinomial classification problem is  well-conditioned ($L^{*}/\mu^{*}<3$) and the ratio $\xi/n$ is in the above interval, then we can find the global optimum of the SDL problem \eqref{eq:ASDL_1} exactly and exponentially fast by Algorithms \ref{alg:SDL_filt_LPGD} and \ref{alg:SDL_feat_LPGD} depending on the activation type. In general, denoting $\xi=\xi'n$ and $\nu=\nu'n$, \eqref{eq:thm_SDL_LGPD_cond} reduces to 
	{\small
		\begin{align}
			&\frac{L^{*}}{\mu^{*}}<3 \,\, \Rightarrow \,\, \left(\frac{L^{*}}{6} <\xi' <\frac{3\mu^{*}}{2}, \quad 0\le  \nu' <\frac{6\xi'-L^{*}}{2} \right) \cup \left( \xi'>\frac{3\mu^{*}}{2},\quad \frac{2\xi'-3\mu^{*}}{6}< \nu' <\frac{6\xi'-L^{*}}{2} \right) \\
			&\frac{L^{*}}{\mu^{*}}\ge 3 \,\, \Rightarrow \,\, \left(\frac{L^{*}-\mu^{*}}{4} <\xi' <\frac{3(L^{*}-\mu^{*})}{4}, \,\, \frac{L^{*}-3\mu^{*}}{4}< \nu' <\frac{6\xi'-L^{*}}{2} \right) \cup \left( \xi'>\frac{3(L^{*}-\mu^{*})}{2},\,\, \frac{2\xi'-3\mu^{*}}{6}< \nu' <\frac{6\xi'-L^{*}}{2} \right).
		\end{align}
	}
	
	Next, we state an analogous result as in Theorem \ref{thm:SDL_LPGD} for the feature-based SDL problem in \eqref{eq:SDL_feat_1}. 
	
	
	\begin{theorem}(Exponential convergence  of LPGD for SDL-feat)\label{thm:SDL_LPGD_feat}
		Consider the feature-based SDL problem \eqref{eq:SDL_feat_1}. Let $\bZ_{t}:=\left[ \begin{bmatrix}\A_{t}^{T},\,  \B_{t}^{T}\end{bmatrix}^{T} ,\bGamma_{t}\right]$ denote the iterates of Algorithm \ref{alg:SDL_feat_LPGD}. Assume \ref{assumption:A1}-\ref{assumption:A2} and \ref{assumption:A4} hold. Denote $\lambda_{\max}:=\lambda_{\max}( n^{-1}\X_{\textup{aux}}\X_{\textup{aux}}^{T})$, $\mu:=\min(2\xi,\,2\nu+ \alpha^{-} )$, and $L:=\max( 2\xi, 2\nu +\alpha^{+}\lambda_{\max} n , \, \alpha^{+} + 2\nu)$. Suppose that
		\begin{align}\label{eq:thm_SDL_LGPD_cond_feat}
			\frac{L}{\mu}= \frac{\max( 2\xi, 2\nu +\alpha^{+}\lambda_{\max} n , \, \alpha^{+} + 2\nu) }{ \min(2\xi,\,2\nu+ \alpha^{-} )} < 3.
		\end{align} 
		Fix any stepsize $\tau\in ( \frac{1}{2\mu}, \frac{3}{2L})$ and let $\rho:=2 \max( |1-\tau\mu|,\, |1-\tau L| ) \in (0,1)$. Let $\bZ^{\star}:=\left[ [(\A^{\star})^{T}, (\B^{\star})^{T}]^{T}, \bGamma^{\star}\right]\in \Param$ denote the unknown target parameters to be estimated such that $\rank([(\A^{\star})^{T}, (\B^{\star})^{T}]^{T})\le r$.  Denote $\rho:=2(1-\tau\mu)$. Then the followings hold:
		\begin{description}
			\item[(i)] Denote $[\Delta\X ^{\star},\Delta\bGamma^{\star}] :=\frac{1}{\tau}\left( \bZ^{\star} - \Pi_{\Param}(\bZ^{\star} - \tau \nabla f_{\textup{SDL-feat}}(\bZ^{\star})) \right)$ (see \eqref{eq:def_grad_mapping}). Then for $t\ge 1$, 
			\begin{align}\label{eq:LPGD_SDL_feat_thm}
				\lVert \bZ_{t} - \bZ^{\star}  \rVert_{F}  \le  \rho^{t}  \, \lVert  \bZ_{0} - \bZ^{\star}\rVert_{F} + \frac{\tau }{1-\rho}  \left( \sqrt{3r} \lVert \Delta \X^{\star} \rVert_{2}  + \lVert \Delta \bGamma^{\star}\rVert_{F}\right).
			\end{align}
			
			\item[(ii)] Suppose $\bZ^{*}=[\X^{*},\bGamma^{*}]$ is any stationary point of $f_{\textup{SDL-feat}}$ over $\Param$ such that $\rank\left( \X^{*} \right)\le r$. Then  for $t\ge 1$, we have 
			\begin{align}\label{eq:PSGD_linear_conv2}
				\lVert \bZ_{t} - \bZ^{*}  \rVert_{F}&\le  \rho^{t}  \, \lVert  \bZ_{0} - \bZ^{*}\rVert_{F}, \\
				f_{\textup{SDL-feat}}\left(  \begin{bmatrix}\A_{t} \\ \B_{t}\end{bmatrix},\, \bGamma_{t}
				\right)  -  f_{\textup{SDL-feat}}(\bZ^{*})
				&\le \left( \lVert \nabla f_{\textup{SDL-feat}}(\bZ^{*}) \rVert + L \rho^{t} \right) \rho^{t} \lVert \bZ_{0}-\bZ^{*}\rVert_{F}.
			\end{align}
		\end{description}
	\end{theorem}
	
	We give some remark on the parameter regime \eqref{eq:thm_SDL_LGPD_cond_feat} where Theorem \ref{thm:SDL_LPGD_feat} holds. First, suppose no auxiliary covariate is used (e.g., $\X_{\textup{aux}}=O$) so that $\lambda_{\max}=0$. Then \eqref{eq:thm_SDL_LGPD_cond_feat} reduces to 
	\begin{align}\label{eq:thm_SDL_LGPD_cond_feat2}
		\frac{\max( 2\xi,\, 2\nu + \alpha^{+})}{\min(2\xi, \, 2\nu + \alpha^{-} ) } < 3 \qquad \Longleftrightarrow \qquad 	\max\left( \frac{\alpha^{+}}{4} ,\, \frac{\alpha^{+}-6\xi \alpha^{-}}{4}\right)<\nu < \frac{\xi}{3}.
	\end{align}
	In particular, the above condition is satisfied if $\nu$ and $\xi$ grow in $n$ in arbitrary rate and $\nu<\xi/3$. Other special case of interest is when there is no $L_{2}$-regularization, that is $\nu=0$. In this case \eqref{eq:thm_SDL_LGPD_cond_feat} reduces to 
	\begin{align}\label{eq:thm_SDL_LGPD_cond_feat3}
		\frac{\max( 2\xi, \alpha^{+}\lambda_{\max} n  ) }{ \min(2\xi,\,  \alpha^{-} )} < 3.
	\end{align}
	
	Note that for any fixed $\xi>0$, the ratio on the left-hand side above blows up as $n\rightarrow \infty$. This is also the case if we scale $\xi$ and $\nu$ to grow linearly in $n$.  Hence it is necessary to have $L_{2}$-regularization for the feature-based SDL model in order for the low-rank PGD Algorithm \ref{alg:SDL_feat_LPGD} converges exponentially fast to the global optimum. This is in contrast to the filter-based SDL model, which allows to have no $L_{2}$-regularization in some regime (see \eqref{eq:thm_SDL_LGPD_cond_filt1}) and still enjoy exponential convergence of Algorithm \ref{alg:SDL_filt_LPGD}. 
	
	
	\subsection{{Subexponential convergence of BCD-DR for SDL}}
	\label{subsection:SDL_BCD}
	
	In this subsection, we discuss convergence guarantees of Algorithm \ref{algorithm:SDL} for the strongly constrained SDL problem. Recall that the algorithm is a direct application of \textit{block coordinate descent with diminishing radius} (BCD-DR) in \cite{lyu2020convergence}, and the theoretical result we present here can be derived based on the main result in the aforementioned reference. 
	
	The augmented SDL training problem \eqref{eq:ASDL_1} with general convex constraints on each factor $\W,\H,\Beta,\bGamma$ is a constrained nonconvex optimization problem, so in general, it is difficult to guarantee to find a globally optimal solution. Indeed, our strong global convergence guarantee of augmented SDL in Theorem \ref{thm:SDL_LPGD} is not applicable in this case, since the assumption \ref{assumption:A1} is violated. Previous works on a related model of supervised NMF \cite{austin2018fully, leuschner2019supervised} do not provide any theoretical convergence guarantee.  However, we can exploit the multi-convexity of the objective function and apply a block coordinate descent (BCD) type algorithm to seek to find a locally optimal solution. We apply BCD with diminishing radius developed in \cite{lyu2020convergence} and establish convergence to local optima. Furthermore, we show that, in order to achieve an $\eps$-approximate locally optimal solution to \eqref{eq:SDL_1}, our algorithm needs $O(\eps^{-1})$ iterations.

	Before we state our result, we give some background on BCD-type algorithms in the optimization literature. It is known that the global convergence of BCD to stationary points is not guaranteed when there are more than two blocks (see, e.g., \cite{ kolda2009tensor}), and such a guarantee is known only with some additional regularity conditions \cite{grippof1999globally, grippo2000convergence, bertsekas1999nonlinear}. As illustrated in a counterexample of {Powel \cite{powell1973search}}, the standard BCD  with more than two blocks may result in circulating a set of non-stationary points and may fail to converge to any stationary points even in a subsequential sense. To handle the four-block multi-convex minimization problem in our case (see \eqref{eq:ASDL_1}), we use the recently proposed version of BCD in \cite{lyu2020convergence} that uses an additional `diminishing radius' (BCD-DR) condition in order to guarantee global convergence to stationary points of the objective function for three-block optimization as well as to obtain a worst-case rate of convergence to stationary points. An advantage of using this version of BCD is that a rate of convergence result is available, which can be directly applied to our case of SDL training problem. See Theorem \ref{thm:SDL_BCD} for more details.

	As the main theoretical result in this work, we establish that Algorithm \ref{algorithm:SDL} converges to the stationary points of the SDL objective $L$ in \eqref{eq:ASDL_1} under a mild assumption. Furthermore, we show that Algorithm \ref{algorithm:SDL} converges to an `$\eps$-approximate' solution within $O(\eps^{-1})$ iterations.

	
	
	Now we state our theoretical result that provides a convergence guarantee of Algorithm \ref{algorithm:SDL} to first-order optimal (stationary) points as well as a bound on the iteration complexity. 
	
	\begin{theorem}\label{thm:SDL_BCD}
		Assume \ref{assumption:A4} and the constraint sets $\mathcal{C}^{\textup{dict}}$, $\mathcal{C}^{\textup{code}}$,  $\mathcal{C}^{\textup{beta}}$, $\mathcal{C}^{\textup{aux}}$ introduced in Algorithm \ref{algorithm:SDL} are convex and compact. Let $\bZ_{T} = (\W_{T}, \H_{T}, \Beta_{T},\, \bGamma_{T})$ denote the output of Algorithm \ref{algorithm:SDL}, assuming either the filter-based or feature-based model in \eqref{eq:ASDL_1}. Write  $\Param=\mathcal{C}^{\textup{dict}}\times \mathcal{C}^{\textup{code}}\times \mathcal{C}^{\textup{beta}}\times \mathcal{C}^{\textup{aux}}$. Choose the radii $r_{t}$ such that $\sum_{t=1}^{\infty} r_{t}=\infty$ and $\sum_{t=1}^{\infty}r_{t}^{2}<\infty$. Then for every initial estimate $\bZ_{0}$ and choice of parameters $\xi$ and $\blambda$, the followings hold:
		\begin{description}
			\item[(i)] $\bZ_{t}$ converges to the set of stationary points of $L$ over $\Param$. 
			
			\item[(ii)] For each $T\ge 1$, we have 	
			\begin{align}\label{eq:thm_convergence_bd}
				\min_{1\le k \le T}  \,\, 	\left[ -\inf_{\bZ\in \Param}  \left\langle \nabla L(\bZ_{k}),\, \frac{\bZ - \bZ_{k} }{\lVert \bZ - \bZ_{k}\rVert_{F}} \right\rangle \right]  =  O\left( \left( \sum_{k=1}^{T} r_{k} \right)^{-1} \right).
			\end{align}
			
			\item[(iii)] Suppose $r_{k}=1/(\sqrt{k}\log k)$. Then for each $\eps>0$, an $\eps$-stationary point is achieved within iteration $O(\eps^{-1}(\log \eps^{-1})^{2})$.
			
		\end{description}
	\end{theorem}
	
	Note that since \eqref{eq:SDL_1} is for fitting the SDL model on a fixed training data $(\X_{\textup{data}}, \X_{\textup{aux}}, \Y_{\textup{label}})$, restricting the norms of parameters by some large constant does not lose any generality, so the compactness assumption in Theorem \ref{thm:SDL_BCD} can be enforced for training unconstrained SDL. Moreover, this assumption allows one to put additional nonnegativity constraints to entail supervised nonnegative matrix factorization models \cite{austin2018fully, leuschner2019supervised}. It does not, however, entail supervised PCA models \cite{ritchie2020supervised} or low-rank matrix constraints as the Grassmannian constraint is nonconvex.


	\subsection{Statistical estimation guarantees}
	\label{subsection:statistical_estimation}

	In this subsection, we propose generative models for SDL and provide statistical estimation guarantees of the assumed true parameters.

	\subsubsection{Statistical estimation for weakly constrained filter-based SDL}
	\label{subsection:statistical_estimation_filter}
	
	In this section, we assume a generative model for the filter-based SDL \eqref{eq:SDL_filt_1} and state statistical parameter estimation guarantee. Suppose that the data, auxiliary covariate, and label triples $(\x_{i}, \x_{i}', y_{i})$ are drawn i.i.d. according to the following joint distribution:
	\begin{align}\label{eq:SDL_prob_model_conv_filt}
		&\x_{i} = \B^{\star}[:,i] +\beps_{i}, \quad \x_{i}'=\C^{\star}[:,i]+\beps_{i}',\quad y_{i}\,|\, \x_{i}, \x_{i}' \sim \text{Multinomial}\big(1, \g\left( (\A^{\star})^{T} \x_{i} +  (\bGamma^{\star})^{T}\x_{i}' \right) \big), \\
		&\textup{where}\quad  \A^{\star}\in \R^{p\times \kappa},\,  \B^{\star}\in \R^{p\times n},\, \C^{\star}\in \R^{q\times n},\,    \bGamma^{\star}\in \R^{q\times \kappa},\,\text{s.t. $\rank([\A^{\star},\B^{\star}])\le r$ and $[\A^{\star},\B^{\star},\bGamma^{\star}]\in \Param$},
	\end{align}
	where $\Param\subseteq \R^{p\times \kappa}\times \R^{p\times n} \times \R^{q\times \kappa}$ is a convex constraint set. In the above model, each $\beps_{i}$ (resp., $\beps_{i}'$) are  $p\times 1$ (resp., $q\times 1$) vector of i.i.d. mean zero Gaussian entries with variance $\sigma^{2}$ (resp., $(\sigma')^{2}$). We call the above the \textit{generative filter-based SDL model}. In what follows, we will assume that the noise levels $\sigma,\sigma'$ are known and focus  on estimating $\A^{\star}, \B^{\star}, $ and $\bGamma^{\star}$ with the unknown nuisance parameter $\C^{\star}$. 
	
	
	Denote $\X_{\textup{data}}=[\x_{1},\dots,\x_{n}]\in \R^{p\times n}$, $\X_{\textup{aux}}=[\x_{1}',\dots,\x_{n}']\in \R^{p\times n}$, and $\Y_{\textup{label}}=[y_{1},\dots,y_{n}]\in \{0,\dots,\kappa\}^{n}$. For each particular realization of  $(\X_{\textup{data}}, \X_{\textup{aux}}, \Y_{\textup{label}})$,  the ($L_{2}$-regularized) normalized negative log likelihood of observing it from the model  \eqref{eq:SDL_prob_model_conv_filt} with parameter $[\A,\B,\C, \bGamma]$ and noise level $\sigma,\sigma'$ can be computed as 
	\begin{align}\label{eq:SDL_likelihood_conv_filter}
		\L_{n}(\A, \B, \C,\bGamma)&  :=  \left( - \sum_{i=1}^{n} \sum_{j=0}^{\kappa}\mathbf{1}(y_{i}=j)g_{j}(\A^{T} \x_{i} +  \bGamma^{T} \x_{i}' ) ) \right)   +  \frac{1}{2\sigma^{2}} \lVert \X_{\textup{data}} - \B \rVert_{F}^{2} \\
		&\qquad + \nu\left( \lVert \A \rVert_{F}^{2} + \lVert \bGamma\rVert_{F}^{2} \right)  + \frac{pn\log \sigma}{2} +   \frac{qn\log \sigma'}{2} +  \frac{1}{(2(\sigma')^{2})} \lVert \X_{\textup{aux}}-\C \rVert_{F}^{2}. 
	\end{align}
	Note that the added $L_{2}$ regularizer above can be understood by using a Gaussian prior for the parameters and interpreting the right-hand side above as the negative logarithm of the posterior distribution function (up to a constant). Then estimating true parameters by minimizing the above amounts to the maximum a posterior estimate.  
	
	Note that the problem of estimating $\A$ and $\B$ are coupled due to the low-rank model assumption $\rank([\A,\B])\le r$, while the problem of estimating $\C$ is separable and is not our interest. The joint estimation problem for $[\A,\B,\bGamma]$ is equivalent to the filter-based SDL problem \eqref{eq:SDL_filt_CALE} with tuning parameter $\xi=(2\sigma^{2})^{-1}$ without the $L_{2}$ regularization term for $\A$ and $\bGamma$. This motivates us to estimate the true `SDL parameters' $\A^{\star}, \B^{\star}$, and $\bGamma^{\star}$ as follows: 
	\begin{align}\label{eq:SDL_conv_filter_MLE_estimate}
		[\hat{\A}, \hat{\B}, \hat{\bGamma}]  &\leftarrow \textup{Output of Algorithm \ref{alg:SDL_filt_LPGD} with $\xi=\frac{1}{(2\sigma^{2})^{-1}} $ for $T=O(\log n)$ iterations}.  
	\end{align}

	The following result gives a confidence region for the true combined parameters $[\A^{\star}, \B^{\star},\bGamma^{\star}]$ centered at the above estimate in \eqref{eq:SDL_conv_filter_MLE_estimate}. Roughly speaking, it states that the true parameter $[\A^{\star},\B^{\star},\bGamma^{\star}]$ is within $O(\log n/\sqrt{n})$ the estimate given in \eqref{eq:SDL_conv_filter_MLE_estimate} with high probability, provided that the noise is sufficiently small (that is, $2\sigma^{2} < \frac{12}{nL^{*}}$) and the classification problem is well-conditioned (that is, $L^{*}/\mu^{*}<3$, see Theorem \ref{thm:SDL_LPGD_STAT}). The first condition of small noise variance is reasonable since we are tying to estimate a low-rank matrix $\B^{\star}$ of size $p\times n$ from $n$ samples.  {\color{black} On the other hand, when the latter condition is not satisfied, we can use a sufficiently large $L_{2}$-regularization coefficient $\nu \sim \nu'n$ for some constant $\nu'=O(1)$ to guarantee exponential convergence to the global optimum of \eqref{eq:SDL_likelihood_conv_filter} (i.e., regularized MLE). In this case,  the true parameter is guaranteed to be within $O(1/\sqrt{n})$ the estimate given in \eqref{eq:SDL_conv_filter_MLE_estimate} plus a $O(1)$ term of regularization cost $O(\nu'(\lVert \A^{\star} \rVert_{2} + \lVert \bGamma^{\star} \rVert_{F}))$. }
	

	\begin{theorem}(Statistical estimation for weakly constrained SDL-filt)\label{thm:SDL_LPGD_STAT}
		Assume \ref{assumption:A1}-\ref{assumption:A4} hold. Suppose $(\X_{\textup{data}}, \X_{\textup{aux}}, \Y_{\textup{label}}) \in \R^{p\times n}\times \R^{q\times n} \times \{0,\dots,\kappa\}^{n}$ is generated according to the generative model \eqref{eq:SDL_prob_model_conv_filt} with true parameter $[\A^{\star}, \B^{\star}, \bGamma^{\star}, \C^{\star}] $ such that $\bZ^{\star}:=[[\A^{\star}, \B^{\star}], \bGamma^{\star}]\in \Param$  and $\rank([\A^{\star}, \B^{\star}])\le r$. Let $\bZ_{t}:=[[\A_{t}, \B_{t}], \bGamma_{t}]$ denote the iterates of Algorithm \ref{alg:SDL_filt_LPGD} for the filter-based SDL problem \eqref{eq:SDL_filt_CALE} with tuning parameter $\xi= (2\sigma^{2})^{-1}$ and $L_{2}$-regularization parameter $\nu\ge 0$.  Fix any $\mu^{*}\le \delta^{-}\alpha^{-}$ and $L^{*}\ge \delta^{+}\alpha^{+}$. Let $\mu:=\min(2\xi,\,2\nu+\mu^{*}n)$ and $L:=\max(2\xi,\, 2\nu+ L^{*}n)$ and suppose that
		\begin{align}\label{eq:thm_SDL_filt_cond}
			\frac{L}{\mu}= \frac{\max( \xi,\, \nu + \frac{L^{*} n}{2} )}{\min( \xi, \, \nu +\frac{\mu^{*} n}{2})} < 3.
		\end{align}
		Fix any stepsize $\tau\in ( \frac{1}{2\mu}, \frac{3}{2L})$. Denote $\rho:=2(1-\tau\mu)$.  Then the followings hold:
		\begin{description}
			\item[(i)] Suppose $\bZ^{\star}-\tau \nabla_{\bZ} \mathcal{L}_{n} (\bZ^{\star})\in \Param$. Then for $\eps>0$, there exists an explicit constant $c>0$ and an absolute constant $C>0$, such that for all $t\ge 1$ and all sufficiently large $n\ge 1$,
			\begin{align}
				\P\left( \lVert \bZ_{t} - \bZ^{\star} \rVert_{F} \le \rho^{t}  \, \lVert  \bZ_{0} - \bZ^{\star}\rVert_{F} + \eps(p,n)  + \frac{3\nu}{(1-\rho)L^{*}n} \left( \lVert  \A^{\star}\rVert_{2} + \lVert  \bGamma^{\star}\rVert_{F}\right) \right) \ge 1-\frac{1}{n},
			\end{align}
			where for some 
			\begin{align}\label{eq:def_thm_stat_eps}
				\eps(p,n):=  \frac{3}{2(1-\rho)L^{*}}\left[ c \frac{\log n}{\sqrt{n}}  + 3C\sigma\left( \frac{\sqrt{p}}{n} + \frac{2}{\sqrt{n}}\right)   \right].
			\end{align}
			
			\item[(ii)] Suppose $\bZ^{\star}-\tau \nabla_{\bZ} \mathcal{L}_{n} (\bZ^{\star})\notin \Param$. Then for all $t\ge 1$ and all sufficiently large $n\ge 1$,
			\begin{align}
				\P\left( \lVert \bZ_{t} - \bZ^{\star} \rVert_{F} \le \rho^{t}  \, \lVert  \bZ_{0} - \bZ^{\star}\rVert_{F} + \eps(p,n)\sqrt{\min(p,n)}  + \frac{3\nu}{(1-\rho)L^{*}n} \left( \lVert  \A^{\star}\rVert_{2} + \lVert  \bGamma^{\star}\rVert_{F}\right) \right) \ge 1-\frac{1}{n},
			\end{align}
			where $\eps(p,n)$ is defined in \eqref{eq:def_thm_stat_eps}.
		\end{description}
	\end{theorem}

	Since we consider constrained parameter estimation problem, it is natural to consider two cases depending on whether the true parameter $\bZ^{\star}$ is in the `interior' of the constraint set $\Param$. Indeed, Theorem \ref{thm:SDL_LPGD_STAT} is stated for two cases depending on whether the gradient descent update of the true parameter, $\bZ^{\star} - \tau \nabla \mathcal{L}_{n}(\bZ^{\star})$, still lies inside $\Param$. This is in fact the case in the traditional setting of unconstrained parameter space, i.e., $\Param$ equals the whole space $\R^{p\times \kappa}\times \R^{p\times n} \times \R^{q\times \kappa}$. In this case the corresponding gradient mapping, $G(\bZ^{\star},\tau):=\frac{1}{\tau}\left( \bZ^{\star} - \Pi_{\Param} \left( \bZ^{\star} - \tau \nabla \bar{\mathcal{L}}_{n}(\bZ^{\star}) \right)\right)$, equals the gradient $\nabla \bar{\mathcal{L}}_{n}(\bZ^{\star}) $. Otherwise, the gradient mapping $G(\bZ^{\star}, \tau)$ does not need to equal the gradient $\nabla \bar{\mathcal{L}}_{n}(\bZ^{\star}) $. In this case, we use the crude bound $\lVert G(\bZ^{\star}, \tau)  \rVert_{F}\le \lVert \nabla \bar{\mathcal{L}}_{n}(\bZ^{\star})  \rVert_{F}$ to obtain the desired result with additional $\sqrt{\min(p,n)}$ factor.

	\subsubsection{Statistical estimation for weakly constrained feature-based SDL }
	\label{subsection:statistical_estimation_filter_feat}
	
	Similarly as in the previous section, here we assume a generative model for the feature-based SDL \eqref{eq:SDL_feat_1} and state statistical parameter estimation guarantee. Suppose that the data, auxiliary covariate, and label triples $(\x_{i}, \x_{i}', y_{i})$ are independently drawn according to the following joint distribution:
	\begin{align}\label{eq:SDL_prob_model_conv_feat}
		&\x_{i} = \B^{\star}[:,i] +\beps_{i}, \quad \x_{i}'=\C^{\star}[:,i]+\beps_{i}',\quad y_{i}\,|\, \x_{i}, \x_{i}' \sim \text{Multinomial}\big(1, \g\left( \A^{\star} +  (\bGamma^{\star})^{T}\x_{i}' \right) \big), \\
		&\textup{where}\quad  \A^{\star}\in \R^{\kappa\times n},\,  \B^{\star}\in \R^{p\times n},\,C^{\star}\in \R^{q\times n},\,    \bGamma^{\star}\in \R^{q\times \kappa} \\
		& \text{s.t. $\rank([(\A^{\star})^{T},(\B^{\star})^{T}]^{T})\le r$ and $[\A^{\star},\B^{\star},\bGamma^{\star}]\in \Param$},
	\end{align}
	where $\Param\in \R^{(p+q)\times \kappa}  \times \R^{q\times \kappa}$ is a convex constraint set. As before, each $\beps_{i}$ (resp., $\beps_{i}'$) are  $p\times 1$ (resp., $q\times 1$) vector of i.i.d. mean zero Gaussian entries with variance $\sigma^{2}$ (resp., $(\sigma')^{2}$). We call the above the \textit{generative feature-based SDL model}. In what follows, we will assume that the noise levels $\sigma,\sigma'$ are known and focus  on estimating the SDL parameters $\A^{\star}, \B^{\star}$, and $\bGamma^{\star}$. Note that the samples $(\x_{i},\x_{i},y_{i})$ are assumed to be independent but not necessarily identically distributed, since the means $\B^{\star}[:,i]$ and $\C^{\star}[:,i]$ may depend on the sample index $i$. 
	
	For each particular realization of the observed data of size $n$, the ($L_{2}$-regularized) normalized negative log likelihood of observing it from the model  \eqref{eq:SDL_prob_model_conv_feat} with parameter $[\A,\B, \C,\bGamma]$ and noise level $\sigma,\sigma'$ can be computed as 
	\begin{align}\label{eq:SDL_likelihood_conv_feature}
		\L_{n}(\A, \B, \C,\bGamma)&  :=  \left( - \sum_{i=1}^{n} \sum_{j=0}^{\kappa}\mathbf{1}(y_{i}=j)g_{j}(\A +  \bGamma^{T} \x_{i}' ) ) \right)   +  \frac{1}{2\sigma^{2}} \lVert \X_{\textup{data}} - \B \rVert_{F}^{2} \\
		&\qquad + \nu\left( \lVert \A \rVert_{F}^{2} + \lVert \bGamma\rVert_{F}^{2} \right)  + \frac{pn\log \sigma}{2} +   \frac{qn\log \sigma'}{2} +  \frac{1}{(2(\sigma')^{2})} \lVert \X_{\textup{aux}}-\C \rVert_{F}^{2}. 
	\end{align}
	Similarly as before, we can estimate the true SDL parameters $\A^{\star}, \B^{\star}, \bGamma^{\star}$ as follows: 
	\begin{align}\label{eq:SDL_conv_feat_MLE_estimate}
		[\hat{\A}, \hat{\B}, \hat{\bGamma}]  &\leftarrow \textup{Output of Algorithm \ref{alg:SDL_feat_LPGD} with $\xi=\frac{1}{2\sigma^{2}} $ for $T=O(\log n)$ iterations}. 
	\end{align}

	The following result gives a confidence region for the true combined parameters $[\A^{\star}, \B^{\star},\bGamma^{\star}]$ centered at the above estimate in \eqref{eq:SDL_conv_feat_MLE_estimate}, which is analogous to Theorem \ref{thm:SDL_LPGD_STAT} for the generative filter-based SDL model.

	\begin{theorem}(Statistical estimation for weakly constrained SDL-feat)\label{thm:SDL_LPGD_STAT_feat}
		Assume \ref{assumption:A1}-\ref{assumption:A2} and \ref{assumption:A4} hold.  Suppose $(\X_{\textup{data}}, \X_{\textup{aux}}, \Y_{\textup{label}}) \in \R^{p\times n}\times \R^{q\times n} \times \{0,\dots,\kappa\}^{n}$ is  generated according to  \eqref{eq:SDL_prob_model_conv_feat} with true parameter $[\A^{\star}, \B^{\star}, \bGamma^{\star}, \C^{\star}] $ such that $\bZ^{\star}:=[[(\A^{\star})^{T}, (\B^{\star})^{T}]^{T}, \bGamma^{\star}]\in \Param$  and  $\rank([(\A^{\star})^{T}, (\B^{\star})^{T}]^{T})\le r$. Let $\bZ_{t}:=\left[ \begin{bmatrix}\A_{t}^{T},\,  \B_{t}^{T}\end{bmatrix}^{T} ,\bGamma_{t}\right]$ denote the iterates of Algorithm \ref{alg:SDL_feat_LPGD} with $\xi= (2\sigma^{2})^{-1}$. Denote $\lambda_{\max}:=\lambda_{\max}( n^{-1}\X_{\textup{aux}}\X_{\textup{aux}}^{T})$. Let $\mu:=\min(2\xi,\,2\nu+ \alpha^{-} )$, $L:=\max( 2\xi, 2\nu +\alpha^{+}\lambda_{\max} n , \, \alpha^{+} + 2\nu)$ and suppose that
		\begin{align}\label{eq:thm_SDL_LGPD_cond_feat_STAT}
			\frac{L}{\mu}= \frac{\max( 2\xi, 2\nu +\alpha^{+}\lambda_{\max} n , \, \alpha^{+} + 2\nu) }{ \min(2\xi,\, \alpha^{-}+2\nu )} < 3.
		\end{align} 
		Fix any stepsize $\tau\in ( \frac{1}{2\mu}, \frac{3}{2L})$. Denote $\rho:=2(1-\tau\mu)$.  Then the followings hold:
		\begin{description}
			\item[(i)] Suppose $\bZ^{\star}-\tau \nabla_{\bZ} \mathcal{L}_{n} (\bZ^{\star})\in \Param$. Then for $\eps>0$, there exists an explicit constant $c>0$ and an absolute constant $C>0$, such that for all $t\ge 1$ and all sufficiently large $n\ge 1$,
			\begin{align}
				\P\left( \lVert \bZ_{t} - \bZ^{\star} \rVert_{F} \le \rho^{t}  \, \lVert  \bZ_{0} - \bZ^{\star}\rVert_{F} + \eps(p,n)  + \frac{3\nu}{(1-\rho)L} \left( \lVert  \A^{\star}\rVert_{2} + \lVert  \bGamma^{\star}\rVert_{F}\right) \right) \ge 1-\frac{1}{n},
			\end{align}
			where 
			\begin{align}\label{eq:def_thm_stat_epsfeat}
				\eps(p,n):=  \frac{3}{2(1-\rho)L}\left[ c \log n  + 3C \left( \sigma + \frac{\gamma_{\max}}{\sqrt{\log 2}} \right) \left( \sqrt{p}  + 2\sqrt{n}\right)   \right].
			\end{align}
			
			\item[(ii)] Suppose $\bZ^{\star}-\tau \nabla_{\bZ} \mathcal{L}_{n} (\bZ^{\star})\notin \Param$. Then for all $t\ge 1$ and all sufficiently large $n\ge 1$,
			\begin{align}
				\P\left( \lVert \bZ_{t} - \bZ^{\star} \rVert_{F} \le \rho^{t}  \, \lVert  \bZ_{0} - \bZ^{\star}\rVert_{F} + \eps'(p,n)\sqrt{\min(p,n)}  + \frac{3\nu}{(1-\rho)L^{*}n} \left( \lVert  \A^{\star}\rVert_{2} + \lVert  \bGamma^{\star}\rVert_{F}\right) \right) \ge 1-\frac{1}{n},
			\end{align}
			where for the same constants $c,C>0$ as in \textbf{\textup{(i)}},
			\begin{align}\label{eq:def_thm_stat_eps'}
				\eps'(p,n):=  \frac{3}{2(1-\rho)L}\left[ c \log n  + 3C \left( \sigma + \frac{\gamma_{\max}}{\sqrt{\log 2}} \right) \left( \sqrt{p}  + 2\sqrt{n}\right)\sqrt{\min(p,n)}   \right].
			\end{align}
		\end{description}
	\end{theorem}
	
	The scaling of the statistical error terms $\eps(p,n)$ and $\eps'(p,n)$ for the generative feature-based model in Theorem \ref{thm:SDL_LPGD_STAT_feat} is different from that for the generative filter-based model in Theorem \ref{thm:SDL_LPGD_STAT_feat}. For the filter-based model, one has to have $\xi=(2\sigma^{2})^{-1}$ comparable to the linear term $L^{*}n$, and depending on whether the prediction model is well-conditioned ($L^{*}/\mu^{*}$<3), one can have zero or linear $L_{2}$-regularization parameter $\nu$. On the other hand, for the feature-based model in Theorem \ref{thm:SDL_LPGD_STAT_feat}, the requirement for the noise variance could be weaker when no auxiliary covariate is used ($\lambda_{\max}=0$). In this case, the hypothesis \eqref{eq:thm_SDL_LGPD_cond_feat} becomes \eqref{eq:thm_SDL_LGPD_cond_feat2}, which reads 
	\begin{align}
		\max\left( \frac{\alpha^{+}}{4} ,\, \frac{\alpha^{+}-12(2\sigma^{2})^{-1} \alpha^{-}}{4}\right)<\nu < \frac{(2\sigma^{2})^{-1}}{3}.
	\end{align}
	Hence $\xi=(2\sigma^{2})^{-1}$ need not grow linearly in $n$ as in the filter-based case. However, when auxiliary covariate is used ($\lambda_{\max}>0$), then $L\ge \alpha^{+}\lambda_{\max} n$, so both $\xi=(2\sigma^{2})^{-1}$ and $\nu$ should grow linearly in $n$, which is a similar requirement as for the filter-based case when the classification model is not well-conditioned ($L^{*}/\mu\ge 3$), see Theorem \ref{thm:SDL_LPGD_STAT}.

	\subsubsection{Statistical estimation for strongly constrained filter-based SDL} 
	\label{subsection:statistical_estimation_BCD}
	
	In this subsection, we introduce a generative model closely related to the strongly constrained filter-based SDL model in \eqref{eq:ASDL_1}, where we seek to estimate individual matrix parameters $\W,\H,\Beta,\bGamma$, instead of the combined parameters $\A=\W\Beta$, $\B=\W\H$, and $\bGamma$ as we considered in Subsections \ref{subsection:statistical_estimation_filter}. 
	
	Suppose that the data, auxiliariy covariate, and label triples $(\x_{i}, \x_{i}', y_{i})$ are drawn i.i.d. according to the following generative model:
	\begin{align}\label{eq:SDL_prob_BCD_filt}
		&\x_{i} \sim  N\left( \W^{\star}\h^{\star}, \sigma^{2}\I_{p} \right), \quad \x_{i}'\sim N(\blambda^{\star}, (\sigma')^{2}\I_{q}),\\
		& y_{i}\,|\, \x_{i}, \x_{i}' \sim \text{Multinomial}\big(1, \g\left( (\Beta^{\star})^{T} (\W^{\star})^{T} \x_{i}+ (\bGamma^{\star})^{T}\x_{i}'\right) \big), \\
		&\textup{where}\quad  \W^{\star}\in \R^{p\times r},\,  \h^{\star}\in \R^{r\times 1},\,  \Beta^{\star}\in  \R^{r\times \kappa},\,  \bGamma^{\star}\in \R^{q\times \kappa},\,\blambda^{\star}\in \R^{q\times 1} \,\, \text{s.t. $[\W^{\star},\h^{\star},\Beta^{\star},\bGamma^{\star}]\in \Param$}.
	\end{align}
	Here $\Param= \mathcal{C}^{\textup{dict}}\times \mathcal{C}^{\textup{code}}\times  \mathcal{C}^{\textup{beta}} \times \mathcal{C}^{\textup{aux}}$ is the product of convex constraint sets on individual factors. We assume $(\x_{i},\x_{i}',y_{i})$ for $i=1,\dots,n$ are independent  and also $\x_{i}$ and $\x_{i}'$ are independent for each $1\le i \le n$. We call the above the \textit{generative filter-based SDL model}. Assuming that $\sigma$ and $\sigma'$ are known, our goal is to estimate the true factors $\W^{\star}$, $\h^{\star}$, $\Beta^{\star}, \bGamma^{\star}$, and $\blambda^{\star}$ from an observed sample $(\x_{i},\x'_{i},y_{i})$, $i=1,\dots,n$ of size $n$. Note that unlike the generative model \eqref{eq:SDL_prob_model_conv_feat} for weakly constrained SDL, here we employ a stronger assumption that the samples  $(\x_{i},\x'_{i},y_{i})$ are identically distributed. This is for a technical reason that analyzing statistical properties of the corresponding MLE for the strongly constrained SDL is more challenging than that for the weakly constrained case. 
	
	We consider the maximum likelihood estimation framework with $L_{2}$-regularization of the parameters. Namely, we estimate the true parameter as the minimizer of the following loss function 
	\begin{align}\label{eq:SDL_likelihood_BCD_filter_full}
		L(\W, \h, \Beta, \bGamma,\blambda) &  :=  \mathcal{L}_{n}(\W,\h,\Beta, \bGamma)  + \frac{pn\log \sigma}{2} +   \frac{qn\log \sigma'}{2} +  \frac{1}{2(\sigma')^{2}} \sum_{i=1}^{n} \lVert \x_{i}'-\blambda \rVert^{2}, \\
	\end{align}
	where we define 
	\begin{align}\label{eq:SDL_likelihood_BCD_filter}
		\mathcal{L}_{n}(\W,\h,\Beta, \bGamma)&:=-\sum_{j=0}^{\kappa}\mathbf{1}(y_{i}=j)g_{j}(\Beta^{T}\W^{T}\x_{i} +  \bGamma^{T} \x_{i}' ) )   +  \frac{1}{2\sigma^{2}} \lVert \X_{\textup{data}} - \W[\h,\dots,\h] \rVert_{F}^{2} \\
		&\qquad + n\nu\left( \lVert \W \rVert_{F}^{2} + \lVert \h \rVert_{F}^{2} + \lVert \blambda \rVert_{F}^{2} + \lVert \bGamma\rVert_{F}^{2} \right).
	\end{align}
	
	Similarly as before, we estimate the true parameters $\W^{\star}, \h^{\star}, \Beta^{\star}, \bGamma^{\star}$, and $\blambda^{\star}$ as follows: 
	\begin{align}\label{eq:SDL_conv_feat_MLE_estimate}
		\begin{cases}
			\hat{\bZ}:=[\hat{\W}, \hat{\h}, \hat{\Beta}, \hat{\bGamma}]  &\leftarrow \textup{Output of Algorithm \ref{algorithm:SDL} with the objective  $\mathcal{L}_{n}$ in \eqref{eq:SDL_likelihood_BCD_filter}}  \\
			& \qquad \textup{ for $T=O(\log n^{-1})$ iterations with $r_{k}=1/(\sqrt{k}\log k)$},  \\
			\hat{\blambda} &\leftarrow \bar{\x}':=\frac{1}{n} \sum_{i=1}^{n} \x_{i}'.
		\end{cases}
	\end{align}
	Note that $\hat{\bZ}$ above is obtained by approximately minimizing the regularized negative log likelihood $\mathcal{L}_{n}$ defined in \eqref{eq:SDL_likelihood_BCD_filter}. Our main result in this section considers estimation accuracy of the true parameter $\bZ^{\star}:=[ \W^{\star}, \h^{\star}, \Beta^{\star}, \bGamma^{\star}]$ via the approximate regularized MLE $\hat{\bZ}$.
	
	Let  $\mathcal{L}(\bZ)=\mathcal{L}_{n}(\W,\h,\Beta,\bGamma)$ denote the objective in \eqref{eq:SDL_likelihood_BCD_filter}. Define the expected regularized negative log likelihood function as
	\begin{align}\label{eq:def_L_bar_STAT}
		\bar{\mathcal{L}}(\bZ) := \E_{(\x,\x',y)}\left[ \mathcal{L}(\bZ) \right].
	\end{align}
	In the classical local consistency theory of maximum likelihood estimator (e.g., \cite{fan2001variable}), it is crucial to assume that the expected negative log-likelihood function $\bar{\mathcal{L}}$ without regularization ($\nu=0$) is strongly convex at the true parameter $\bZ^{\star}$. Equivalently, this is to say that the \textit{Fisher information}, which is the Hessian $\nabla^{2} \bar{\mathcal{L}}$ of the expected negative  log likelihood function (still with $\nu=0$) evaluated at $\bZ^{\star}$ is positive definite. However, as we will discuss shortly, this is not the case for the generative filter-based SDL model in \eqref{eq:SDL_prob_BCD_filt}. This can be seen since the equivalent CALE formulation \eqref{eq:SDL_filt_CALE} does not have identifiability in general, see \eqref{eq:identifiability}.
	
	In order to circumvent this issue, we use additional $L_{2}$-regularizer with coefficient $\nu$ in order to make the Fisher information is positive definite after regularization.  To be precise, we explicitly compute the Fisher information as follows. According to Lemma \ref{lem:SDL_filt_BCD_STAT_derivatives} in Section \ref{section:appendix_SDL_BCD_STAT}, it turns out the the (regularized) Fisher information  $\nabla^{2} \bar{\mathcal{L}}(\bZ^{\star})$ has the following $4\times 4$ symmetric block structure
	\begin{align}\label{eq:expected_hessian_block}
		\nabla^{2} \bar{\mathcal{L}}(\bZ^{\star})= 
		\begin{blockarray}{ccccc}
			& \vect(\W)^{T} & \h^{T} & \vect(\Beta)^{T} & \vect(\bGamma)^{T} \\
			\begin{block}{c[cccc]}
				\vect(\W) & A_{11} & A_{12}  & A_{13} & O\\
				\h & A_{21}  &  A_{22} & O & O\\
				\vect(\Beta) & A_{31} & O  & A_{33}& A_{34}\\
				\vect(\bGamma)  & O & O   & A_{43} & A_{44} \\
			\end{block}
		\end{blockarray} 
		+ 2\nu \I,
	\end{align}
	where the explicit formulas for the blocks are given in the appendix. Our key observation here is that, if $\nu$ is large enough so that the `$\nu$-regularized' Fisher information $\nabla^{2} \bar{\mathcal{L}}(\bZ^{\star})$ is `block diagonally dominant', that is,
	\begin{align}\label{eq:block_diag_dominant}
		\lambda_{\min}(A_{ii}) + 2\nu > \sum_{j\ne i} \lVert A_{ij} \rVert_{2} \quad \forall 1\le i \le 4,
	\end{align}
	then it is indeed positive definite, see \cite{feingold1962block}). An explicit sufficient condition that implies the above is given in \eqref{eq:nu_min_BCD_STAT_explicit} in Lemma \ref{lem:SDL_filt_BCD_derivatives} in the appendix. 
	We now give the main result in this section below.

	\begin{theorem}(Algorithmic and Statistical estimation for strongly constrained SDL-filt)\label{thm:SDL_BCD_STAT_filt}
		Assume \ref{assumption:A4} hold and that  the constraint sets $\mathcal{C}^{\textup{dict}}$, $\mathcal{C}^{\textup{code}}$,  $\mathcal{C}^{\textup{beta}}$, $\mathcal{C}^{\textup{aux}}$ in \eqref{eq:SDL_prob_BCD_filt} are convex and compact. Suppose $(\X_{\textup{data}}, \X_{\textup{aux}}, \Y_{\textup{label}}) \in \R^{p\times n}\times \R^{q\times n} \times \{0,\dots,\kappa\}^{n}$ is generated according to \eqref{eq:SDL_prob_BCD_filt} with true parameter  $[\W^{\star},\h^{\star},\Beta^{\star},\bGamma^{\star}]\in \Param$.  Then the following hold:
		\begin{description}
			\item[(i)] (Algorithmic convergence guarantee) Let $\bZ_{t}:=[\W_{t}, \H_{t}, \blambda_{t} ,\bGamma_{t}]$ denote the iterates of Algorithm \ref{algorithm:SDL} (BCD-DR) with the objective  $\mathcal{L}_{n}$ in \eqref{eq:SDL_likelihood_BCD_filter}. Then $\bZ_{t}$ converges almost surely to the set of stationary points of $\mathcal{L}_{n}$ over $\Param$ as $t\rightarrow \infty$. Furthermore, an $\eps$-stationary point is reached within  $O(\eps^{-1}(\log \eps^{-1})^{2})$ iterations for each $\eps>0$. 
			
			\vspace{0.1cm}
			\item[(ii)] (Regularized local consistency) 
			If $\nu$ is larger than some explicit constant, then  $\mathcal{L}_{n}$ admits a local minimizer $\tilde{\bZ}$ over $\Param$ near $\param^{\star}$ with high probability. More explicitly, there exists constants $M,\lambda_{*},c_{1},c_{2}>0$ such that whenever 
			\begin{align}
				\nu \ge \lambda_{*} \quad: \textup{sufficient regularization (see \eqref{eq:nu_min_BCD_STAT_explicit}); and} \\
				\frac{n}{n^{-1/2}+2\nu \lVert \bZ^{\star} \rVert_{F}} \ge \frac{4MC}{3\nu} \quad: \textup{sufficient sample size}, 
			\end{align}
			then we have 
			\begin{align}\label{eq:nonasymptotic_consistency_thm}
				\P\left( \lVert \tilde{\bZ} - \bZ^{\star} \rVert_{F} \le C(n^{-1/2} + 2\nu \lVert  \bZ^{\star} \rVert_{F})  \right) \ge 1 - c_{1} \exp\left(-\frac{(C\lambda_{*}-4)^{2}}{32} \right)  - \frac{c_{2}}{\sqrt{n}}.
			\end{align}
		\end{description}
	\end{theorem}

	The proof of Theorem \ref{thm:SDL_BCD_STAT_filt} \textbf{(i)} is similar to that of Theorem \ref{thm:SDL_BCD}. Namely, by an extensive and explicit computation of the Hessian of the loss function $\mathcal{L}_{n}$ in \eqref{eq:SDL_prob_BCD_filt} and apply the general convergence result of BCD-DR in \cite{lyu2020convergence}. On the other hand, we recall that the vanilla Fisher information for the generative filter-base SDL model in \eqref{eq:SDL_prob_BCD_filt} is not necessarily positive definite, but it is indeed positive definite after adding a large enough $L_{2}$ regularization term. Hence the classical local consistency theory of MLE does not immediately apply here. Indeed, our proof of Theorem \ref{thm:SDL_BCD_STAT_filt} \textbf{(ii)} relies on a substantial (non-asymptotic) generalization of such theory that we establish in Section \ref{subsection:MLE}.

	We also remark that it is standard in the literature \cite{fan2001variable} to establish the asymptotic normality of the sequence of $\sqrt{n}$-consistent MLE as the sample size tends to infinity. Indeed, by generalizing the standard argument, we can establish such asymptotic normality of the MLE in the constrained setting either by assuming the true parameter is in the interior of the parameter space or restricting onto the coordinates of the true parameter that lies in the interior of the parameter space. However, it does not seem that such a result can be established for the generative SDL models we consider here. Recall that a premise of such asymptotic normality statement is that we can use vanishing regularization (i.e., $\nu=o(1)$) so that the consistency bound \eqref{eq:nonasymptotic_consistency_thm} becomes precise as $n\rightarrow \infty$. However, for the generative SDL model we discuss in this work, we saw that the true parameter $\bZ^{\star}$ does not yield positive definite Fisher information, so we have to use $L_{2}$-regularization coefficient $\nu>0$ that is bounded away from $0$ even if we increase the sample size.

	Lastly in this section, we remark that an analogous generative approach can be taken for the feature-based SDL model. This model without auxiliary covariate was considered in \cite{mairal2008supervised}. Unlike for the filter-based case, we would need to formulate a latent-variable model where the `code matrix' $\H=[\h_{1},\dots,\h_{n}]$ is the latent variable, and our general theory of non-asymptotic local consistency of constrained and regularized MLE applies only approximately. In Section \ref{subsection:MLE}, we give more detailed discussion on this approach and a sketch of proof for a non-asymptotic local consistency result analogous to Theorem \ref{thm:SDL_BCD_STAT_filt}.
	

	\subsection{A non-asymptotic local consistency of constrained and regularized MLE}
	\label{subsection:MLE}

	In this section, we provide a general result on the non-asymptotic local consistency of MLE in a general setting, where the unknown true parameter used for a generative model may lie on the boundary of the parameter space and the Fisher information at the true parameter is not necessarily positive definite. The result we present in this section is general and could be of independent interest. 
	
	Suppose $\pi_{\param}$ is a probability distribution on $\R^{d}$ parameterized by $\param\in \Param\subseteq \R^{p}$. For a given set of $n$ vectors $\x_{1},\dots,\x_{n}\in \R^{p}$, consider the following regularized maximum likelihood estimation problem:  
	\begin{align}\label{eq:def_MLE_problem}
		\hat{\param}_{n}:= \argmin_{\param\in \Param} \left[ \mathcal{L}(\X=[\x_{1},\dots,\x_{n}]\,;\, \param):=\left( -\sum_{i=1}^{n} \log \pi_{\param}(\x_{i}) \right) + n R_{n}(\param)\right],
	\end{align}
	where $R_{n}(\param)$ is a suitable choice of regularizer for parameter $\param$. 
	
	Now suppose there is true and unknown parameter $\param^{\star}\in \Param$ such that we have i.i.d. samples $\x_{1},\dots,\x_{n}$ from $\pi_{\param^{\star}}$. In this case, $\hat{\param}_{n}$ in \eqref{eq:def_MLE_problem} is a minimizer of the random loss function $\mathcal{L}$ over the constraint set $\Param$, which we call called the \textit{constrained and regularized maximum likelihood estimator} (MLE) of $\param^{\star}$. Note that here we consider a general constrained MLE problem, meaning that the constraint set $\Param$ can be a proper convex subset of the whole space $\R^{p}$ and $\param^{\star}$ could be at the boundary of $\Param$. We also consider a general setting where the loss function $\mathcal{L}$ in \eqref{eq:def_MLE_problem} may be nonconvex in $\param$. 
	
	In this general setting, we would like to provide a high-probability guarantee that there exists a local minimizer of \eqref{eq:def_MLE_problem} that is close to the true parameter $\param^{\star}$. When $\Param$ is taken to be the full space $\R^{p}$ or $\param^{\star}$ is assumed to be in the interior of $\Param$, this type of result is provided by the classical local consistency theory of MLE \cite{fan2001variable} in an asymptotic setting where the sample size $n$ tends to infinity. Below in Theorem \ref{thm:MLE_local_consistency}, we generalize this classical result in the non-asymptotic, constrained, and regularized setting. 
	

	\begin{theorem}
		[Non-asymptotic local consistency of constrained and regularized MLE]
		\label{thm:MLE_local_consistency}
		Consider the constrained and regularized MLE problem \eqref{eq:def_MLE_problem} with unknown true parameter $\param^{\star}$ from a convex subset $\Param\subseteq \R^{p}$. Assume the following holds:
		\begin{description}
			\item[(a1)] (Smoothness) For each realization of the data $\X=[\x_{1},\dots,\x_{n}]\in \R^{p\times n}$, the function $\param\mapsto \mathcal{L}(\X\,;\, \param)$ is three-times continuously differentiable and $R_{n}(\param)$ is differentiable. Furthermore, we have $\E_{\x\sim \pi_{\param^{\star}}}[\lVert \nabla \mathcal{L}(\x; \param^{\star}) \rVert^{3}]<\infty$.  
			
			\item[(a2)] (First-order optimality) The true parameter $\param^{\star}$ is a stationary point of the expected likelihood function $\bar{\mathcal{L}}_{0}(\param):=\E_{X\sim \pi_{\param^{\star}}}  \left[ -\log \pi_{\param}(X) \right]$ over $\Param$. That is, 
				\begin{align}
					\langle \nabla \bar{\mathcal{L}}_{0}(\param^{\star}),\, \param -\param^{\star} \rangle \ge 0 \quad \forall \param\in \Param.
				\end{align}

			\item[(a2)] (Approximate second-order optimality) Let $\bar{\mathcal{L}}(\param):=\E_{X\sim \pi_{\param^{\star}}}  \left[ -\log \pi_{\param}(X) + R_{n}(\param) \right]$ denote the  expected regularized negative log likelihood function. Then the regularized Fisher information  $\nabla^{2}\bar{\mathcal{L}}(\param)$ is positive definite at $\param=\param^{\star}$ with minimum eigenvalue $\lambda_{*}>0$. 
		\end{description}
		Fix $n\ge 1$ and define $\alpha_{n}:=n^{-1/2} + \lVert\nabla R_{n}(\param^{\star}) \rVert$ and  $M:= \underset{1\le i,j,k\le p}{\max} \,\, \underset{\param\in \Param,\, \lVert \param-\param^{\star} \rVert=C\alpha_{n} }{\sup} \left|  \frac{ \partial^{3} \mathcal{L}(\X\,;\,\param)}{\partial \param_{i} \partial \param_{j} \partial \param_{k}} \right|$. Suppose $n$ is large enough so that $n/\alpha_{n} \ge \frac{2MC}{3\lambda_{*}}$. Then there are constants $c_{1},c_{2}>0$ such that 
	\begin{align}\label{eq:local_consistency_thm}
				\P\left( \inf_{\substack{\param\in \Param \\ \lVert \param-\param^{\star} \rVert_{F} = C\alpha_{n}} } \, 	\mathcal{L}(\X\,; \,\param)  -  \mathcal{L}(\X\,;\, \param^{\star})  >0  \right) \ge 1 - c_{1} \exp\left(-\frac{(C\lambda_{*}-4)^{2}}{32\sqrt{p}} \right)  - \frac{c_{2}}{\sqrt{n}},
		\end{align} 
		That is, with high probability explicitly depending on $C$, $\lambda_{*}$, and $n$, there exists a local maximizer of $\param\mapsto \mathcal{L}(\X\,;\,\param)$ within distance $C \alpha_{n}$  from $\param^{\star}$. 
	\end{theorem}

	\section{Numerical validation of convergence analysis}
	
	In this section, we numerically validate our theoretical convergence results of various convex and nonconvex SDL training algorithms as stated in Theorems \ref{thm:SDL_LPGD}, \ref{thm:SDL_LPGD_feat}, and \ref{thm:SDL_BCD}. For the experiment, we use a semi-synthetic MNIST dataset ($p=28^{2}=784$, $q=0$, $n=500$ $\kappa=1$) and fake job postings dataset ($p=,2840$, $q=72$, $n=17,880$ $\kappa=1$), which will be described in detail in Sections \ref{sec:MNIST_setup} and \ref{sec:fakejob_setup}, respectively.

	We first validate our theoretical exponential convergence results of the convex SDL algorithms using Figures \ref{fig:benchmark_MNIST} and \ref{fig:benchmark_fakejob} left. Note that the convexity and smoothness parameters $\mu$ and $L$ in Theorems \ref{thm:SDL_LPGD} and \ref{thm:SDL_LPGD} are difficult to compute exactly in practice, in which case cross-validation of hyperparameters are usually employed. For $\xi\in \{ 0.1, 1 \}$ in Figures \ref{fig:benchmark_MNIST} and \ref{fig:benchmark_fakejob} left, we indeed observe linear decay of training loss as dictated by our theoretical results both for the filter- and the feature-based convex SDL algorithms. Increasing the tuning parameter $\xi$ further to $5$ and $10$, we still observe overall linear convergence but superlinear decay in shorter time scales.  
	
	\begin{figure}[h!]
		\centering
		\includegraphics[width=1\linewidth]{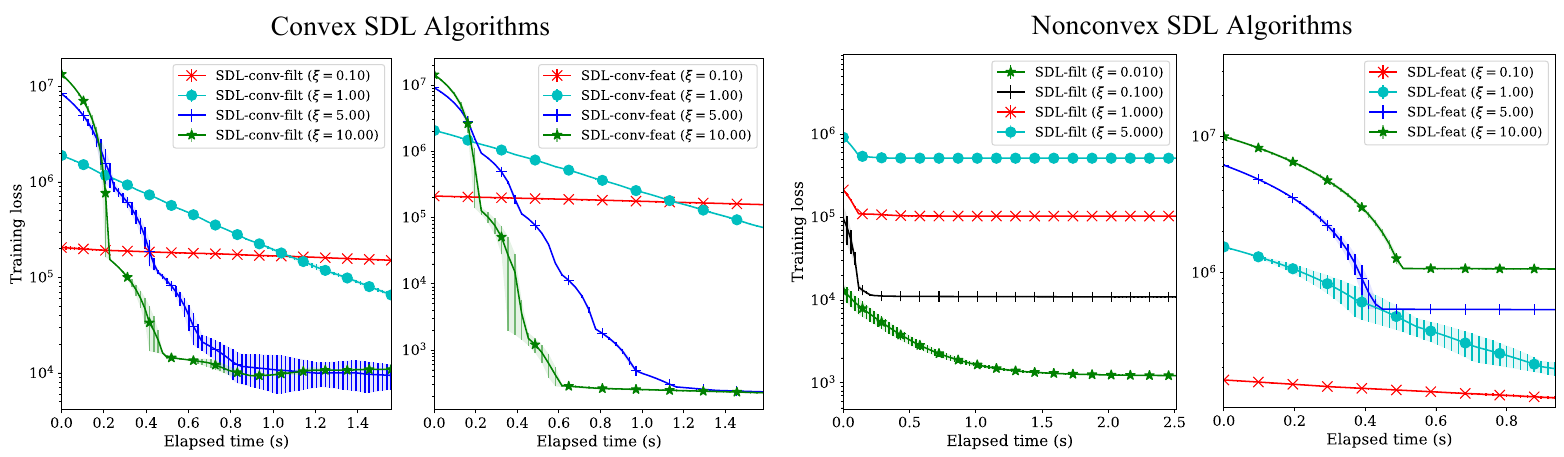} 
		\caption{ Training loss versus elapsed CPU time for Algorithms \ref{alg:SDL_filt_LPGD}, \ref{alg:SDL_feat_LPGD}, and \ref{algorithm:SDL} on the semi-synthetic MNIST dataset  ($p=28^{2}=784$, $q=0$, $n=500$ $\kappa=1$) for various choices of the nuance parameter $\xi$ in log scale. For convex SDL algorithms we used $L_{2}$-regularization coefficient $\nu=2$ and fixed stepsize $\tau=0.01$.	We report the average training loss over five runs and the shades indicate the standard deviation. }
		\label{fig:benchmark_MNIST}
	\end{figure}
	
	Turning our attention to the nonconvex SDL algorithms based on block coordinate descent (see Algorithm \ref{algorithm:SDL}), we are expected to see at least a polynomial rate of decay as stated in Theorem \ref{thm:SDL_BCD}, which should appear as asymptotically decreasing concave curves in the log-plot in Figures \ref{fig:benchmark_MNIST} and \ref{fig:benchmark_fakejob} right. We observe that this is the case for all instances of SDL-filter with no $L_{2}$-regularization. For SDL-feature, we verify similar convergence behavior for both datasets. 
	
	In addition, we report that SDL-feature may converge faster than SDL-filter, especially with large problem dimension $p$ (see the discussion in Section \ref{sec:alg_complexity}). However, it seems the convergence behavior for SDL-filter is more stable than that of SDL-feature. We observe that SDL-feature with tuning parameter $\xi=0.1$ or smaller on the fake job posting dataset may have increasing training loss at the beginning but later it settles down for decreasing training loss, which still confirms asymptotic convergence in Theorem \ref{thm:SDL_BCD}. We observe better convergence of SDL-feature on the real dataset when using nonzero $L_{2}$-regularization.



	\begin{figure}[h!]
		\centering
		\includegraphics[width=1\textwidth]{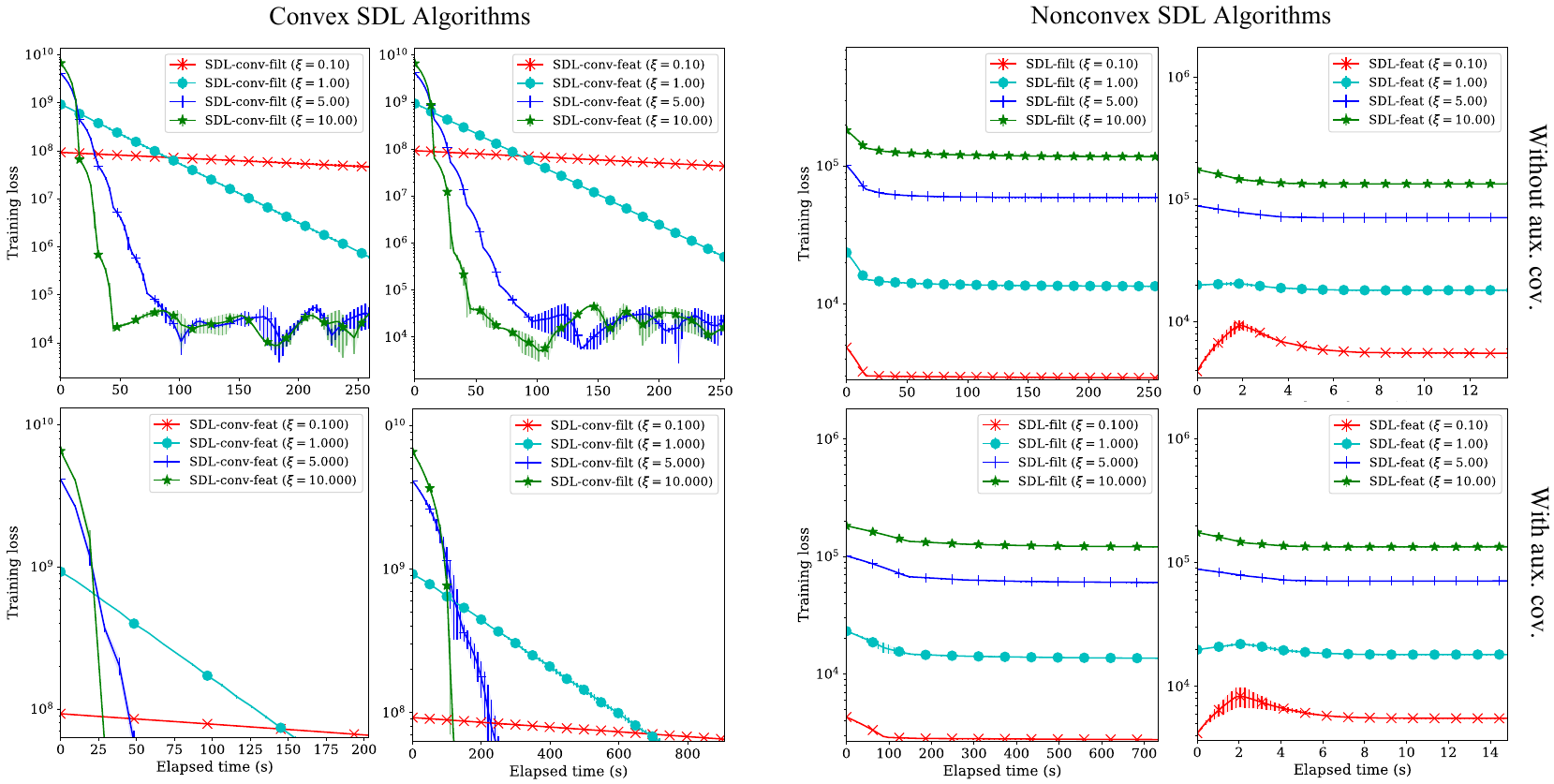}
		\caption{ Training loss versus elapsed CPU time for Algorithms \ref{alg:SDL_filt_LPGD}  and \ref{alg:SDL_feat_LPGD} on the fake job postings dataset  ($p=,2840$, $q=72$, $n=17,880$ $\kappa=1$) for various choices of the nuance parameter $\xi$ in log scale. Top row does not use auxiliary covariates and the bottom row does for training. For convex SDL algorithms as well as SDL-feature, we used $L_{2}$-regularization coefficient $\nu=2$. We used fixed stepsize of $\tau=0.01$ for the convex SDL algorithms.  We report the average training loss over five runs and the shades indicate the standard deviation.  }
		\label{fig:benchmark_fakejob}
	\end{figure}

	\section{Simulation studies}
	\label{sec:MNIST_setup}
	
	In this section, we illustrate our methods on a simulated data set based on the MNIST database of handwritten digits. Recall that the MNIST dataset consists of total 70000 black-and-white images of size $28\times 28=784$ pixels, corresponding to one of the 10 digits from $\{0,1,\dots,9\}$. We will synthesize a dataset of labeled images where the best reconstructive and discriminative dictionaries are known and distinct so that the effect of supervising dictionary learning can be seen vividly. 
	
	\subsection{ Experiment set-up }
	Roughly speaking we synthesize an image by a random linear combination of digits `2' and `5' and assign the label of 1 if it is `more similar' to digit `4' than to digit `7', and otherwise assign the label 0. The similarity of an image to a given image can be quantified by taking the convolution, the sum of the entry-wise product of pixel values. 
	
	\begin{figure}[h!]
		\centering
		\includegraphics[width=0.9\linewidth]{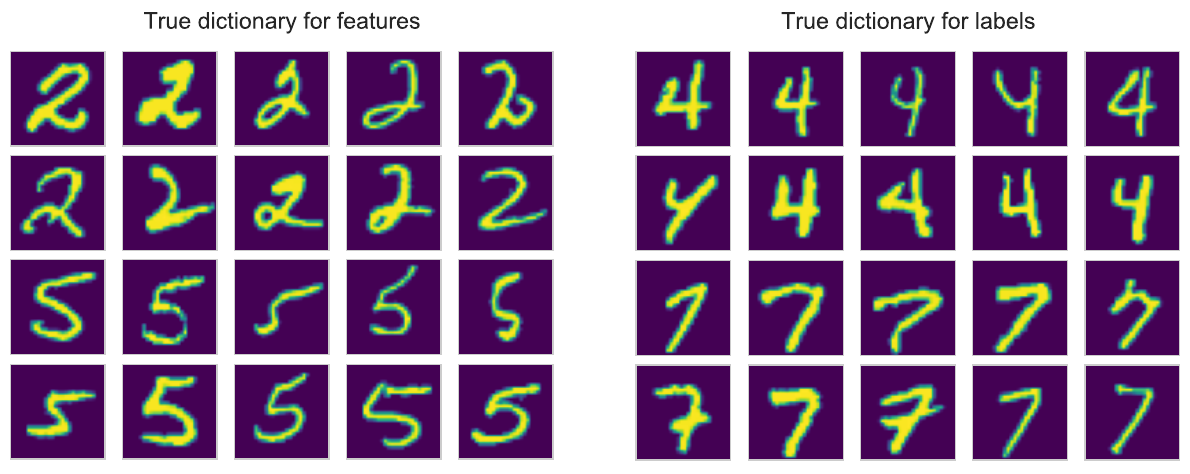} 
		\caption{ (Left) 20 basis images of digits `2' and `5' that comprises the true dictionary matrix of features, $W_{\textup{true}, X} \in \R^{784\times 20}$.  (Right) 20 basis images of digits `4' and `7' that comprises the true dictionary matrix of labels, $W_{\textup{true}, Y} \in \R^{784\times 20}$. }
		\label{fig:Figure2}
	\end{figure}
	
	More precisely, denote $p=28^{2}=784$, $n=500$, $\bar{r}=20$, $r=2$, and $\kappa=1$. First, we randomly select 10 images each from digits '2' and '5'. Vectorizing each image as a column in $p=784$ dimension, we obtain a true dictionary matrix for features $\W_{\textup{true}, X}\in \R^{p\times \bar{r}}$. Similarly, we randomly sample 10 images of each from digits '4' and '7' and obtain the true dictionary matrix of labels $\W_{\textup{true}, \Y}\in \R^{p\times \bar{r}}$. Next, we sample a code matrix $\H_{\text{true}} \in \R^{\bar{r} \times n}$ whose entries are i.i.d. with the uniform distribution $U([0,1])$.  Then the `pre-feature' matrix $\X_0 \in \R^{p\times n}$ of vectorized synthetic images is generated by $\W_{\textup{true}, X} \H_\text{true}$. The feature matrix $\X_{\textup{data}}\in \R^{p\times n}$ is then generated by adding an independent Gaussian noise $\eps_{j}\sim N(\mathbf{0}, \sigma^{2} I_{p})$ to the $j$th column of $\X_{0}$, for $j=1,\dots,n,$ with $\sigma=0.5$. We generate the binary label matrix $\Y=[y_{1},\dots,y_{n}]  \in \{0,1\}^{1\times n }$   (recall $\kappa=1$) as follows: Each entry $y_{i}$ is an indepedent Bernoulli variable with probability $p_{i} = \left( 1+\exp{(-\Beta_{\textup{true},\Y}^{T} \W_{\textup{true},\Y}^{T} \X_{\textup{data}}[:,i]  )}  \right)^{-1}$, where 
	$\Beta_{\textup{true}, \Y} = [1,-1]$.


	Now that we have generated a data set of labeled synthetic images $(\X,\Y)\in \R^{p\times n}\times \R^{1\times n}$, we can evaluate various methods of binary classificaiotn models.  We use the following six models: (1) the standard logistic regression with 784 variables (LR), (2) logistic regression with 2 basis factors obtained by NMF (NMF - LR), (3) Nonconvex filter-bsaed SDL (Algorithm \ref{algorithm:SDL}) (dubbed SDL-filt), (4) Nonconvex feature-based SDL (Algorithm \ref{algorithm:SDL}) (dubbed SDL-feat), (5) Convex filter-based SDL (Algorithm \ref{alg:SDL_filt_LPGD}) (dubbed SDL-conv-filt), and  (6) Convex feature-based SDL (Algorithm \ref{alg:SDL_feat_LPGD}). For all SDL methods, we learn only $r=2$ dictionary atoms whereas there are total 40 true dictionary atoms in $\W_{\textup{true},X}$ and $\W_{\textup{true},Y}$ combined. This is to force the algorithms to compute the dictionary of two atoms of `best comprimise'. For both nonconvex SDL methods, we used nonnegativity constraints on $\W$ and $\H$. For method (2), after learning a dictionary matrix $\hat{\W}\in \R^{p\times r}$ by NMF on $\X_{\textup{data}}$, we use logistic regression with 2-dimensional feature vectors, which are the columns of $\hat{\W}^{T} \X_{\textup{data}} \in \R^{2\times n}$.  Hence the same atoms in $\hat{\W}$ serves as basis for data reconstruction as well as filters for label prediction.

	Using 80\% training set and 20\% test set, we compared the model performance with respect to the accuracy and the F-score. Separate runs were made with 200 iterations and tuning parameters $\xi \in \{ 0.1, 1, 5, 10\}$ for all four SDL methods. The algorithms stopped after 200 iterations and we fitted the models for the same data 5 times to evaluate the performance. For convex algorithms, we used $L_{2}$-regularization coefficient $\nu=2$ and do not use any $L_{2}$-regularization for the nonconvex algorithms.

	\subsection{Performance evaluation}
	
	Next, we evaluate the performance of the six methods from the perspective of multi-objective optimization, following the analysis in \cite{ritchie2020supervised}. That is, we view SDL algorithms as solving an optimization problem with two possibly non-aligned objectives of minimizing data reconstruction error as well as maximizing label classification performance. Thus, each trained model can be associated with a point $(a,b)$ in a two-dimensional `Pareto' plane, where  $a=\lVert \X_{\textup{data}} - \hat{\W} \hat{\H} \rVert_{F}^{2}/  \lVert \X_{\textup{data}}\rVert_{F}^{2}$ denotes the normalized reconstruction error and $b$ denotes the classification performance measured by two metrics of accuracy and F-score. The Pareto plots are shown in Figure \ref{fig:MNIST_pareto}. An ideal method corresponds to a point in the upper left corner, having zero reconstruction loss and perfect classification.

	\begin{figure}[h!]
		\centering
		\includegraphics[width=1\linewidth]{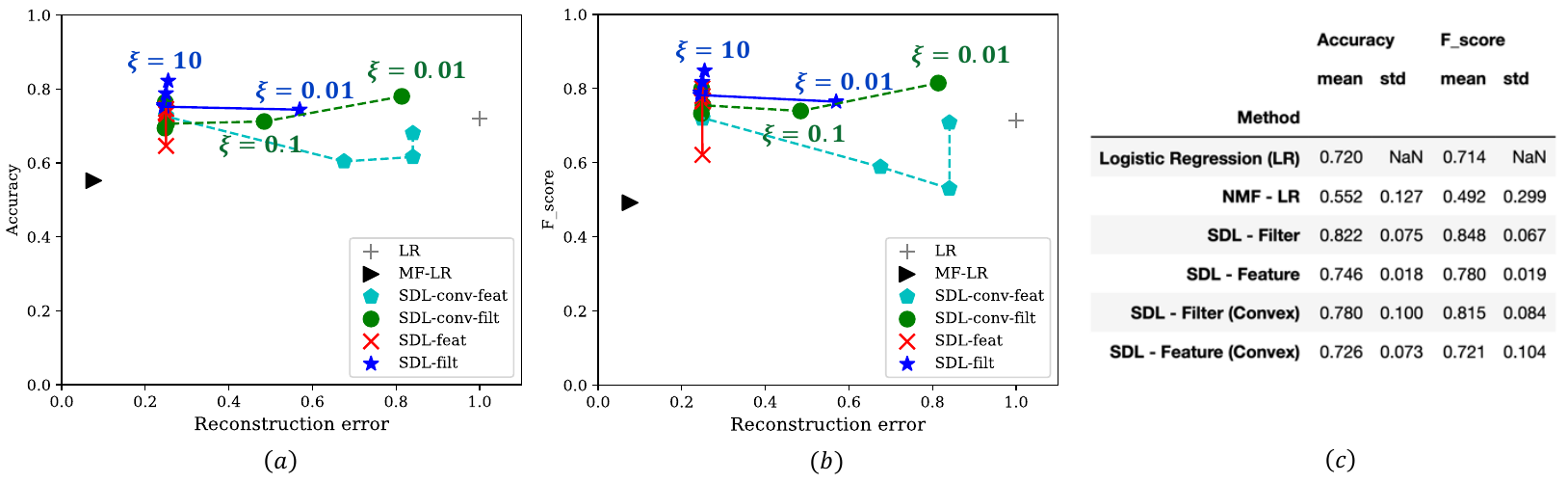} 
		\caption{ Pareto plot of relative reconstruction error vs. classification accuracy (a) and F-score (b) for various models on simulated MNIST dataset.  (c) Best average classification results for five values of tuning parameter $\xi\in \{0.01, 0.1, 1, 5, 10\}$.  See table \ref{table:MNIST_table} for more details.}
		\label{fig:MNIST_pareto}
	\end{figure}
	
	We observe that NMF-LR achieves the smallest reconstruction error among all methods but suffers for the classification task.  This is expected from the construction of the dataset, as the synthetic images are nonnegative linear combinations of images of digits '2' and '5', and the same dictionary atoms are also used as filters for label prediction.  On the other hand, logistic regression on pixels does not compress the data matrix so we assigned a relative reconstruction error of 1. In Figure \ref{fig:MNIST_pareto}, we observe that, except in a few instances, most SDL models lie between the two extremes of (1) LR and (2) NMF-LR in the sense that achieving significantly better classification accuracies (both in terms of accuracy and F-score) with small reconstruction error. For instance, SDL-filt with $\xi=10$ achieves a relative reconstruction error of about 0.22 and classification accuracy above 80$\%$, which is more than double the performance of NMF-LR and also about $11\%$ better than LR accuracy. It is interesting to note that SDL-conv-filt with $\xi=0.1$ achieves the best classification accuracy of about $91\%$ but its reconstruction error is quite large at around $0.7$.

	For more qualitative analysis, we plot various estimated dictionaries $\hat{\W}$  and compare them. Figure \ref{fig:MNIST_filt_dict} shows how the dictionary matrix $\hat{\W}$ estimated by SDL-filter changes depending on the level of tuning parameter $\xi\in \{0.01, 0.1, 1\}$. When $\xi=1$, the combined SDL loss in \eqref{eq:ASDL_1} puts some significant weight on the matrix factorization term, so the learned dictionary $\hat{\W}$ should be reconstructive of the synthesized images in $\X_{\textup{data}}$. Indeed, the learned atoms in Figure  \ref{fig:MNIST_filt_dict} left shows shapes of digits `5' and `2'. Further, the second atom resembling '2' is associated with a negative regression coefficient, indicating that being close to `2' may be partially aligned with being close to `7', which corresponds to being a negative example. On the other hand, decreasing the value of $\xi$ increases the amount of supervision. The learned atoms in Figure \ref{fig:MNIST_filt_dict} middle and right resembles less of the digits '2' and '5', but it seems that some abstract shape with large positive (for $\xi=1$) and large negative (for $\xi=0.01$) are learned and the resulting classification accuracies increase. The other atoms seem to be the shape of `8', which are seemingly learned from superpositions` 2' and `5'. One can regard the learned dictionary atoms as the `best effort' of SDL-filter in balancing the two partially aligned reconstructions and discrimination taken from the space of images of linear combinations of `2' and `5'.


	\begin{figure}[h!]
		\centering
		\includegraphics[width=1\linewidth]{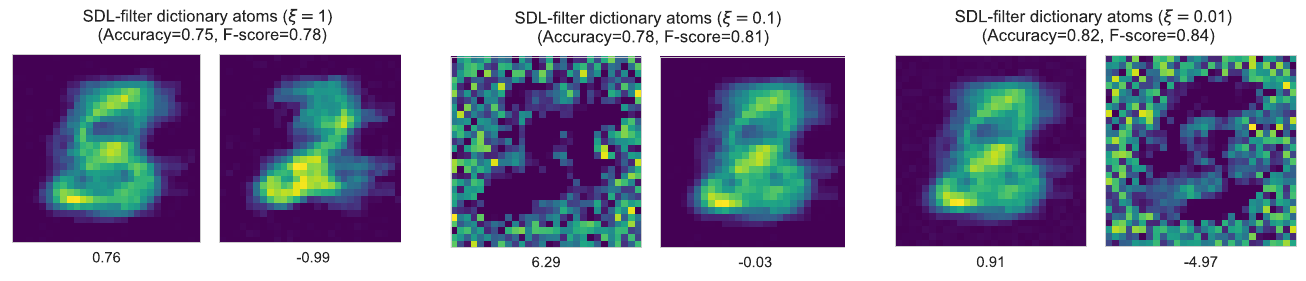}
		\caption{Estimated dictionary matrix $\hat{\W}$ from SDL-filter  depending on the level of tuning parameter $\xi\in \{ 1,0.1,0.01 \}$. The smaller the tuning parameter is the stronger the supervision effect is.}
		\label{fig:MNIST_filt_dict}
	\end{figure}

	\begin{figure}[h!]
		\centering
		\includegraphics[width=1\linewidth]{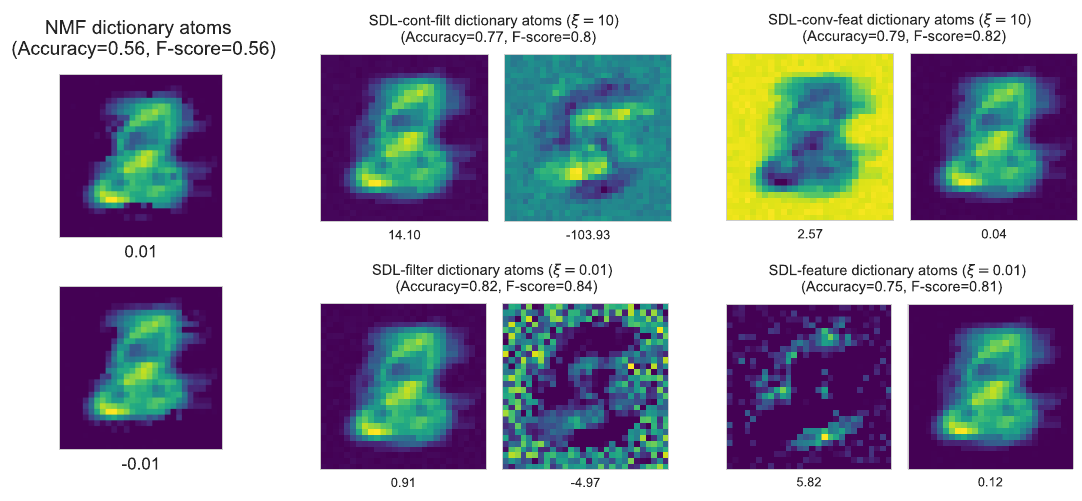} 
		\caption{Estimated basis matrix $\hat{W}_x$ from filter based method depending on the level of tuning parameter $\xi$. $\xi = 0$ (Left), $\xi = 0.5$ (Middle), $\xi = 1$ (Right)}
		\label{fig:MNIST_dict_comparison}
	\end{figure}

	In Figure \ref{fig:MNIST_dict_comparison}, we plot the estimated  dictionaries $\hat{\W}$ from all five methods except LR with chosen level of tuning parameter $\xi$. It is interesting that NMF learned almost identical atoms (with inferior classification performance) that resemble the typical shape of nonnegative linear combinations of digits `2' and `5', instead of learning separate atoms of shapes `2' and `5' individually.  For convex SDL algorithms, we find that typically a larger tuning parameter is required for fast convergence, which we also observed in Figure \ref{fig:benchmark_MNIST}.

	\section{Applications}
	\subsection{Supervised Topic Modeling on fake job postings dataset}

	
	According to Better Business Bureau, a non-profit organization that monitors and evaluates job postings, there were 3,434 fake job postings reported in 2019. These scam postings result in huge financial loss and the average loss per victim is \$3,000 according to the FBI reports \cite{FBI}. In this section, we use our methods to classify fake job postings on a dataset 'real-or-fake-job posting-prediction' in Kaggle \cite{fakejob_data}. In doing so, the new method can also simultaneously learn topics that are most effective in classifying fake job postings.	We compared the performance of classifiers based on multiple indexes such as accuracy, and F-score, and find what factors are highly associated with fake job postings. Identifying the relevant characteristics of scams and building prediction models will help prevent potential financial losses in advance.

	\subsubsection{Dataset description and preliminary analysis}
	\label{sec:fakejob_setup}
	
	There are 17,880 postings and 15 variables in the dataset including binary variables, categorical variables, and textual information of \textit{job description}. Among the 17,880 postings, 17,014 are true job postings (95.1\%) and 866 are fraudulent postings (4.84\%), which shows a high imbalance between the two classes. This imbalance is the main characteristic of the dataset. We coded fake job postings as positive examples and true job postings as negative examples. Due to the high imbalance, the accuracy of classification can be trivially high (e.g., by classifying everything to be negative), and hence achieving a high F-score is of importance. 
	
	In our experiments, we represented each job posting as a $p=2480$ dimensional word frequency vector computed from its \textit{job description} and augmented with $q=72$ auxiliary covariates of binary and categorical variables including indicators of the posting having a company logo or the posted job  in the United States or not. For computing the word frequency vectors, we represent the job description variable as a term/document frequency matrix with Term Frequency-Inverse Document Frequency (TF-IDF) normalization \cite{pedregosa2011scikit}. TF-IDF measures the relative importance of each term in a collection of documents. If a word is common to all documents, then it is less likely to have an important meaning. 
	The top 2480 most frequent words were used for the analysis. 

	\begin{figure}[h!]
		\centering
		\includegraphics[width=1\textwidth]{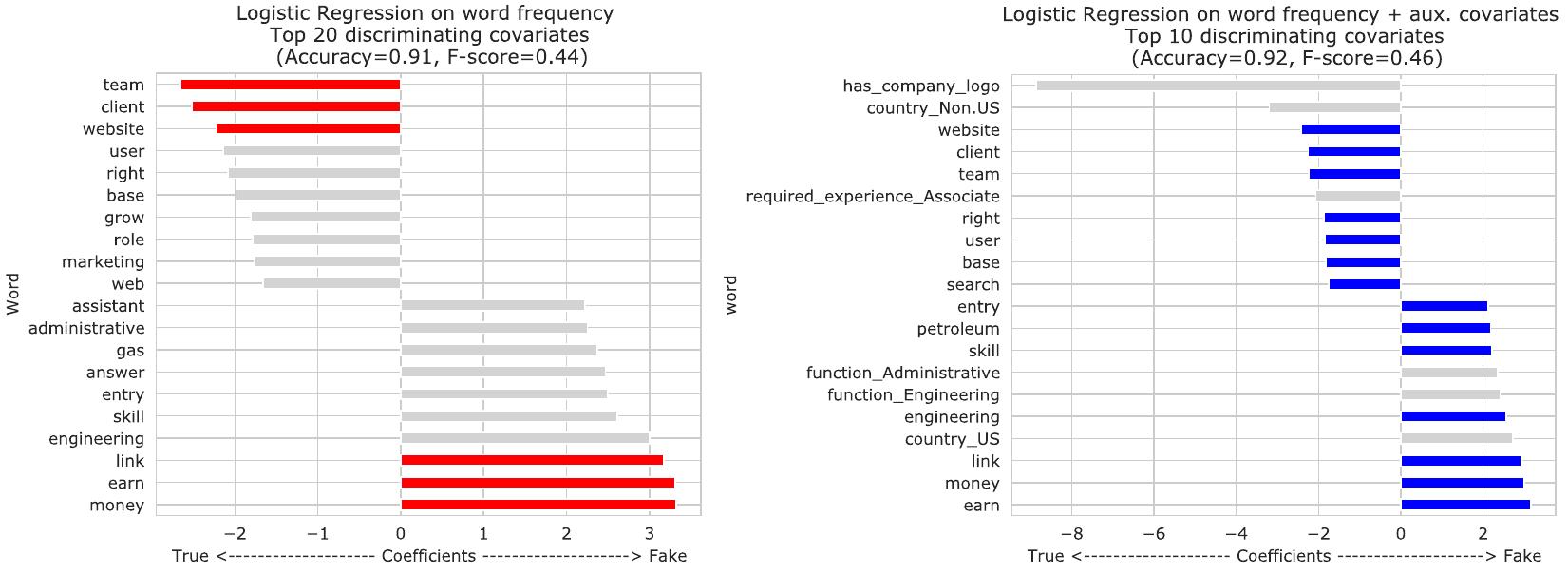}
		\caption{The top 20 variables with largest coefficients from logistic regression on $p=2480$ words in job description (left) and $p+q=2480+72$ words and auxiliary covariates combined (right). In the left panel, red bars indicate the words that appear as dominant keywords in the forthcoming topic modeling analysis. In the right panel, blue bars indicate words and grey bars indicate auxiliary covariates.}
		\label{fig:fakejob_LR}
	\end{figure}
	
	As a preliminary analysis, we first apply logistic regression either on the $p$-dimensional word frequency vectors or on the $(p+q)$-dimensional combined feature vectors (Figure \ref{fig:fakejob_LR} right). For the former experiment, Figure \ref{fig:fakejob_LR} left shows 10 words each with positive and negative regression coefficients with the largest absolute values. The results indicate that having a large frequency of words such as `earn', `money', and `link' is positively correlated with being a fake job, whereas postings with a high frequency of words such as `team', `client', and `website' as well as with company logo and jobs outside of the US are more likely to be true job postings.

	\subsubsection{Supervised Topic Modeling with Auxiliary Covariates}
	
	Topic modeling is a classical technique in text data analysis that seeks to find a small number of `topics', which are groups of words that share semantic context. The grounding assumption is that a given text may be built upon such topics as latent variables. Methods such as nonnegative matrix factorization (NMF) \cite{lee1999learning} and latent Dirichlet allocation (LDA) \cite{steyvers2007probabilistic, blei2003latent, jelodar2019latent} have been successfully used to detect or estimate such latent semantic factors. Also, `supervised' topic modeling techniques have been studied, where one seeks to group words not only by their semantic contexts but also using their `functional contexts' that are provided by additional class labels. See, for example, \cite{mcauliffe2007supervised} for LDA-based approaches and \cite{haddock2020semi} for NMF combined with linear regression model (see \eqref{eq:SDL_regression_H1}). 
	Here we mainly compare two methods, namely, (1) NMF with logistic regression and (2) SDL-filter with nonnegativity constraints on $\W$ and $\H$. However, we do compare the performance of all four SDL models in Figures \ref{fig:benchmark_fakejob} and \ref{fig:pareto_fakejob}. We note that for the purpose of topic modeling, it is crucial to use nonnegativity constraint on the dictionary matrix $\W$ in the SDL model \eqref{eq:ASDL_1} as word frequencies are nonnegative and we would like to decompose a given document's word frequency as additively rather than subtractively in order for better interpretability (see, e.g., \cite{lee1999learning}).

	\begin{figure}[h!]
		\centering
		\includegraphics[width=1\textwidth]{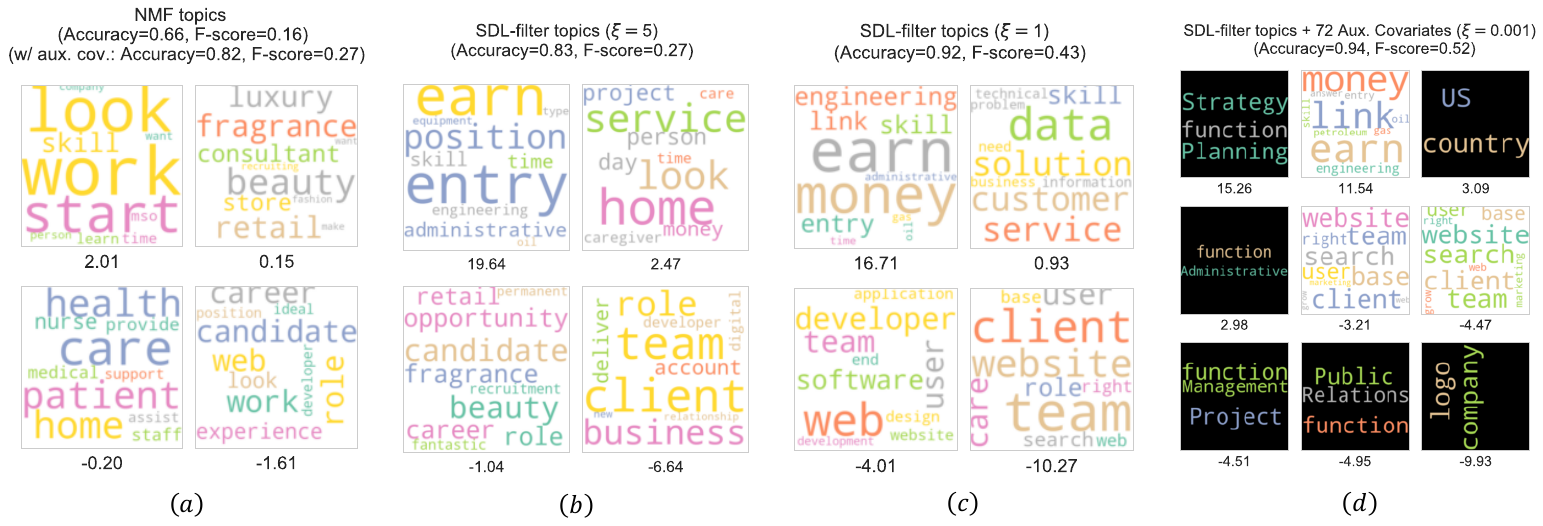}
		\caption{Comparison between (supervised) topics learned by NMF and SDL-filter for the fake job posting data. (a) Four out of 25 topics learned by NMF are shown together with the corresponding logistic regression coefficients. (b) Four out of 20 supervised topics learned by SDL-filter with tuning parameter $\xi=5$ are shown together with the corresponding logistic regression coefficients. (c) Similar as (b) but with tuning parameter $\xi=1$. (d) Nine out of 20 supervised topics (white background)+ 72 auxiliary covariates (dark background) learned from SDL-filt with 72 auxiliary covariates are shown with their corresponding logistic regression coefficients. Corresponding classification accuracy and F-scores are also shown in the subtitles with fake job postings being the positive examples. For topic wordclouds (white background), word sizes are proportional to their frequency. 
		} 
		\label{fig:fakejob_topics}
	\end{figure}
	
	
	First consider  Figure \ref{fig:fakejob_topics} (a), which shows topics (shown as wordclouds) learned by NMF and their associated regression coefficients. Namely, after learning a dictionary matrix $\W\in\R^{p\times 25}$ by NMF from the job description matrix of shape $p\times n$ with $n=17,780$, each of the $r$ columns of $\W$ becomes the topic frequency vector and top 10 words with higheset frequencies are shown as wordcloud. NMF was able to find topics that summarize specific job information. More specifically, the upper right and lower left topics correspond to beauty and healthcare-related jobs. However, as can be seen by the low F-score reported in Figure \ref{fig:fakejob_topics} (a), while the 25 topics learned by NMF give generic job descriptions, they may not be helpful to determine if a job posting is indeed fake. The main reason that we have these topics is that the dataset is highly imbalanced. Since most of the postings are true job postings (95\%), when we first conduct dimension reduction based on NMF, the topics that we learned are mainly determined by dominant true job postings, rather than fake job postings. 
	
	On the other hand, some selected supervised topics (out of $20$ total) learned by SDL-filter with $\xi=5$ and $1$ are shown in Figures \ref{fig:fakejob_topics} (b) and (c). In each case, the upper left and lower right topics are the ones with the largest positive and negative regression coefficients, respectively, and the upper right and lower left ones are manually selected for illustration purposes. For $\xi=1$ in Figure \ref{fig:fakejob_topics} (c), notice that the upper left topic with positive regression coefficient consists of words that appear frequently on fake job postings (e.g., `money', `earn', `link'), while the lower right topic uses the words from true job postings (e.g., `team', `client', `website'), both detected by logistic regression in Figure \ref{fig:fakejob_LR}. Topics with neutral regression coefficients are mainly used to reconstruct data matrix rather than for classification purposes. Note that the corresponding F-score of 0.43 is achieved using only  20 variables (topics) and is on par with the F-scores obtained by logistic regression using $p=2480$ or $p+q=2552$ variables in Figure \ref{fig:fakejob_LR}. 
	
	Increasing the tuning parameter $\xi$ from 1 to $5$ weakens the supervision effect. Accordingly, the two neutral topics in  Figure  \ref{fig:fakejob_topics} (b) becomes generic job descriptions as found by NMF in Figure  \ref{fig:fakejob_topics} (a), but the two extreme ones (upper left and lower right) maintain similar content and large absolute values of their regression coefficients. 
	
	We also conduct a similar analysis using SDL-filter with $\xi=0.001$ and $r=20$ topics along with 72 auxiliary covariates after converting categorical variables to one-hot-encoding. In Figure \ref{fig:fakejob_topics} (d), we show the covariates with the largest absolute regression coefficients, which is a mix of supervised topics (white background) and auxiliary variables (dark background). This setting achieves the best F-score of $0.52$ by using $20+72$ variables, which is still significantly less than using all 2552 variables while enjoying better interpretability. We see that SDL-filter automatically combines words that are positively or negatively associated with fake job postings in an ensemble with auxiliary covariates. In other words, SDL-filter seems to perform simultaneous supervised topic modeling on text data, while incorporating auxiliary covariates for improved performance. 
	
	
	\subsubsection{Evaluation of model performance}
	
	We provide a summary of classification accuracy and F-score of various settings in Table  \ref{table:fakejob_table} (see Tables \ref{table:fakejob_table} and \ref{table:fakejob_table1} for the full results). Note that due to the high imbalance in the dataset (only 5\% of fake job postings), getting a high classification accuracy is trivial (e.g., by classifying all as true job postings), so getting high $F$-score is more important. One can see that SDL-filter is overall the method of best classification performance, both in terms of accuracy and F-score, which improves when using auxiliary covariates. In contrast to comparable performance of convex SDL algorithms in the semi-synthetic MNIST data in Figure \ref{fig:benchmark_MNIST}, for fake job postings dataset they show mediocre  performance. It seems that for larger datasets in high dimension, one needs more extensive hyperparameter tuning for convex SDL methods than the nonconvex ones.

	\begin{table}[h!]
		\centering
		\includegraphics[width=1\linewidth]{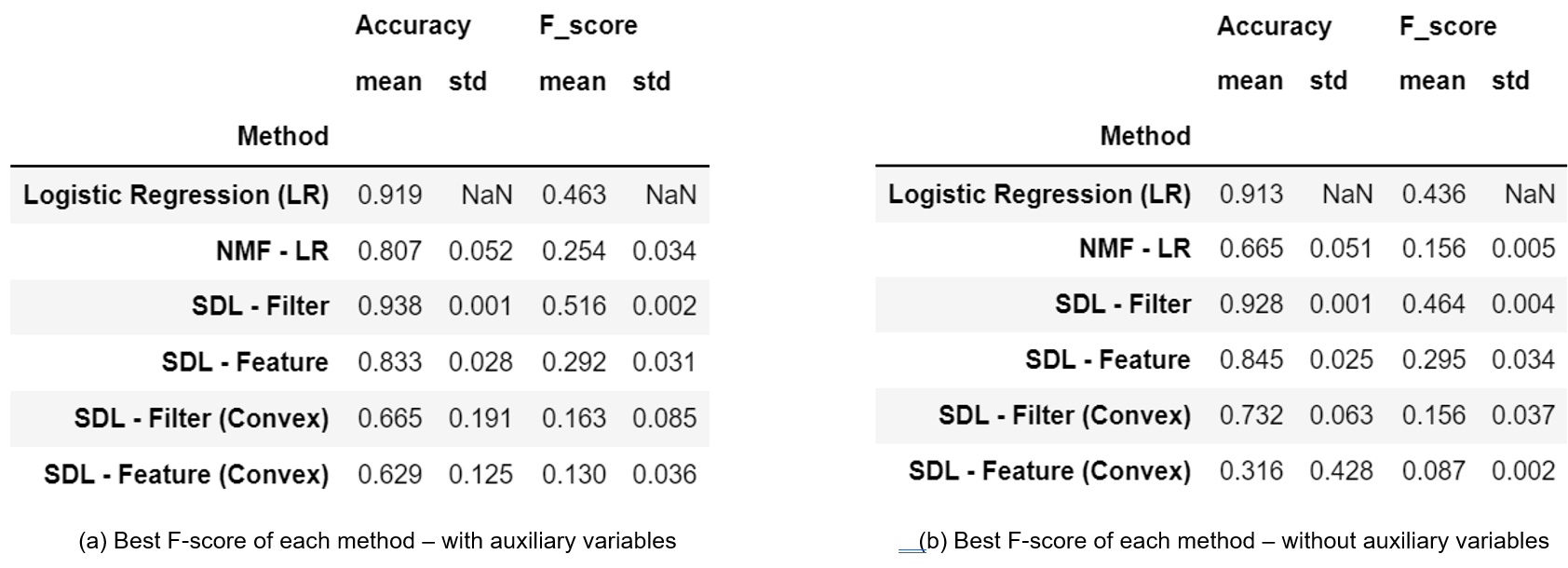} 
		\caption{ 
		Tables of best average F-score over five runs from each of the six methods for the fake job postings data in Section \ref{sec:fakejob_setup} with (left) and without (right) the auxiliary covariates for tuning parameter $\xi\in \{0.01, 0.1, 1, 5, 10\}$. See Tables \ref{table:fakejob_table1}  and \ref{table:fakejob_table2} in Appendix for more details.}
		\label{table:fakejob_table}
	\end{table}

	\begin{figure}[h!]
		\centering
		\includegraphics[width=0.8\textwidth]{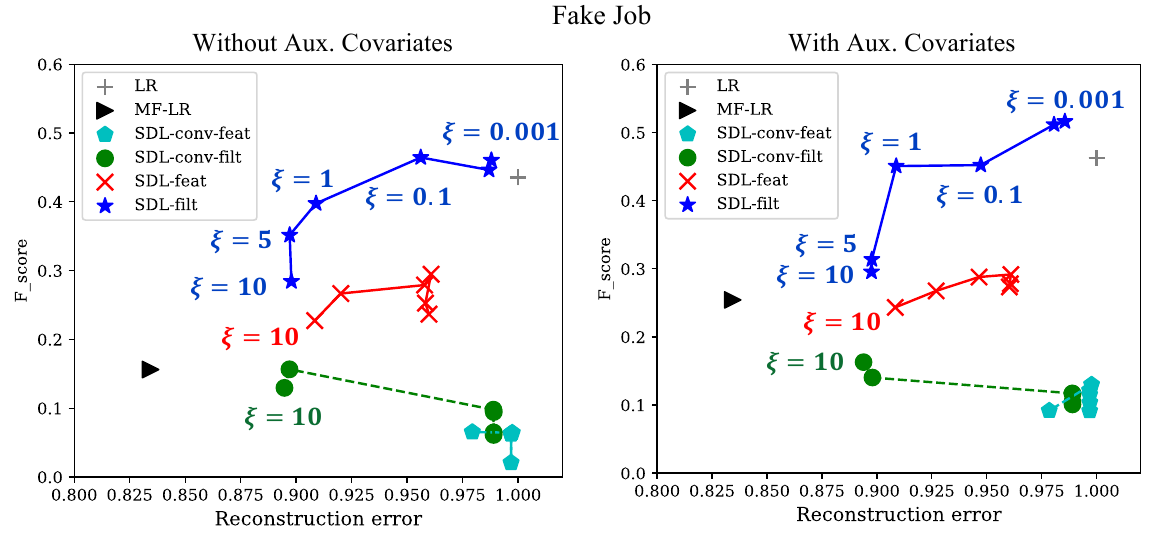}
		\caption{Pareto plot of relative reconstruction error vs. classification accuracy/F-score for various models on fake job postings dataset. 
		}
		\label{fig:pareto_fakejob}
	\end{figure}

	As in Figure \ref{fig:MNIST_pareto} for the semi-synthetic MNIST dataset, we also provide a Pareto plot in Figure \ref{fig:pareto_fakejob} to evaluate the performance of various SDL models on the fake job posting dataset against the benchmark models of logistic regression (LR) and NMF followed by logistic regression (NMF-LR). Recall that the Pareto plot shows how a model simultaneously performs two objectives of reducing the reconstruction error $\lVert \X_{\textup{data}} - \W\H \rVert$ as well as increasing the classification accuracy. As before, increasing the tuning parameter $\xi$ in various SDL models seems to interpolate between two extremes of LR and NMF-LR. We observe that SDL-filter performs best overall, in some cases achieving both goals with better classification performance than LR. The inferior performance of convex SDL models on the real dataset in contrast to their superior performance on the semi-synthetic dataset in Figure \ref{fig:MNIST_pareto} indicates that, in practice, the convex SDL models require more hyperparameter tuning. For instance, we did not try to fine-tune the stepsize, which we fixed at $\tau=0.01$ throughout the experiments.

	We also report that the topics learned from SDL-feature share similar characteristics with respect to the tuning parameter $\xi$ in figure \ref{fig:fakejob_topics}. We also mention that the topics learned from convex SDL algorithms, which cannot be used with a nonnegative constraint on the dictionary matrix $\W$, turns out to be uninformative and not very much interpretable. This is expected since the use of nonnegativity constraints on $\W$ and $\H$ (i.e., strong constraints on the SDL model, see \ref{assumption:A1}) is crucial for matrix-factorization-based topic modeling experiments (see \cite{lee1999learning}). We omit figures for the topics of SDL-feature and the convex SDL models.

	\subsection{Supervised dictionary learning on chest X-ray images for pneumonia detection}
	
	Pneumonia is an acute respiratory infection that affects the lungs. According to WHO reports, it accounts for 15\% of all deaths of children under 5 years old, killing 808,694 children in 2017. Moreover, currently, about 15\% of COVID-19 patients suffer from severe pneumonia \cite{WebMD}.   Chest X-ray is one inexpensive way to diagnose pneumonia, but rapid radiological
	interpretation is not always available. A successful statistical model to classify pneumonia from chest X-ray images will enable rapid pneumonia diagnosis with high accuracy, which will be able to reduce the burden on clinicians and help their decision-making process. In this section, we apply our SDL methods for chest X-ray images for pneumonia detection.

	\begin{figure}[h!]
		\centering
		\includegraphics[width=0.8\textwidth]{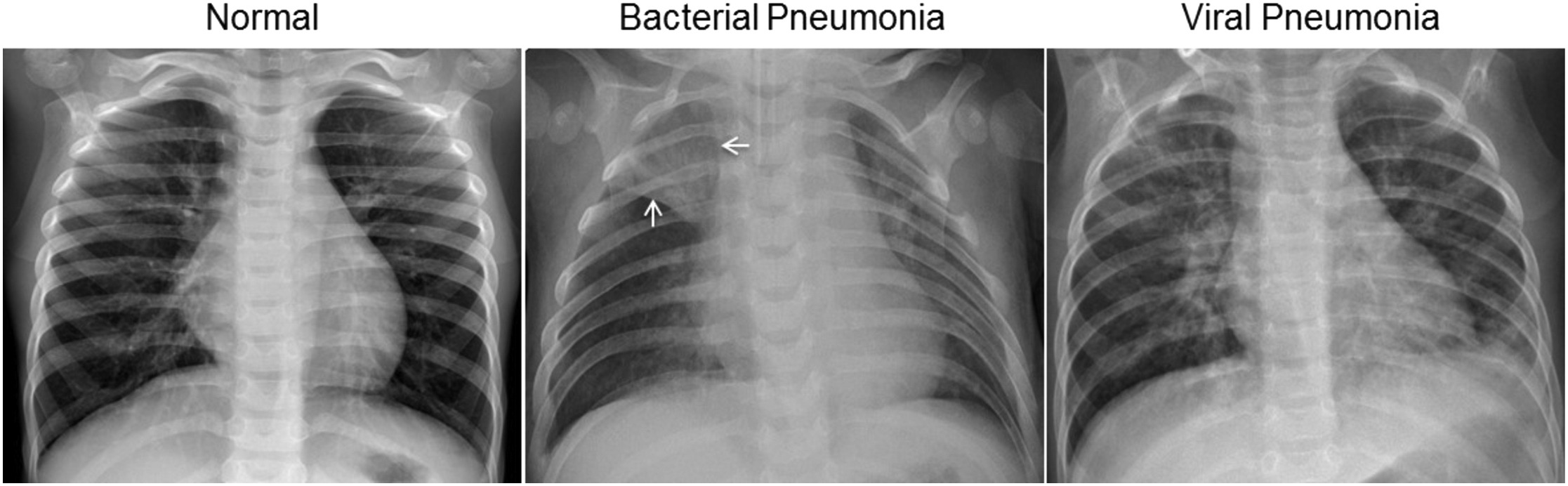}
		\caption{ The normal chest X-ray (left panel) depicts clear lungs without any areas of abnormal opacification in the image. Bacterial pneumonia (middle) typically exhibits a focal lobar consolidation, in this case in the right upper lobe (white arrows), whereas viral pneumonia (right) manifests with a more diffuse “interstitial” pattern in both lungs. The figure and the description are excerpted from \cite{kermany2018identifying}. 
		}
		\label{fig:pneumonia_cell}
	\end{figure}
	
	The pneumonia data set was first introduced in \cite{kermany2018identifying}. There is a total of 5,863 chest X-ray images from children, consisting of 4,273  pneumonia patients and 1,583 healthy subjects. The images were collected from pediatric patients one to five years old from Guangzhou Women and Children’s Medical Center, Guangzhou. For the analysis of chest X-ray images, all images were initially screened for quality control and two expert physicians diagnosed the images. In the reference \cite{kermany2018identifying}, an extremely accurate image classification system has been developed (with a classification accuracy of 92.8\%) using sophisticated deep neural network image classifiers. We intend to demonstrate that our SDL methods yield interesting and promising results for medical image classification tasks while being significantly simpler and easier to train than the deep neural network models.

	In order to apply our SDL methods, we resize each chest X-ray image into an $180 \times 180$ pixel image. Vectorizing each image, we obtain the data matrix $\X_{\textup{data}}\in \R^{32,400\times 5,863}$. We label pneumonia images with 1 and normal images with 0, obtaining the labeled matrix $\Y_{\textup{data}} \in \{0,1\}^{1\times 5,863}$. We used the deterministic test/train split provided by the original references \cite{kermany2018identifying}, where the train and the test sets consist of 5,216 and 624 images, respectively. Standard logistic regression with $p=180^{2}=32,400$ individual pixels as the explanatory variable yields a classification accuracy of 82\%. However, it is not entirely reasonable to assume individual pixels in the image to be correlated with pneumonia. Instead, we may associate certain latent shapes with pneumonia by using dictionary learning methods.

	\begin{figure}[h!]
		\centering
		\includegraphics[width=1\textwidth]{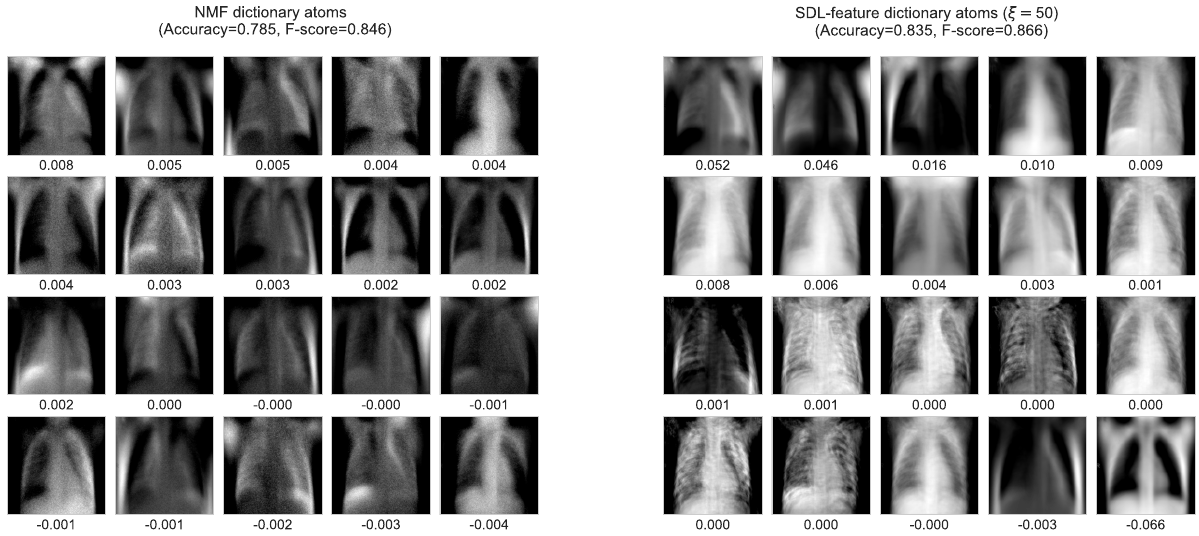}
		\caption{25 dictionary atoms learned from chest X-ray images by NMF and SDL-feature with $\xi=50$. Corresponding logistic regression coefficients, as well as classification performances, are also shown. For SDL-feature, we used $L_{1}$ regularization coefficient of $5$ for the code matrix $\H$ and no $L_{2}$-regularization. Positive regression coefficients indicate a positive correlation with having pneumonia. There is a clear contrast between the two extreme atoms (upper left and lower right) according to their correlation with pneumonia. 
		}
		\label{fig:pneumonia_dict}
	\end{figure}

	Figure \ref{fig:pneumonia_dict} shows 25 dictionary atoms of size $180\times 180$ learned by NMF (left) and SDL-feature (right) together with their corresponding logistic regression coefficient shown. For SDL-feature, we used tuning parameter $\xi=50$ and $L_{1}$-regularization coefficient of $5$ on the code matrix $\H$. Namely, we used Algorithm \ref{algorithm:SDL} for the feature-based SDL model in \eqref{eq:ASDL_1} together with an additional $L_{1}$ regularization term of $\lambda \lVert \H  \rVert_{1}$ added to the loss function in \eqref{eq:ASDL_1} with $\lambda=5$. Using the $L_{1}$-regularization on the code matrix is standard in dictionary learning literature \cite{mairal2007sparse,mairal2008supervised,mairal2010online} in order to learn sparse representation on an over-complete dictionary. For this particular application, we find such regularization helps improving the classification accuracy better than using $L_{2}$-regularization on any of the factors. We found the value of $\xi=50$ and $\lambda=5$ by using a grid search for $\xi=\{0.01, 0.1, 1, 10, 50, 100\}$ and $\lambda=\{0,1,5,10\}$. 
	
	
	Note that the code matrix $\H$ is constrained to be nonnegative during optimization.  Since the probability of the existence of pneumonia is calculated by the logit transformation of $\Beta (\mathbf{1},\H^{T})^{T}$, the positive value of $\Beta$ indicates that the corresponding atom is more related to pneumonia with the magnitude of the regression coefficient indicates the strength of such association, and the negative value of $\Beta$ suggests that the atom is related to normal. Comparing the NMF and SDL-feature dictionaries in Figure \ref{fig:pneumonia_dict}, we find that the regression coefficients for NMF atoms are quite neutral, but there are three atoms (two upper left and one lower right) with an order of magnitude larger absolute regression coefficient learned by SDL-feature. A closer investigation shows that the atom with regression coefficients $-0.066$ has almost no signal (dark) around the lung, whereas the two atoms with regression coefficients 0.052 and 0.046 show the opposite contrast, having most signals (bright) around the lung and weak signals elsewhere. Although one should pay extra caution when interpreting machine learning results as a clinical statement, this observation seems to be well-aligned with some basic characteristics of normal or pneumonia chest X-ray images as shown in Figure \ref{fig:pneumonia_cell} (see also the caption).  
	
	Lastly, we report some additional details of these experiments. First, the training of SDL-feature on the chest X-ray dataset and making predictions is extremely efficient and the entire process of training and testing takes under 2 minutes on an average laptop computer (100 iterations on an Apple M1 chip). Second, we find that SDL-filter achieves higher classification performance with accuracy $84.3\%$ and F-score $87.9\%$ with tuning parameter $\xi=0.01$, $r=20$ atoms, and the same $L_{1}$-regularization coefficient on $\H$ of $5$. However, the learned atoms are too noisy and not quite interpretable as the ones learned by SDL-feature in Figure \ref{fig:pneumonia_dict} right. Training and testing take about ten minutes on the same machine. Third, convex SDL models, in this case, take a long time (about an hour) with the same number of iterations and we omit the results.

	\section{Concluding Remarks}
	
	
	In this paper, we provided a comprehensive treatment of a large class of supervised dictionary learning methods in terms of model construction, optimization algorithms and their convergence properties, and statistical estimation guarantees for the corresponding generative models. SDL models find best balance between two objectives of data modeling by latent factors and label classification while achieving simultaneous dimension reduction and classification, which makes them suitable for coping with high-dimensional data. We demonstrated that our methods show comparable performance with classical models for classifying fake job postings as well as chest X-ray images for pneumonia while learning basis topics or images that are directly associated with fake jobs or pneumonia, respectively. In addition, the new methods achieve such comparable performance with much reduced number of variables and with much more homogeneous and interpretable models. 
	
	Our method can potentially be used for a number of high-dimensional classification problems, especially for areas where interpretability is required such as natural language processing and biomedical image processing. While a number of sophisticated deep-learning-based approaches are gaining popularity due to their extreme success in diverse problems including image classification and voice recognition, an inherent downside is the loss of interpretability due to a severe over-parameterization and sophisticated design of such algorithms. In this work, we showed that combining two classical methods of nonnegative matrix factorization and logistic regression could achieve comparable performance while maintaining the transparency of the method and the interpretability of the results. 
	
	
	One of the main techniques we developed in this work is a `double lifting procedure', which transforms the SDL problem into a CAFE problem \eqref{eq:CAFE} and then to a CALE problem \eqref{eq:CALE}; then we can use globally guaranteed low-rank projected gradient descent (Algorithm \ref{algorithm:LPGD}) to efficiently find the global optimum of the resulting CALE problem, and then we can pull that solution back to the original space. While this approach has proven to be quite powerful in analyzing SDL problems in the present work, it is interesting to note that our double-lifting technique does not immediately apply to 
	the supervised PCA model proposed by Ritchie et al.  \cite{ritchie2020supervised} for finding low-dimensional subspace that is also effective for a regression task: 
	\begin{align}\label{eq:SPCA_regression_W1}
		\min_{\W\in \R^{p\times r},\, \W^{T}\W=\I_{r},\,  \Beta\in \R^{1\times r}  } \left[  f\left( \W \begin{bmatrix} \Beta ,\,  \W^{T}  \end{bmatrix} \right):= \lVert \Y_{\textup{label}} -  \Beta^{T}  \W^{T}\X_{\textup{data}} \rVert_{F}^{2} + \xi  \lVert \X_{\textup{data}} - \W \W^{T} \X_{\textup{data}}   \rVert_{F}^{2}\right].
	\end{align}
	Even though we can realize the objective function in the right hand side of \eqref{eq:SPCA_regression_W1} as a function depending on the product $\W[\Beta,\W^{T}]$, the two matrix factors $\W$ and $[\Beta,\W^{T}]$ are \textit{not} decoupled as before. It is for a future investigation to devise a lifting technique for SPCA and related problems, and obtain a strong global convergence guarantee.

	\section*{Acknowledgements}

    HL is partially supported by NSF  DMS-2206296 and DMS-2010035.

	\small{
		\bibliographystyle{amsalpha}
		\bibliography{mybib}
	}

	\addresseshere

	\newpage

	\appendix

	\section{Proof of main results}

	\subsection{Proof of Theorem \ref{thm:CALE_LPGD}}
	We first establish Theorem \ref{thm:CALE_LPGD}, which shows exponential convergence  of the low-rank projected gradient descent (Algorithm \ref{algorithm:LPGD}) for the CALE problem \ref{eq:CALE}. The proof is similar to the standard argument that shows exponential convergence  projected gradient descent with fixed step size for constrained strongly convex problems (see, e.g., \cite[Thm. 10.29]{beck2017first}). However, when we have a strongly convex minimization problem with low-rank-constrained matrix parameter, then the constraint set of low-rank matrices is not convex, so one cannot use non-expansiveness of convex projection operator. Indeed, the rank-$r$ projection $\Pi_{r}$ by truncated SVD is not guaranteed to be non-expansive. In order to circumvent this issue, we use the idea of approximating rank-$r$ projection by a suitable linear projection on a carefully chosen linear subspace, an approach used in \cite{wang2017unified}. Then one can show that rank-$r$ projection is at most $2$-Lipschitz in some sense, so if the contraction constant in standard analysis of projected gradient descent for strongly convex objectives is small enough ($<1/2$), then overall one still retains exponential convergence.

	\begin{lemma}(Linear approximation of rank-$r$ projection)\label{lem:rank_r_lin_appx}
		Fix $\Y\in \R^{d_{1}\times d_{2}}$, $R\ge r \in \mathbb{N}$, and denote $\X=\Pi_{r}(\Y)$ and $\hat{\X} = \Pi_{\mathcal{A}}(
		\Y)$, where $\mathcal{A}\subseteq \R^{d_{1}\times d_{2}}$ is a linear subspace. Let $\X=\U\bSigma \V^{T}$ denote the SVD of $\X$. Suppose there exists $\overline{\U}\in \R^{d_{1}\times R}$ and $\overline{\V} \in \R^{d_{2}\times R}$ such that 
		\begin{align}
			\mathcal{A} = \left\{ \A \in \R^{d_{1}\times d_{2}}\,\big|\, \textup{col}(\A^{T}) \subseteq \textup{col}(\overline{\V}),\, \textup{col}(\A) \subseteq \textup{col}(\overline{\U})  \right\}, \quad \textup{col}(\U)\subseteq \textup{col}(\overline{\U}), \quad \textup{col}(\V)\subseteq \textup{col}(\overline{\V}).
		\end{align}
		Then $\X=\Pi_{r}(\hat{\X})$. 
	\end{lemma}
	
	\begin{proof}
		Write $\Y-\X=\dot{\U} \dot{\bSigma} \dot{\V}^{T}$  for its SVD. Let $d:=\rank(\Y)$ and let $\sigma_{1}\ge \dots \ge \sigma_{d}>0$ denote the nonzero singular values of $\Y$. Since $\X=\Pi_{r}(\Y) = \U\bSigma\V^{T}$ and $\Y=\U\bSigma\V^{T} + \dot{\U} \dot{\bSigma} \dot{\V}^{T}$, we must have that $\bSigma$ consists of the top $r$ singular values of $\Y$ and the rest of $d-r$ singular values are contained in $\dot{\bSigma}$. Furthermore, $\textup{col}(\U) \perp \textup{col}(\dot{\U})$.
		
		Now, since $\X\in \mathcal{A}$ and $\Pi_{\mathcal{A}}$ is linear, we get 
		\begin{align}\label{eq:3r_r_approx1_lem}
			\hat{\X} = \Pi_{\mathcal{A}}( \X + (\Y -\X) ) = \U \bSigma \V^{T} + \Pi_{\mathcal{A}}(\dot{\U} \dot{\bSigma} \dot{\V}^{T}). 
		\end{align}
		Let $\bZ:= \Pi_{\mathcal{A}}(\dot{\U} \dot{\bSigma} \dot{\V}^{T})$ and write its SVD as $\bZ = \widetilde{\U}\widetilde{\bSigma}\widetilde{\V}^{T}$. Then note that $(\U^{T}\overline{\U} \, \overline{\U}^{T})^{T}=\overline{\U}\,\overline{\U}^{T}\U=\U$ since $\overline{\U}\,\overline{\U}^{T}:\R^{d_{1}}\rightarrow \R^{d_{1}}$ is the orthogonal projection onto $\textup{col}(\overline{\U})\supseteq \textup{col}(\U)$. Hence $\U^{T}\overline{\U}\,\overline{\U}^{T} = \U^{T}$, so we get 
		\begin{align}
			\U^{T} \bZ =  	\left( \U^{T} \overline{\U}\, \overline{\U}^{T}\right) \dot{\U} \dot{\bSigma} \dot{\V}^{T} 	\V^{T} \overline{\V} = \left( \U^{T} \dot{\U} \right) \dot{\bSigma} \dot{\V}^{T} 	\V^{T} \overline{\V} = O.
		\end{align}
		It follows that $\U^{T}\widetilde{\U}=O$, since $	\U^{T}\widetilde{\U} = \U^{T} \bZ \widetilde{\V} (\widetilde{\bSigma})^{-1} = O$. Therefore, rewriting \eqref{eq:3r_r_approx1_lem} gives the SVD of $\hat{\X}$ as 
		\begin{align}
			\hat{\X} = \begin{bmatrix} 
				\U & \widetilde{\U}
			\end{bmatrix}
			\begin{bmatrix} 
				\bSigma & O\\
				O & \widetilde{\bSigma} 
			\end{bmatrix}
			\begin{bmatrix} 
				\V \\ \widetilde{\V}
			\end{bmatrix}
			.
		\end{align}
		Furthermore, $\lVert  \Pi_{\mathcal{A}}(\dot{\U} \dot{\bSigma} \dot{\V}^{T}) \rVert_{2} \le \lVert\dot{\bSigma} \rVert_{2} =\sigma_{r+1}^{t}$, so $\Sigma$ consists of the top $r$ singular values of $\hat{\X}$. It follows that $\X=\U \bSigma \V^{T}$ is the best rank-$r$ approximation of $\hat{\X}$, as desired. 
	\end{proof}
	
	\begin{proof}[\textbf{Proof of Theorem} \ref{thm:CALE_LPGD}]
		Denote $\bZ^{\star}=[\X^{\star}, \bGamma^{\star}]\in \Param\subseteq \R^{d_{1}\times d_{2}}\times \R^{d_{3}\times d_{4}}$.  Let $\mathcal{A}$ denote a linear subspace of $\R^{d_{1}\times d_{2}}$. Denote 
		\begin{align}
			\hat{\bZ}_{t} =\Pi_{\mathcal{A}\times \R^{d_{3}\times d_{4}}}\left(  \Pi_{\Param} \left( \bZ_{t-1} -  \tau \nabla f(\bZ_{t}) \right) \right).
		\end{align}
		We will choose $\mathcal{A}$ in such a way that 
		\begin{align}\label{eq:mathcal_A_subspace_cond}
			\X_{t} = \Pi_{r}(\hat{\X}_{t}) \in \argmin_{\X, \rank(\X)\le r} \lVert \hat{\X}_{t} - \X  \rVert_{F}, \quad  \bZ^{\star}\in \mathcal{A} \times \R^{d_{3}\times d_{4}}.
		\end{align}
		For instance, $\mathcal{A}=\R^{d_{1}\times d_{2}}$ satisfies the above conditions, although this choice does not give optimal control on the variance term (the second term in the right hand side of \eqref{eq:CALE_LPGD_thm}). We will first derive a general bound with such $\mathcal{A}$, and at the end of the proof, we will give a specific construction of such $\mathcal{A}$ to obtain the bound in the assertion. 
		
		Denote $\Delta \bZ^{\star}:=\bZ^{\star} - \Pi_{\Param}\left( \bZ^{\star} - \tau \nabla f(\bZ^{\star}))\right)$. Using $\bZ^{\star}\in \mathcal{A}\times \R^{d_{3}\times d_{4}}$ and linearity of the linear projection $\Pi_{\mathcal{A}\times \R^{d_{3}\times d_{4}}}$, write 
		\begin{align}
			\bZ^{\star} &= \Pi_{\mathcal{A}\times \R^{d_{3}\times d_{4}}}(\bZ^{\star}) \\ 
			&=\Pi_{\mathcal{A}\times \R^{d_{3}\times d_{4}}} \left( \Pi_{\Param}(\bZ^{\star} - \tau \nabla f(\bZ^{\star})) \right)  + \Pi_{\mathcal{A}\times \R^{d_{3}\times d_{4}}} \left( \bZ^{\star}-  \Pi_{\Param}(\bZ^{\star} - \tau \nabla f(\bZ^{\star}))  \right) \\
			&=\Pi_{\mathcal{A}\times \R^{d_{3}\times d_{4}}} \left( \Pi_{\Param}(\bZ^{\star} - \tau \nabla f(\bZ^{\star})) \right) +  \Pi_{\mathcal{A}\times \R^{d_{3}\times d_{4}}} \left(  \Delta \bZ^{\star} \right).
		\end{align}
		Namely, the first term above is a one-step update of a projected gradient descent at $\bZ^{\star}$ over $\Param$ with stepsize $\tau$, and the second term above is the error term. If $\bZ^{\star}$ is a stationary point of $f$ over $\Param$, then $-\nabla f(\bZ^{\star})$ lies in the normal cone of $\Param$ at $\bZ^{\star}$, so $\bZ^{\star}$ is invariant under the projected gradient descent and the error term above (the second term in the last expression) is zero. If $\bZ^{\star}$ is only approximately stationary, then the error above is nonzero. 
		
		Now,  recall that $\hat{Z}_{t}$ is obtained by using the orthogonal projection $\Pi_{\mathcal{A}}$ onto the convex subset $\mathcal{A}$ instead of the rank-$r$ projection $\Pi_{r}$  to obtain the matrix coordinate of $\hat{\bZ}_{t}$. Notice that $\Pi_{\Param}$ and $\Pi_{\mathcal{A}\times \R^{d_{3}\times d_{4}}}$ are non-expansive being projection onto a convex set, while the rank-$r$ projection $\Pi_{r}$ is not in general. Also using the linearity of the  subspace projection $\Pi_{\mathcal{A}\times \R^{d_{3}\times d_{4}}}$, we get 
		\begin{align}
			&\lVert \hat{\bZ}_{t} - \bZ^{\star} \rVert_{F}  \\
			&\qquad = \left\lVert   \Pi_{\mathcal{A}\times \R^{d_{3}\times d_{4}}}\left(  \Pi_{\Param} \left( \bZ_{t-1} -  \tau \nabla f(\bZ_{t-1}) \right) \right) - \Pi_{\mathcal{A}\times \R^{d_{3}\times d_{4}}}\left(  \Pi_{\Param} \left( \bZ^{\star} -  \tau \nabla f(\bZ^{\star}) \right) \right) +\Pi_{\mathcal{A}\times \R^{d_{3}\times d_{4}}}\left( \Delta \bZ^{\star} \right)  \right\rVert_{F}  \\
			&\qquad \le  \left\lVert  \bZ_{t-1} -  \tau \nabla f(\bZ_{t-1})  - \bZ^{\star} +  \tau \nabla f(\bZ^{\star})  \right\rVert_{F}  + \lVert \Pi_{\mathcal{A}\times \R^{d_{3}\times d_{4}}}\left( \Delta \bZ^{\star}  \right)  \rVert_{F} \\
			&\qquad \le \max( |1-\tau L| ,\, |1-\tau \mu|  ) \, \lVert  \bZ_{t-1} - \bZ^{\star}\rVert_{F} + \lVert \Pi_{\mathcal{A}\times \R^{d_{3}\times d_{4}}}\left( \Delta \bZ^{\star} \right)  \rVert_{F}.
		\end{align}
		
		The last inequality follows from the fact that $\bZ_{t}$ and $\bZ^{\star}$ have rank $\le r$ and the restricted strong convexity and smoothness properties (Definition \ref{def:RSC}). Namely, fix $\X,\Y\in \R^{d_{1}\times d_{2}}\times \R^{d_{3}\times d_{4}}$ whose first matrix components have rank $\le r$. Assuming $\nabla f$ is continuous, 
		\begin{align}
			\X -  \tau \nabla f(\X)  - \Y +  \tau \nabla f(\Y) &= (\X-\Y) - \tau ( \nabla f(\X) -  \nabla f(\Y)) \\
			&= \int_{0}^{1} \left(  \I-\tau \nabla^{2}(\X + s(\Y-\X)) \right)(\X-\Y)\,ds.
		\end{align}
		Using the inequality $\lVert \A \B \rVert_{F} \le \lVert \A \rVert_{2} \lVert \B \rVert_{F}$, this gives 
		\begin{align}
			\lVert \X -  \tau \nabla f(\X)  - \Y +  \tau \nabla f(\Y) \rVert_{F} & \le \sup_{\bZ=[\bZ_{1},\bZ_{2}]: \, \rank(\bZ_{1})\le r}   \lVert  \I-\tau \nabla^{2}f( \bZ)\rVert_{2}  \, \lVert \X - \Y \rVert_{F} \\
			&\le \rho \,\lVert \X - \Y \rVert_{F},
		\end{align}
		where $\eta:=\max( |1-\tau L| ,\, |1-\tau \mu| )$. Indeed, for the second inequality above, note that the eigenvalues of $\nabla^{2} f(\bZ)$ are contained in $[\mu,L]$, so the eigenvalues of $\I-\tau \nabla^{2}f(\bZ) $ are between $\min(1-\tau L,\, 1-\tau \mu)$ and $\max(1-\tau L,\, 1-\tau \mu)$. Combining with the previous inequality, it follows that 
		\begin{align}\label{eq:pf_LPGD_pf1}
			\lVert \hat{\bZ}_{t} - \bZ^{\star}  \rVert_{F}  \le  \eta \, \lVert  \bZ_{t-1} - \bZ^{\star}\rVert_{F} + \lVert \Pi_{\mathcal{A}\times \R^{d_{3}\times d_{4}}}\left(\Delta \bZ^{\star}\right)  \rVert_{F}.
		\end{align}
		Next, recall the notation $\bZ_{t}=[\X_{t}, \bGamma_{t}]$ and $\hat{\bZ}_{t}=[\hat{\X}_{t}, \bGamma_{t}]$.  By construction, we have $\X_{t} =\Pi_{r}( \hat{\X}_{t})$, so $\X_{t}$ is the best rank-$r$ approximation of $\hat{\X}_{t}$ in the sense that  $\X_{t}=\argmin_{\X, \rank(\X)\le r} \lVert \hat{\X}_{t} - \X  \rVert_{F}$. Then observe that 
		\begin{align}
			\lVert \bZ_{t} - \bZ^{\star}  \rVert_{F}  &\le 	\lVert \bZ_{t} - \hat{\bZ}_{t}  \rVert_{F} + 	\lVert \hat{\bZ}_{t} - \bZ^{\star}  \rVert_{F}  \\
			&= 	\lVert \X_{t} - \hat{\X}_{t}  \rVert_{F} + 	\lVert \hat{\bZ}_{t} - \bZ^{\star}  \rVert_{F}  \\
			&\le 	\lVert\X^{\star} -  \hat{\X}_{t}  \rVert_{F}  + 	\lVert \hat{\bZ}_{t} - \bZ^{\star}  \rVert_{F}   \le 2 	\lVert \hat{\bZ}_{t} - \bZ^{\star}  \rVert_{F}, 
		\end{align}
		so by combining with \eqref{eq:pf_LPGD_pf1}, we get 
		\begin{align}
			\lVert \bZ_{t} - \bZ^{\star}  \rVert_{F}  \le  2\eta \, \lVert  \bZ_{t-1} - \bZ^{\star}\rVert_{F}+ \lVert \Pi_{\mathcal{A}\times \R^{d_{3}\times d_{4}}}\left(\Delta \bZ^{\star} \right)  \rVert_{F}.
		\end{align}
		Note that $0\le \eta<1/2$ if and only if $\tau\in (\frac{1}{2\mu}, \frac{3}{2L})$, and this interval is non-empty if and only if $L/\mu<3$. Hence for such choice of $\tau$, $0<2\eta<1$,  so  
		by a recursive application of the above inequality, we obtain 
		\begin{align}\label{eq:linear_conv_ineq_pf1}
			\lVert \bZ_{t} - \bZ^{\star}  \rVert_{F}  \le  (2\eta)^{t}\, \lVert  \bZ_{0} - \bZ^{\star}\rVert_{F} +  \frac{1}{1-2\eta} \lVert \Pi_{\mathcal{A}\times \R^{d_{3}\times d_{4}}}\left(\Delta \bZ^{\star} \right)  \rVert_{F}.
		\end{align}

		Finally, we bound the variance term in the last expression by choosing a suitable linear subspace $\mathcal{A}\subseteq \R^{d_{1}\times d_{2}}$ satisfying \eqref{eq:mathcal_A_subspace_cond}. Note that $\Delta \bZ^{\star}=\tau [\Delta \X^{\star}, \Delta \bGamma^{\star}]$, where the latter is defined in the statement of Theorem \ref{thm:CALE_LPGD}. Recall that $\bZ^{\star}=[\X^{\star}, \bGamma^{\star}]$. Let $\X^{\star} = \U^{\star} \bSigma^{\star} (\V^{\star})^{T}$ denote the SVD of $\X^{\star}$.  For each iteration $t$, denote $\bZ_{t} = [\X_{t}, \bGamma_{t}]$ and let $\X_{t} = \U_{t} \bSigma_{t} \V_{t}^{T}$ denote the SVD of $\X_{t}$. Since $\X_{t}$ and $\X^{\star}$ have rank at most $r$, all of both $\U^{\star}$, $\U_{t}$, $\V^{\star}$, and $\V_{t}$ have at most $r$ columns. Define a matrix $\U_{3r}$ so that its columns form a basis for the subspace spanned by the columns of $[\U^{\star}, \U_{t-1}, \U_{t}]$. Then $\U_{3r}$ has at most $3r$ columns. Similarly, let  $\U_{3r}$ be a matrix so that its columns form a basis for the subspace spanned by the columns of $[\V^{\star}, \V_{t-1}, \V_{t}]$. Then $\V_{3r}$ has at most $3r$ columns. Now, define the subspace 
		\begin{align}\label{eq:mathcal_A_construction}
			\mathcal{A} := \left\{ \Delta\in \R^{d_{1}\times d_{2}}\,|\, \textup{span}(\Delta^{T}) \subseteq \textup{span}(\V_{3r}),\, \textup{span}(\Delta) \subseteq \textup{span}(\U_{3r})  \right\}.
		\end{align}
		Note that $\mathcal{A}$ is a convex subset of $\R^{d_{1}\times d_{2}}$.  Also note that, by definition, $\X^{\star}, \X_{t}, \X_{t-1}\in \mathcal{A}$. Let $\Pi_{\mathcal{A}}$ denote the projection operator onto $\mathcal{A}$. More precisely, for each $\X\in \R^{d_{1}\times d_{2}}$, we have 
		\begin{align}
			\Pi_{\mathcal{A}}(\X) = \U_{3r}\U_{3r}^{T} \X \V_{3r}\V_{3r}^{T}. 
		\end{align}
		Then by Lemma \ref{lem:rank_r_lin_appx}, we have $\X_{t} =\Pi_{r}( \hat{\X}_{t})$. Hence $\mathcal{A}$  in \eqref{eq:mathcal_A_construction} satisfies \eqref{eq:mathcal_A_subspace_cond}. Therefore, \eqref{eq:linear_conv_ineq_pf1} holds for the $\mathcal{A}$ chosen as in \eqref{eq:mathcal_A_construction}. 
		
		Now, note that  $\Pi_{\mathcal{A}\times \R^{d_{3}\times d_{4}}}(\Delta \X^{\star}, \Delta \bGamma^{\star}) = [\Pi_{\mathcal{A}}(\Delta \X^{\star}) , \Delta \bGamma^{\star} ]$ and  $\rank(\mathcal{A})\le 3r$. Thus by triangle inequality, 
		\begin{align}\label{eq:LPGD_last_gap_bound}
			\lVert \Pi_{\mathcal{A}\times \R^{d_{3}\times d_{4}}}\left( \Delta \X^{\star}, \Delta \bGamma^{\star}  \right)  \rVert_{F} &\le \lVert \Pi_{\mathcal{A}}( \Delta \X^{\star})  \rVert_{F} + \lVert  \Delta \bGamma^{\star} \rVert_{F} \le \sqrt{ 3r} \lVert \Delta \X^{\star} \rVert_{2} +  \lVert \Delta \bGamma^{\star} \rVert_{F}.
		\end{align}
		This completes the proof of \textbf{(i)}.
		
		Next, we show \textbf{(ii)}. Suppose $\mathbf{Z}^{\star}$ is a stationary point of $f$ over $\Param$. Then $\Delta \bZ^{\star}=O$ so the first part of the assertion follows from \textbf{(i)}. For the second part, suppose that $\nabla f$ is $L'$-Lipschiz over $\Param$ for some $L'>0$. Then by Cauchy-Schwarz inequality, 
		\begin{align}
			\left| f(\bZ_{n}) - f(\bZ^{\star}) \right| & = \left|  \int_{0}^{1} \left\langle \nabla f \left(\bZ_{n} + s(\bZ^{\star} - \bZ_{n}) \right),\, \bZ_{n}-\bZ^{\star} \right\rangle \,ds \right| \\
			&\le   \int_{0}^{1}\left\lVert  \nabla f \left(\bZ_{n} + s(\bZ^{\star} - \bZ_{n}) \right) \right\rVert \lVert \bZ_{n}-\bZ^{\star} \rVert \,ds \\ 
			&\le   \int_{0}^{1} \left( \left\lVert  \nabla f (\bZ^{\star}) \rVert + s L\lVert \bZ_{n}-\bZ^{\star} \right\rVert  \right) \lVert \bZ_{n}-\bZ^{\star} \rVert \,ds \\ 
			&\le   \left( \left\lVert  \nabla f (\bZ^{\star}) \rVert +  L\lVert \bZ_{n}-\bZ^{\star} \right\rVert  \right) \lVert \bZ_{n}-\bZ^{\star} \rVert.
		\end{align}
		Then \eqref{eq:PSGD_linear_conv2} follows by combining the above inequality with \textbf{(i)}.
	\end{proof}
	
	\begin{remark}\label{rmk:pf_thm_LPGD}
		Note that in \eqref{eq:LPGD_last_gap_bound}, we could have used the following crude bound 
		\begin{align}\label{eq:LPGD_last_gap_bound2}
			\left\lVert \Pi_{\mathcal{A}\times \R^{d_{3}\times d_{4}}}\left( \Delta \X^{\star}, \Delta \bGamma^{\star}  \right)  \right\rVert_{F} \le 	\left\lVert \left[  \Delta \X^{\star}, \Delta \bGamma^{\star}  \right]  \right\rVert_{F} &\le 
			\lVert \Delta \X^{\star} \rVert_{F} +  \lVert \Delta \bGamma^{\star} \rVert_{F} \\
			&\le \sqrt{\rank(\Delta \X^{\star})}	\lVert \Delta \X^{\star} \rVert_{2} +  \lVert \Delta \bGamma^{\star} \rVert_{F},
		\end{align}
		which is also the bound we would have obtained if we choosed the trivial linear subspace $\mathcal{A}=\R^{d_{1}\times d_{2}}$ in the proof of Theorem \ref{thm:CALE_LPGD} above. While we know $\rank(\X^{\star})\le r $, we do not have an a priori bound on $ \rank(\Delta \X^{\star})$, which could be much larger then $\sqrt{3r}$. A smarter choice of the subspace $\mathcal{A}$ as we used in the proof of Theorem \ref{thm:CALE_LPGD} ensures that we only need the factor $\sqrt{3r}$ in place of the unknown factor $\sqrt{\rank(\Delta \X^{\star})}$ as in \eqref{eq:LPGD_last_gap_bound}. 
	\end{remark}

	\subsection{Proof of Theorems \ref{thm:SDL_LPGD} and \ref{thm:SDL_LPGD_feat}} 
	
	Next, prove Theorems \ref{thm:SDL_LPGD} and \ref{thm:SDL_LPGD_feat}, which amounts to verify that the hypothesis of Theorem \ref{thm:CALE_LPGD} holds for the SDL problems in \eqref{eq:SDL_filt_1} and \eqref{eq:SDL_feat_1}.
	
	We begin with some preliminary computation. Let $\a_{s}$ denote the activation corresponding to the $s$th sample (see \eqref{eq:ASDL_1}). More precisely,  $\a_{s}=\A^{T}\x_{s}+\bGamma^{T}\x'_{s}$ for the filter-based model with $\A\in \R^{p\times \kappa}$, and $\a_{s}=\A[:,s]+\bGamma^{T}\x'_{s}$ with $\A\in \R^{\kappa\times n}$. In both cases, $\B\in \R^{p\times n}$ and $\bGamma\in \R^{q\times \kappa}$. Then the objective function $f$ in \eqref{eq:SDL_filt_CALE}   can be written as
	\begin{align}\label{eq:f_SDL_pf0}
		f(\A, \B, \bGamma)&:=\left( -\sum_{s=1}^{n} \sum_{j=0}^{\kappa} \mathbf{1}(y_{i}=j)  \log g_{j}(  \a_{s}  ) \right) +   \xi \lVert  \X_{\textup{data}}  -\B\rVert_{F}^{2} + \nu \left( \lVert \A \rVert_{F}^{2}+\lVert \bGamma \rVert_{F}^{2} \right) \\
		&= \sum_{s=1}^{n} \left(  \log \left( 1+\sum_{c=1}^{\kappa} h( \a_{s}[c]  ) \right) -  \sum_{j=1}^{\kappa} \mathbf{1}(y_{i}=j) \log  h( \a_{s}[j]  )  \right) +   \xi \lVert  \X_{\textup{data}}  -\B\rVert_{F}^{2} + \nu \left( \lVert \A \rVert_{F}^{2} + \lVert \bGamma \rVert_{F}^{2}  \right),
	\end{align}
	where $\a_{s}[i]\in \R$ denotes the $i$th component of $\a_{s}\in \R^{\kappa}$. In the proofs we provided below, we compute the Hessian of $f$ above explicitly for the filter- and the feature-based cases and use Theorem \ref{thm:CALE_LPGD} to derive the result. Recall the functions $\dot{\h}$ and $\ddot{\H}$ introduced in \ref{assumption:A4}. For each label $y\in \{0,\dots,\kappa\}$ and activation $\a\in \R^{\kappa}$, the negative log likelihood of observing label $y$ from the probability distribution $\g(\a)$ defined in \eqref{eq:ell_log_likelihood} can be written as 
	\begin{align}
		\ell_{0}(y,\a) := \log\left( \sum_{c=1}^{\kappa} h(\a[c]) \right)  - \sum_{c=1}^{\kappa}\mathbf{1}(y=c) \log h(\a[c]).
	\end{align}
	Then we have the following relations 
	\begin{align}
		\nabla_{\a} \ell_{0}(y,\a) = \dot{\h}(y,\a), \qquad \nabla_{\a}\nabla_{\a^{T}} \ell_{0}(y,\a) = \ddot{\H}(y,\a).
	\end{align}

	\begin{proof}[\textbf{Proof of Theorem} \ref{thm:SDL_LPGD}]
		Let $f=f_{\textup{SDL-filt}}$ denote the loss function for the filter-based SDL model in \eqref{eq:SDL_filt_CALE}. Fix $\bZ_{1},\bZ_{2}\in \Param\subseteq \R^{d_{1}\times d_{2}}\times \R^{d_{3}\times d_{4}} $. By \ref{assumption:A1} $\Param$ is convex, so $t\bZ_{1} + (1-t) \bZ_{2}\in \Param$ for all $t\in [0,1]$. Then by the mean value theorem, there exists $t^{*}\in [0,1]$ such that for $\bZ^{*}=t^{*}\bZ_{1} + (1-t^{*})\bZ_{2}$,
		\begin{align}\label{eq:RSC_thm_pf}
			f(\bZ_{2}) - f(\bZ_{1}) - \langle \nabla f(\bZ_{1}), \, \bZ_{2}-\bZ_{1} \rangle =  \left( \vect(\bZ_{2}) - \vect(\bZ_{1})  \right)^{T} \nabla_{\vect(\bZ)}\nabla_{\vect(\bZ)^{T}} f( \bZ^{*} )  \left( \vect(\bZ_{2}) - \vect(\bZ_{1})  \right).
		\end{align}
		Hence, according to Theorem \ref{thm:CALE_LPGD}, it suffices to verify that 
		for some $\mu,L>0$ such that $L/\mu<3$, 
		\begin{align}
			\frac{\mu}{2} \I \preceq \nabla_{\vect(\bZ)}\nabla_{\vect(\bZ)^{T}} f( \bZ^{*} ) \preceq \frac{L}{2} \I
		\end{align}
		for all $\bZ^{*}=[\X,\bGamma]$ with $\rank(\X^{*})\le r$.

		
		To this end, let $\a_{s}$ denote the activation corresponding to the $s$th sample (see \eqref{eq:ASDL_1}). More precisely,  $\a_{s}=\A^{T}\x_{s}+\bGamma^{T}\x'_{s}$ for the filter-based model we consider here. We discussed that the objective function $f$ in \eqref{eq:SDL_filt_CALE} can be written as \eqref{eq:f_SDL_pf0}. Denote 
		\begin{align}\label{eq:thm_SDL_filt_as_def}
			\a_{s}= \A^{T}\x_{s}+\bGamma^{T}\x_{s}' =: \vast[ \left\langle \underbrace{\begin{bmatrix}
					\A[:,j] \\ \bGamma[:,j] \end{bmatrix}}_{=:\u_{j}} ,\, \underbrace{\begin{bmatrix} 
					\x_{s} \\ \x_{s}' \end{bmatrix}}_{=:\bphi_{s}}  \right\rangle; \,\, j=1,\dots,\kappa \vast]^{T}\in \R^{\kappa},
		\end{align}
		where we have introduced the notations $\u_{j}\in \R^{(p+q)\times 1}$ for $j=1,\dots,\kappa$  and $\bphi_{s}\in \R^{(p+q) \times 1 }$ for $s=1,\dots, n$. Denote $\U:=[\u_{1},\dots,\u_{\kappa}]\in \R^{(p+q)\times \kappa}$, which is a matrix parameter that combines $\A$ and $\bGamma$. Also denote $\bPhi=(\bphi_{1},\dots,\bphi_{n})\in \R^{(p+q)\times n}$ that combined feature matrix of $n$ observations. Then we can compute the gradient and the Hessian of $f$ above as follows:
		\begin{align}\label{eq:SDL_filt_gradients}
			&\nabla_{\vect(\U)}  f(\U,\B) = \left(  \sum_{s=1}^{n} \dot{\h}(y_{s},\U^{T}\bphi_{s}) \otimes \bphi_{s} \right) + 2\nu \vect(\U), \quad \nabla_{\B}  f(\U,\B) = 2\xi (\B-\X_{\textup{data}}) \\ 	
			&\nabla_{\vect(\U)}\nabla_{\vect(\U)^{T}}  f(\U,\B) =  \left( \sum_{s=1}^{n} \ddot{\H}(y_{s},\U^{T}\bphi_{s}) \otimes \bphi_{s}\bphi_{s}^{T}\right) + 2\nu \I_{(p+q)\kappa}, \\
			&\nabla_{\vect(\B)}\nabla_{\vect(\B)^{T}}  f(\U,\B) = 2\xi \I_{pn}, \qquad \nabla_{\vect(\B)}\nabla_{\vect(\U)^{T}}  f(\U,\B) = O,
		\end{align}
		where $\otimes$ above denotes the Kronecker product and the functions $\dot{\h}$ and $\ddot{\H}$ are defined in  \eqref{eq:Hddot_def}.
		
		Recall that the eigenvalues of $\A\otimes \B$, where $\A$ and $\B$ are two square matrices, are given by $\lambda_{i}\mu_{j}$, where $\lambda_{i}$ and $\mu_{j}$ run over all eigenvalues of $\A$ and $\B$, respectively. Hence denoting $\H_{\U}:=\sum_{s=1}^{N} \ddot{\H}(y_{s},\U^{T}\bphi_{s},) \otimes \bphi_{s}\bphi_{s}^{T}$ and using \ref{assumption:A2}-\ref{assumption:A3}, we can deduce 
		\begin{align}\label{eq:MNL_evals_bounds}
			\lambda_{\min}(\H_{\U}) &\ge n \lambda_{\min}\left( n^{-1} \bPhi \bPhi^{T} \right) 	\min_{1\le s \le N,\, \U} \lambda_{\min}\left( \ddot{\H}(y_{s},\bphi_{s}, \U) \right) \ge n  \delta^{-}\alpha^{-} \ge n \mu^{*}>0, \\
			\lambda_{\max}(\H_{\U}) &\le n \lambda_{\max}\left( n^{-1}\bPhi \bPhi^{T} \right) 	\max_{1\le s \le N,\, \U} \lambda_{\min}\left( \ddot{\H}(y_{s},\bphi_{s}, \U) \right) \le n \delta^{+}\alpha^{+}\le n L^{*}.
		\end{align}
		This holds for all $\A,\B,\bGamma$ such that $\rank([\A,\B])\le r$ and under the convex constraint in \eqref{assumption:A1} (also recall that $\U$ is the vertical stack of $\A$ and $\bGamma$). Hence we conclude that the objective function $f_{\textup{SDL-filt}}$ in \eqref{eq:SDL_filt_CALE} verifies RSC and RSM properties (Def. \ref{def:RSC}) with parameters $\mu=\min(2\xi,2\nu+n \mu^{*})$ and $L=\max(2\xi, 2\nu + n L^{*})$. It is straightforward to verify that $L/\mu<3$ if and only if \eqref{eq:thm_SDL_LGPD_cond} holds. This verifies \eqref{eq:RSC_thm_pf} for the chosen parameters $\mu$ and $L$. 	Then the rest follows from Theorem \ref{thm:CALE_LPGD}. 
	\end{proof}
	
	Next, we prove Theorem \ref{thm:SDL_LPGD_feat}, the exponential convergence of Algorithm \ref{alg:SDL_feat_LPGD} for the feature-based SDL in \eqref{eq:SDL_feat_1}.

	\begin{proof}[\textbf{Proof of Theorem} \ref{thm:SDL_LPGD_feat}]
		We will use the same setup as in the Proof of Theorem \ref{thm:SDL_LPGD}. The main part of the argument is the computation of the Hessian of loss function $f:=f_{\textup{SDL-feat}}$ in \eqref{eq:SDL_feat_CALE}, which is straightforward but a bit more involved than the corresponding computation for the filter-based case in the proof of Theorem \ref{thm:SDL_LPGD}. To this end, let $\a_{s}$ denote the activation corresponding to the $s$th sample (see \eqref{eq:ASDL_1}). More precisely,  $\a_{s}=\A+\bGamma^{T}\x'_{s}$ for the feature-based model we consider here. Recall the objective function $f$ in  \eqref{eq:SDL_feat_CALE} re-written in \eqref{eq:f_SDL_pf0}. We will compute the gradient and the Hessian of $f$ below.

		Recall that for the feature-based model we consider here, we have  $\a_{s}=\A[:,s]+\bGamma^{T}\x'_{s}$, where in this case $\A\in \R^{\kappa\times n}$ (see \eqref{eq:SDL_feat_CALE}). Denote
		\begin{align}
			\a_{s}=\I_{\kappa} \A[:,s]+\bGamma^{T}\x_{s}' =: \vast[ \left\langle \underbrace{\begin{bmatrix} 
					\I_{\kappa}[:,j] \\ \bGamma[:,j] \end{bmatrix}}_{=:\v_{j}} ,\, \underbrace{\begin{bmatrix} 
					\A[:,s] \\ \x_{s}' \end{bmatrix}}_{=:\bpsi_{s}}  \right\rangle; \,\, j=1,\dots,\kappa \vast]^{T}\in \R^{\kappa}.
		\end{align}
		Note that for the feature-based model here, $\A[:,s]$ is concatenated with the auxiliary covariate $\x'_{s}$, whereas we concatenated $\A[:,j]$ with $\bGamma[:,j]$ for the filter-based case (see \eqref{eq:thm_SDL_filt_as_def})\footnote{This is because for the feature-based model, the column $\A[:,s]\in \R^{\kappa}$ for $s=1,\dots,n$ represent a feature of the $s$th sample, whereas for the filter-based model, $\A[:,j]$ for $j=1,\dots,\kappa$ represents the $j$th filter that is applied to the feature $\x_{s}$ of the $s$th sample.}.

		A straightforward computation shows the following gradient formulas:
		\begin{align}\label{eq:SDL_feat_gradients}
			&\nabla_{\vect(\bGamma)}  f(\A,\B,\bGamma) = \left(  \sum_{s=1}^{n} \dot{\h}(y_{s},\a_{s}) \otimes \x'_{s} \right) + 2\nu \vect(\bGamma), \\
			&\nabla_{\vect(\A)}  f(\V,\B) =  \begin{bmatrix} 
				\dot{\h}(y_{1},\a_{1}) \\
				\vdots  \\
				\dot{\h}(y_{n},\a_{n})
			\end{bmatrix}  
			+ 2\nu \vect(\A), \quad \nabla_{\B}  f(\A,\B,\bGamma) = 2\xi(\B-\X_{\textup{data}}) \\ 	
			&\nabla_{\vect(\bGamma)}\nabla_{\vect(\bGamma)^{T}}  f(\A,\B,\bGamma) =  \left(  \sum_{s=1}^{n} \ddot{\H}(y_{s},\a_{s}) \otimes \x_{s}'(\x_{s}')^{T}\right) + 2\nu \I_{q\kappa}, \\
			&\nabla_{\vect(\A)}\nabla_{\vect(\A)^{T}}  f(\A,\B,\bGamma) = \diag\left( \ddot{\H}(y_{1},\a_{1}),\dots, \ddot{\H}(y_{n},\a_{n}) \right)  + 2\nu \I_{\kappa n} \\
			&\nabla_{\vect(\bGamma)}\nabla_{\vect(\A)^{T}}  f(\A,\B,\bGamma) =  \left[ \ddot{\H}(y_{1},\a_{1})\otimes \x'_{1}, \dots, \ddot{\H}(y_{1},\a_{n})\otimes \x'_{n} \right] \in \R^{\kappa q \times \kappa n} \\
			&\nabla_{\vect(\B)}\nabla_{\vect(\B)^{T}}  f(\A,\B,\bGamma) = 2\xi \I_{pn}, \qquad \nabla_{\vect(\B)}\nabla_{\vect(\V)^{T}}  f(\A,\B,\bGamma) = O.
		\end{align}
		From this we will compute the eigenvalues of the Hessian $\H_{\textup{feat}}$ of the loss function $f$. In order to illustrate our computation in a simple setting, we first assume $\kappa=1=q$, which corresponds to binary classification $\kappa=1$ with one-dimensional auxiliary covariates $q=1$. In this case, we have 
		\begin{align}
			\H_{\textup{feat}}&:=\nabla_{\vect(\A,\bGamma,\B)} \nabla_{\vect(\A,\bGamma,\B)^{T}} f(\A,\B,\bGamma) \\
			&=
			\begin{bmatrix}
				\ddot{h}(y_{1},\a_{1})+2\nu  & 0  & \dots & 0 &  \ddot{h}(y_{1},\a_{1}) x_{1}'  & O \\
				0 & \ddot{h}(y_{2},\a_{2})+2\nu   &\dots & 0 &  \ddot{h}(y_{2},\a_{2}) x_{2}'  & O \\
				\vdots & \vdots & \ddots & \vdots & \vdots & \vdots \\
				0 & \dots & 0 &  \ddot{h}(y_{n},\a_{n})+2\nu  &  \ddot{h}(y_{n},\a_{n}) x_{n}' & O \\
				\ddot{h}(y_{1},\a_{1}) x_{1}' & \ddot{h}(y_{2},\a_{2}) x_{2}' & \dots & \ddot{h}(y_{n},\a_{n}) x_{n}' & \left( \frac{1}{n}\sum_{s=1}^{n}\ddot{h}(y_{s},\a_{s}) (x_{s}')^{2} \right) + 2\nu   & O \\
				O&O&\dots &O &O & 2\xi \I_{pn} 
			\end{bmatrix},
		\end{align}
		where we denoted $\ddot{h}=\ddot{h}_{11}\in \R$ and $x'_{s}=\x_{s}'\in \R$ for $s=1,\dots,n$. In order to compute the eigenvalues of the above matrix, we will use the following formula for determinant of $3\times 3$ block matrix: ($O$ representing matrices of zero entries with appropriate sizes)
		\begin{align}
			\det\left( \begin{bmatrix} 
				A & B & O \\
				B^{T} & C & O \\ 
				O & O & D
			\end{bmatrix} \right) 
			= \det\left( C - B^{T}A^{-1}B \right) \det(A) \det(D).
		\end{align}
		This yields the following simple formula for the characteristic polynomial of $\H_{\textup{feat}}$:
		\begin{align}
			\det( \H_{\textup{feat}} - \lambda \I) &= \left(  \sum_{s=1}^{n} \ddot{h}(y_{s},\a_{s}) (x_{s}')^{2}  -  \sum_{s=1}^{n} \frac{(\ddot{h}(y_{s},\a_{s}))^{2} (x_{s}')^{2}}{\ddot{h}(y_{s},\a_{s})+2\nu } + 2\nu  - \lambda \right) (2\xi - \lambda) \prod_{s=1}^{n} \left( \ddot{h}(\y_{s},\a_{s}) + 2\nu- \lambda\right)\\
			&=\left( \sum_{s=1}^{n}  \frac{2\nu  \ddot{h}(y_{s},\a_{s}) (x_{s}')^{2}}{ \ddot{h}(y_{s},\a_{s})+2\nu } + 2\nu - \lambda \right) (2\xi - \lambda)^{pn} \prod_{s=1}^{n} \left( \ddot{h}(\y_{s},\a_{s}) + 2\nu- \lambda\right).
		\end{align}
		By \ref{assumption:A4}, we know that $\ddot{h}(y_{s},\a_{s})>0$ for all $s=1,\dots,n$, so the first term in the parenthesis in the above display is lower bounded by $2\nu-\lambda$. It follows that 
		\begin{align}
			\lambda_{\min}(\H_{\textup{feat}}) & \ge \min(2\xi, \alpha^{-}+2\nu), \\
			\lambda_{\max}(\H_{\textup{feat}}) &\le \max\left( 2\nu + \alpha^{+}\sum_{s=1}^{n} (x_{s}')^{2} ,\, 2\xi,\, \alpha^{+} + 2\nu\right).
		\end{align}
		
		Now we generalize the above computation for general $\kappa,q\ge 1$ case. First note the general form of the Hessian as below:
		{\small
			\begin{align}
				&\H_{\textup{feat}}:=\nabla_{\vect(\A,\bGamma,\B)} \nabla_{\vect(\A,\bGamma,\B)^{T}} f(\A,\B,\bGamma) \\
				&\quad =
				\begin{bmatrix}
					\ddot{\H}(y_{1},\a_{1})+2\nu  \I_{\kappa} & 0  & \dots & 0 & (\ddot{\H}(y_{1},\a_{1})\otimes \x_{1}')^{T} & O \\
					0 & \ddot{\H}(y_{2},\a_{2})+2\nu  \I_{\kappa}  &\dots & 0 & (\ddot{\H}(y_{2},\a_{2})\otimes \x_{2}')^{T} & O \\
					\vdots & \vdots & \ddots & \vdots & \vdots & \vdots \\
					0 & \dots & 0 &  \ddot{\H}(y_{n},\a_{n})+2\nu  \I_{\kappa} & (\ddot{\H}(y_{n},\a_{n})\otimes \x_{n}')^{T} & O \\
					\ddot{\H}(y_{1},\a_{1})\otimes \x_{1}' & \ddot{\H}(y_{2},\a_{2})\otimes \x_{2}' & \dots & \ddot{\H}(y_{n},\a_{n})\otimes \x_{n}' & \begin{matrix}  \sum_{s=1}^{n}\ddot{\H}(y_{s},\a_{s}) \otimes \x_{s}'(\x_{s})^{T}  \\  + 2\nu  \I_{q\kappa}   \end{matrix} & O \\
					O&O&\dots &O &O & 2\xi  \I_{pn} 
				\end{bmatrix}.
			\end{align}
		}
		Note that for any square symmetric matrix $B$ and a column vector $\x$ of matching size, 
		\begin{align}
			B\otimes \x \x^{T} - (B\otimes \x)^{T}(B+\lambda \I)^{-1} (B\otimes \x) &= \left( B-B(B+\lambda \I)^{-1} B \right) \otimes (\x\x^{T})  \\
			&= (B+\lambda I)^{-1} B \otimes \x\x^{T} \\
			& \preceq \I  \otimes \x\x^{T},
		\end{align}
		where the last diagonal dominace is due to the Woodbury identity for matrix inverse (e.g., see \cite{horn2012matrix}). Hence by a similar computation as before, we obtain
		\begin{align}
			\det( n\H_{\textup{feat}} - \lambda \I) 
			& = \det\left( \sum_{s=1}^{n} 2\nu \left( \ddot{\H}(y_{s},\a_{s}) + 2\nu  \I_{\kappa} \right)^{-1} \ddot{\H}(y_{s},\a_{s}) \otimes \x_{s}'(\x_{s}')^{T}  + (2\nu -\lambda) \I_{q\kappa}  \right) (2\xi n - \lambda)^{pn} \\
			&\hspace{5cm} \times \prod_{s=1}^{n} \det\left( \ddot{\H}(\y_{s},\a_{s}) + (2\nu - \lambda)\I_{\kappa} \right).
		\end{align}
		It follows that 
		\begin{align}
			\lambda_{\min}(\H_{\textup{feat}}) & \ge \min(2\xi,  \alpha^{-}+2\nu), \\
			\lambda_{\max}(\H_{\textup{feat}}) &\le \max\left(2\nu +\alpha^{+} n \lambda_{\max}\left( n^{-1}\X_{\textup{aux}} \X_{\textup{aux}}^{T}  \right),\, 2\xi,\,  \alpha^{+} + 2\nu\right).
		\end{align}
		Then the rest follows from Theorem \ref{thm:SDL_LPGD_feat}.
	\end{proof}

	\section{Proof of Theorem \ref{thm:SDL_BCD}}
	
	In this section, we prove Theorem \ref{thm:SDL_BCD} only for the case of filter-based SDL in \eqref{eq:ASDL_1}. An almost identical argument will show the assertion for the feature-based case. 
	
	Recall the filter-based SDL loss function $f(\A,\B,\bGamma)$ in \eqref{eq:f_SDL_pf0} in terms of the combined variables $[\A,\B,\bGamma]$, where $\A=\W\Beta$ and $\B=\W\H$. For convenience, recall that $\A\in \R^{\kappa\times p}$, $\W\in \R^{p\times r}$, $\Beta\in \R^{r\times \kappa}$, $\H\in \R^{r\times n}$, and $\bGamma\in \R^{q\times n}$.  In the following computations, we will use \textit{commutation matrix} $\C^{(a\times b)}$, which is a special instance of $ab\times ab$ permutation matrix. Namely, for each integers $a,b\ge 1$, there exists a unique matrix $\C^{(a\times b)}\in \{0,1\}^{ab\times ab}$ such that for all $A\in \R^{a\times b}$, we have $\C^{(a,b)}\vect(A) = \vect(A^{T})$. Note that $(\C^{(a,b)})^{T}=\C^{b\times a}$. Furthermore, $(\C^{(a,b)})^{T}\C^{(a,b)}=\I_{ab}$ since $(\C^{(a,b)})^{T}\C^{(a,b)}\vect(A)=\C^{(b,a)}\vect(A^{T})=\vect(A)$. Hence $\C^{(a,b)}$ is positive semi-definite. Throughout this section, we denote $\bZ=[\W,\H,\Beta,\bGamma]$ for the combined SDL parameters. 
	
	\begin{lemma}[Derivatives of the filter-based SDL objective in separate variables]
		\label{lem:SDL_filt_BCD_derivatives}
		Let $L(\bZ)$ denote the objective of the filter-based SDL in \eqref{eq:ASDL_1}. Suppose \ref{assumption:A4} holds. Recall $\dot{\h}$ and $\ddot{\H}$ defined in \eqref{eq:Hddot_def}. Let $\a_{s}:=\Beta^{T}\W^{T}\x_{s}+\bGamma^{T}\x'_{s}$ for $s=1,\dots,n$ and $\K:=[\dot{\h}(y_{1},\a_{1}),\dots, \dot{\h}(y_{1},\a_{1})]\in \R^{\kappa\times n}$. Then we have
		\begin{align}\label{eq:SDL_filt_BCD_derivatives1}
			\nabla_{\W} \,  L(\bZ) &= \X_{\textup{data}} \K^{T} \Beta^{T} +  2\xi(\W\H-\X_{\textup{data}})\H^{T},\qquad
			\nabla_{\Beta} \, L(\bZ)  = \W^{T} \X_{\textup{data}}  \K^{T} \\
			\nabla_{\bGamma} \, L(\bZ) &=  \X_{\textup{aux}}\K^{T}, \qquad  \nabla_{\H} \, L(\bZ) =  2\xi \W^{T}(\W\H-\X_{\textup{data}}).
		\end{align}
		Furthermore, for diagonal terms in the Hessian, we have 
		\begin{align}
			\nabla_{\vect(\W)}\nabla_{\vect(\W)^{T}} \, L(\bZ) 
			&= (\Beta \otimes \X_{\textup{data}}) \diag\left( \ddot{\H}(y_{1},\a_{1}),\dots,\ddot{\H}(y_{n},\a_{n})\right) \C^{(n,\kappa)} (\Beta \otimes \X_{\textup{data}})^{T} +2\xi (\H\H^{T} \otimes \I_{p} ), \\
			\nabla_{\vect(\H)}\nabla_{\vect(\H)^{T}} \, L(\bZ) & = 2\xi (\I_{n}\otimes \W^{T}\W), \\
			\nabla_{\vect(\Beta)}\nabla_{\vect(\Beta)^{T}} \, L(\bZ) 
			&=(\I_{\kappa}\otimes \W^{T}\X_{\textup{data}}) \diag\left( \ddot{\H}(y_{1},\a_{1}),\dots,\ddot{\H}(y_{n},\a_{n})\right)   (\I_{\kappa}\otimes \W^{T}\X_{\textup{data}})^{T}, \\
			\nabla_{\vect(\bGamma)}\nabla_{\vect(\bGamma)^{T}} \, L(\bZ)
			&= (\I_{r}\otimes \X_{\textup{aux}}) \diag\left( \ddot{\H}(y_{1},\a_{1}),\dots,\ddot{\H}(y_{n},\a_{n})\right)  (\I_{r}\otimes \X_{\textup{aux}})^{T}.
		\end{align}
		Lastly, for the off-diagonal terms in the Hessian, we have 
		\begin{align}
			\nabla_{\vect(\Beta)}\nabla_{\vect(\W)^{T}} \, L(\bZ) 
			&= \C^{(\kappa,  n)} (\I_{\kappa}\otimes \W^{T}\X_{\textup{data}})\diag\left( \ddot{\H}(y_{1},\a_{1}),\dots,\ddot{\H}(y_{n},\a_{n})\right)    (\Beta \otimes \X_{\textup{data}})^{T}  +  \C^{(\kappa, r)} (\I_{r}\otimes \X_{\textup{data}} \K^{T} )^{T}, \\
			\nabla_{\vect(\bGamma)}\nabla_{\vect(\W)^{T}} \, L(\bZ) &= O,\\
			\nabla_{\vect(\H)}\nabla_{\vect(\W)^{T}} \, L(\bZ) &= 2\xi \left[ (\H^{T}\otimes \W^{T}) + (\I_{r}\otimes \H^{T}\W^{T})   -  \C^{(n\times r)} (\I_{r}\otimes \X_{\textup{data}})^{T} \right],\\
			\nabla_{\vect(\Beta)}\nabla_{\vect(\H)^{T}} \, L(\bZ) &= \nabla_{\vect(\bGamma)}\nabla_{\vect(\H)^{T}} \, L(\bZ) =  O,\\
			\nabla_{\vect(\bGamma)}\nabla_{\vect(\Beta)^{T}} \, L(\bZ) &= (\I_{r}\otimes \X_{\textup{aux}}) \diag\left( \ddot{\H}(y_{1},\a_{1}),\dots,\ddot{\H}(y_{n},\a_{n})\right)  \C^{(n,\kappa)} (\I_{\kappa}\otimes \W^{T}\X_{\textup{data}})^{T}.
		\end{align}
	\end{lemma}

	\begin{proof}
		Setting $\nu=0$ in \eqref{eq:f_SDL_pf0}, we have $L(\bZ)=f(\A,\B,\bGamma)$. Recall the gradients of $f(\A,\B,\bGamma)$ in \eqref{eq:SDL_filt_gradients}. By using the chain rule and noting that $\A[:,j]=\W \Beta[:,j]$, we can compute
		\begin{align}
			\nabla_{\W} f(\A,\B,\bGamma) &= \left( \sum_{s=1}^{n} \sum_{j=1}^{\kappa} \frac{\partial f(\A,\B,\bGamma)}{\partial \A[:,j]} \frac{\partial \A[:,j]}{\partial \W}\right) + 2\xi(\W\H-\X_{\textup{data}})\H^{T}  \\
			&= \sum_{s=1}^{n}\sum_{j=1}^{\kappa} \x_{s} \dot{h}_{j}(y_{s},\a_{s}) \Beta[:,j]^{T} + 2\xi(\W\H-\X_{\textup{data}})\H^{T} \\
			&= \X_{\textup{data}} \K^{T} \Beta^{T} +  2\xi(\W\H-\X_{\textup{data}})\H^{T},
		\end{align}
		By a similar computation, we can also compute, for each $j=1,\dots,\kappa$,
		\begin{align}
			\nabla_{\Beta[:,j]} L(\bZ) &= \sum_{s=1}^{n}  \frac{\partial \A[:,j]}{\partial \Beta[:,j]}\frac{\partial f(\A,\B,\bGamma)}{\partial \A[:,j]}  = \sum_{s=1}^{n} \W^{T}  \dot{h}_{j}(y_{s},\a_{s}) \x_{s}, 
		\end{align}
		so we get $\nabla_{\Beta} f(\A,\B,\bGamma)  = \W^{T} \X_{\textup{data}}  \K^{T}$. A similar computation shows the remaining two gradients.

		Next, recall the relations for vectorizing product of matrices: for $A\in \R^{a\times b}$, $B\in \R^{b\times c}$, and $C\in \R^{c\times d}$, 
		\begin{align}
			\vect(AB) &= (\I_{c}\otimes A) \vect(B) = (B^{T} \otimes \I_{a}) \vect(A), \\
			\vect(ABC) &= (C^{T}\otimes A)\vect(B) = (\I_{d}\otimes AB) \vect(C) = (C^{T}B^{T}\otimes \I_{a}) \vect(A).
		\end{align}
		From this the previous calculation yields 
		\begin{align}
			\nabla_{\vect(\W)} f(\A,\B,\bGamma) 
			&= \vect\left( \X_{\textup{data}} \K^{T} \Beta^{T} \right) +  2\xi \vect(\W\H\H^{T}) - 2\xi \vect(\X_{\textup{data}} \H^{T}) \\
			&= (\Beta \otimes \X_{\textup{data}}) \vect(\K^{T}) + 2\xi (\H\H^{T} \otimes \I_{p} ) \vect(\W)  - 2\xi \vect(\X_{\textup{data}} \H^{T}).
		\end{align}
		Note that $\vect(\K^{T})^{T}=(\C^{(\kappa,n)} \vect(\K))^{T}=\vect(\K)^{T} \C^{(n,\kappa)}$. Hence we get 
		\begin{align}
			\nabla_{\vect(\W)}\nabla_{\vect(\W)^{T}} f(\A,\B,\bGamma) 
			&= \nabla_{\vect(\W)}\left( \vect(\K)^{T} \C^{(n,\kappa)} (\Beta \otimes \X_{\textup{data}})^{T}  + 2\xi  \vect(\W)^{T}  (\H\H^{T} \otimes \I_{p} ) - 2\xi \vect(\X_{\textup{data}} \H^{T})^{T} \right)\\
			&= (\Beta \otimes \X_{\textup{data}}) \diag\left( \ddot{\H}(y_{1},\a_{1}),\dots,\ddot{\H}(y_{n},\a_{n})\right) \C^{(n,\kappa)} (\Beta \otimes \X_{\textup{data}})^{T} +2\xi (\H\H^{T} \otimes \I_{p} ).
		\end{align}
		Similarly, we can compute 
		\begin{align}
			\nabla_{\vect(\Beta)}\nabla_{\vect(\Beta)^{T}} f(\A,\B,\bGamma) &=\nabla_{\vect(\Beta)} \vect(\W^{T} \X_{\textup{data}}  \K^{T})^{T}\\
			&= \nabla_{\vect(\Beta)} \vect(\K)^{T} (\I_{\kappa}\otimes \W^{T}\X_{\textup{data}})^{T} \\
			&=(\I_{\kappa}\otimes \W^{T}\X_{\textup{data}}) \diag\left( \ddot{\H}(y_{1},\a_{1}),\dots,\ddot{\H}(y_{n},\a_{n})\right)   (\I_{\kappa}\otimes \W^{T}\X_{\textup{data}})^{T}.
		\end{align}
		Also note that 
		\begin{align}
			\nabla_{\vect(\bGamma)}\nabla_{\vect(\bGamma)^{T}} f(\A,\B,\bGamma) &= \nabla_{\vect(\bGamma)} \vect(\X_{\textup{aux}} \K^{T})^{T} \\
			&= \nabla_{\vect(\bGamma)} \vect(\K)^{T} (\I_{r}\otimes \X_{\textup{aux}})^{T} \\
			&= (\I_{r}\otimes \X_{\textup{aux}}) \diag\left( \ddot{\H}(y_{1},\a_{1}),\dots,\ddot{\H}(y_{n},\a_{n})\right)  (\I_{r}\otimes \X_{\textup{aux}})^{T}.
		\end{align}
		Similarly, we get 
		\begin{align}
			\nabla_{\vect(\H)}\nabla_{\vect(\H)^{T}} f(\A,\B,\bGamma) &= \nabla_{\vect(\H)} \left( 2\xi \vect(\W^{T}\W\H)^{T} - 2\xi \vect(\W^{T}\X_{\textup{data}})^{T}\right) \\
			&= 2\xi \nabla_{\vect(\H)} \vect(\H)^{T} (\I_{n}\otimes \W^{T}\W)  = (\I_{n}\otimes \W^{T}\W).
		\end{align}
		
		Next, we compute the off-diagonal terms in the Hessian of $f$. First, we compute 
		\begin{align}
			\nabla_{\vect(\Beta)}\nabla_{\vect(\W)^{T}} f(\A,\B,\bGamma) &= \nabla_{\vect(\Beta)} \vect(\X_{\textup{data}} \K^{T} \Beta^{T})^{T} \\
			&= \left( \frac{\partial}{\partial \vect(\Beta) } \vect(\Beta)^{T}\C^{(\kappa, r)} \right) (\I_{r}\otimes \X_{\textup{data}} \K^{T} )^{T} + \left( \frac{\partial}{\partial \vect(\Beta)}  \C^{(\kappa\times n)} \vect(\K^{T}) \right)  (\Beta \otimes \X_{\textup{data}})^{T}   \\
			&= \C^{(\kappa, r)} (\I_{r}\otimes \X_{\textup{data}} \K^{T} )^{T} + \\
			&\qquad + \C^{(\kappa, n)} (\I_{\kappa}\otimes \W^{T}\X_{\textup{data}})\diag\left( \ddot{\H}(y_{1},\a_{1}),\dots,\ddot{\H}(y_{n},\a_{n})\right)  (\Beta \otimes \X_{\textup{data}})^{T}.
		\end{align}
		Second, note that $\nabla_{\vect(\bGamma)}\nabla_{\vect(\W)^{T}} f(\A,\B,\bGamma)=O$. Third, for the forthcoming computation, we claim that 
		\begin{align}
			\nabla_{\vect(\H)} \vect(\H\H^{T})^{T} = (\H^{T}\otimes \I_{r}) + (\I_{r}\otimes \H^{T}).
		\end{align}
		One can directly verify the above when $\H$ consists of a single column, and the general case can be easily obtained from there. Also note that using the commutation matrix, we can write $\vect(\H^{T})^{T}=(\C^{(r,n)} \vect(\H))^{T}=\vect(\H)^{T} \C^{(n,r)}$. Now observe that 
		\begin{align}
			\nabla_{\vect(\H)}\nabla_{\vect(\W)^{T}} f(\A,\B,\bGamma) &= 2\xi \nabla_{\vect(\H)} \left[ \vect(\W\H\H^{T}) - \vect(\X_{\textup{data}} \H^{T}) \right]^{T} \\
			&= 2\xi \nabla_{\vect(\H)} \left[ \vect(\H\H^{T})^{T} (\I_{r}\otimes \W)^{T}  -  \vect(\H^{T})^{T}(\I_{r}\otimes \X_{\textup{data}})^{T}  \right] \\
			&=   2\xi \left( \nabla_{\vect(\H)}  \vect(\H\H^{T})^{T} \right) (\I_{r}\otimes \W)^{T}  -  \left( \nabla_{\vect(\H)}  \vect(\H^{T})^{T} \right) (\I_{r}\otimes \X_{\textup{data}})^{T} \\
			&=  2\xi \left[ \left( (\H^{T}\otimes \I_{r}) + (\I_{r}\otimes \H^{T}) \right) (\I_{r}\otimes \W)^{T}  -  \C^{(n\times r)} (\I_{r}\otimes \X_{\textup{data}})^{T} \right].
		\end{align}
		Then one can use the mixed-product property to further simplify the last expression as in the assertion. Fourth, noting that $\vect(\K^{T})^{T}=(\C^{(\kappa,n)} \vect(\K))^{T}=\vect(\K)^{T} \C^{(n,\kappa)}$, we can compute  
		\begin{align}
			\nabla_{\vect(\bGamma)}\nabla_{\vect(\Beta)^{T}} f(\A,\B,\bGamma) &= \nabla_{\vect(\bGamma)} \vect(\W^{T} \X_{\textup{data}}  \K^{T})^{T} \\
			&= \nabla_{\vect(\bGamma)} \vect(\K^{T})^{T} (\I_{\kappa}\otimes \W^{T}\X_{\textup{data}})^{T}\\
			&= \nabla_{\vect(\bGamma)} \vect(\K)  \C^{(n,\kappa)} (\I_{\kappa}\otimes \W^{T}\X_{\textup{data}})^{T} \\
			&= (\I_{r}\otimes \X_{\textup{aux}}) \diag\left( \ddot{\H}(y_{1},\a_{1}),\dots,\ddot{\H}(y_{n},\a_{n})\right)  \C^{(n,\kappa)} (\I_{\kappa}\otimes \W^{T}\X_{\textup{data}})^{T}.
		\end{align}
		The remaining zero-second derivatives are easy to see. 
	\end{proof}

	\begin{remark}[Derivatives for the feature-based SDL objective in separate variables]
		\label{rmk:feat_SDL_gradient}
		Arguing similarly as in the proof of Lemma \ref{lem:SDL_filt_BCD_derivatives}, we can compute the derivatives of the feature-based SDL objective in separate variables as follows. Let $L(\bZ)$ denote the objective of the feature-based SDL in \eqref{eq:ASDL_1}. Suppose \ref{assumption:A4} holds. Recall $\dot{\h}$ defined in \eqref{eq:Hddot_def}. Let $\a_{s}:=\Beta^{T}\h_{s}+\bGamma^{T}\x'_{s}$ for $s=1,\dots,n$, where $\H=[\h_{1},\dots,\h_{n}]\in \R^{r\times n}$ being the code matrix. Let $\K:=[\dot{\h}(y_{1},\a_{1}),\dots, \dot{\h}(y_{1},\a_{1})]\in \R^{\kappa\times n}$. Then we 
		\begin{align}\label{eq:SDL_feat_BCD_derivatives1}
			\nabla_{\W} \,  L(\bZ) &= 2\xi(\W\H-\X_{\textup{data}})\H^{T},\qquad
			\nabla_{\Beta} \, L(\bZ)  = \H  \K^{T} \\
			\nabla_{\bGamma} \, L(\bZ) &=  \X_{\textup{aux}}\K^{T}, \qquad  \nabla_{\H} \, L(\bZ) =  \Beta \K +   2\xi \W^{T}(\W\H-\X_{\textup{data}}).
		\end{align}
	\end{remark}

	\begin{lemma}\label{lem:SDL_BCD_smooth}
		\label{lem:BCD_hypothesis}
		Assume the hypothesis of Theorem \ref{thm:SDL_BCD} is true. Then the loss function $L$ in \eqref{eq:ASDL_1} is convex in each block coordinates $\W$, $\H$, $\Beta$, and $\bGamma$. Furthermore, its gradient is continuous and $M$-Lipschitz for some $M>0$ on the admissible parameter space. 
	\end{lemma}
	
	\begin{proof}
		Notice that the diagonal terms in the Hessian given in  Lemma \ref{lem:SDL_filt_BCD_derivatives} are positive semidefinite. (An explicit lower bound on the eigenvalues can also be computed.) This is enough to conclude that $L$ is convex in each factor $\W$, $\H$, $\Beta$, and $\bGamma$ while all the other three are held fixed (i.e., $L$ is multiconvex).  The second part of the assertion follows easily from the first derivative computations in Lemma \ref{lem:SDL_filt_BCD_derivatives} and the compactness assumption in Theorem \ref{thm:SDL_BCD}. 
	\end{proof}
	
	Now we are ready to derive Theorem \ref{thm:SDL_BCD}.
	
	\begin{proof}[\textbf{Proof of Theorem \ref{thm:SDL_BCD}}]
		The result would immediately follow the main result in \cite{lyu2020convergence}. In order to apply the result, we need to verify that 1) the filter-based SDL loss function $L$ in \eqref{eq:ASDL_1} is multiconvex and 2) the gradient of $L$ is $M$-Lipschitz for some constant $M>0$. Under the assumption, \ref{assumption:A4} and the assumed compactness of the parameter space, both of these hypotheses are verified in Lemma \ref{lem:BCD_hypothesis}. This shows the assertion. 
	\end{proof}

	\section{Proof of Theorem \ref{thm:SDL_BCD_STAT_filt}}
	\label{section:appendix_SDL_BCD_STAT}
	
	Throughout this section, let $\mathcal{L}(\bZ)=\mathcal{L}_{n}(\W,\h,\Beta,\bGamma)$ denote the objective in \eqref{eq:SDL_likelihood_BCD_filter}. Denote 
	\begin{align}\label{eq:def_L_bar_STAT}
		\bar{\mathcal{L}}(\bZ) := \E_{(\x,\x',y)}\left[ \mathcal{L}_{1}(\bZ) \right].
	\end{align}
	
	\begin{lemma}[Derivatives of the filter-based SDL objective in separate variables]
		\label{lem:SDL_filt_BCD_STAT_derivatives}
		Let $\mathcal{L}(\bZ)=\mathcal{L}_{n}(\W,\h,\Beta,\bGamma)$ denote the objective in \eqref{eq:SDL_likelihood_BCD_filter}. Suppose \ref{assumption:A4} holds. Recall $\dot{\h}$ and $\ddot{\H}$ defined in \eqref{eq:Hddot_def}. Let $\a_{s}:=\Beta^{T}\W^{T}\x_{s}+\bGamma^{T}\x'_{s}$ for $s=1,\dots,n$ and $\K:=[\dot{\h}(y_{1},\a_{1}),\dots, \dot{\h}(y_{1},\a_{1})]\in \R^{\kappa\times n}$. Also denote $\H=[\h,\dots,\h]\in \R^{r\times n}$. Then we have
		\begin{align}
			\nabla_{\W} \,  \mathcal{L}(\bZ) &= \X_{\textup{data}} \K^{T} \Beta^{T} +  2\xi(\W\H-\X_{\textup{data}})\H^{T} + 2\nu\W,\qquad
			\nabla_{\Beta} \, \mathcal{L}(\bZ)  = \W^{T} \X_{\textup{data}}  \K^{T} + 2\nu \Beta, \\
			\nabla_{\bGamma} \, \mathcal{L}(\bZ) &=  \X_{\textup{aux}}\K^{T} + 2\nu \bGamma, \qquad  \nabla_{\h} \, \mathcal{L}(\bZ) =  2\xi \W^{T}(n\W\h-\sum_{s=1}^{n}\x_{s}) + 2\nu \h.
		\end{align}
		Furthermore, for diagonal terms in the Hessian, we have 
		\begin{align}
			\nabla_{\vect(\W)}\nabla_{\vect(\W)^{T}} \, \mathcal{L}(\bZ) 
			&= (\Beta \otimes \X_{\textup{data}}) \diag\left( \ddot{\H}(y_{1},\a_{1}),\dots,\ddot{\H}(y_{n},\a_{n})\right) \C^{(n,\kappa)} (\Beta \otimes \X_{\textup{data}})^{T} +2\xi (\H\H^{T} \otimes \I_{p} ) + 2\nu \I_{pr}, \\
			\nabla_{\h}\nabla_{\h^{T}} \, \mathcal{L}(\bZ) & = 2\xi n\W^{T}\W + 2\nu \I_{r}, \\
			\nabla_{\vect(\Beta)}\nabla_{\vect(\Beta)^{T}} \, \mathcal{L}(\bZ) 
			&=(\I_{\kappa}\otimes \W^{T}\X_{\textup{data}}) \diag\left( \ddot{\H}(y_{1},\a_{1}),\dots,\ddot{\H}(y_{n},\a_{n})\right)   (\I_{\kappa}\otimes \W^{T}\X_{\textup{data}})^{T} + 2\nu \I_{r\kappa}, \\
			\nabla_{\vect(\bGamma)}\nabla_{\vect(\bGamma)^{T}} \, \mathcal{L}(\bZ)
			&= (\I_{r}\otimes \X_{\textup{aux}}) \diag\left( \ddot{\H}(y_{1},\a_{1}),\dots,\ddot{\H}(y_{n},\a_{n})\right)  (\I_{r}\otimes \X_{\textup{aux}})^{T} +  2\nu \I_{q\kappa}
		\end{align}
		Lastly, for the off-diagonal terms in the Hessian, we have 
		\begin{align}
			\nabla_{\vect(\Beta)}\nabla_{\vect(\W)^{T}} \, \mathcal{L}(\bZ) 
			&=  (\I_{\kappa}\otimes \W^{T}\X_{\textup{data}})\diag\left( \ddot{\H}(y_{1},\a_{1}),\dots,\ddot{\H}(y_{n},\a_{n})\right)  (\Beta \otimes \X_{\textup{data}})^{T}  +  \C^{(\kappa, r)} (\I_{\r}\otimes \X_{\textup{data}} \K^{T} )^{T}, \\
			\nabla_{\vect(\bGamma)}\nabla_{\vect(\W)^{T}} \, \mathcal{L}(\bZ) &= O,\\
			\nabla_{\h}\nabla_{\vect(\W)^{T}} \, \mathcal{L}(\bZ) &= 2\xi \left[ n(\h^{T}\otimes \W^{T}) + n(\I_{r}\otimes \h^{T}\W^{T})   -  (\mathbf{1}_{1\times n}\otimes \I_{r}) \C^{(n\times r)} (\I_{r}\otimes \X_{\textup{data}})^{T} \right],\\
			\nabla_{\vect(\Beta)}\nabla_{\h^{T}} \, \mathcal{L}(\bZ) &= \nabla_{\vect(\bGamma)}\nabla_{\h^{T}} \, \mathcal{L}(\bZ) =  O,\\
			\nabla_{\vect(\bGamma)}\nabla_{\vect(\Beta)^{T}} \, \mathcal{L}(\bZ) &= (\I_{r}\otimes \X_{\textup{aux}}) \diag\left( \ddot{\H}(y_{1},\a_{1}),\dots,\ddot{\H}(y_{n},\a_{n})\right)  \C^{(n,\kappa)} (\I_{\kappa}\otimes \W^{T}\X_{\textup{data}})^{T}.
		\end{align}
	\end{lemma}
	
	\begin{proof}
		For $\H=[\h,\dots,\h]$, note that 
		\begin{align}
			\nabla_{\h} \vect(\H)^{T} = \mathbf{1}_{1\times n}\otimes \I_{r} ,\qquad    \nabla_{\h} \vect(\H\H^{T})^{T} = n(\h^{T}\otimes \I_{r}) + n(\I_{r}\otimes \h^{T}).
		\end{align}
		Then the assertion follows from similar computations as in the proof of Lemma \ref{lem:SDL_filt_BCD_derivatives}.
	\end{proof}

	\begin{lemma}\label{lem:SDL_BCD_STAT_smooth}
		\label{lem:BCD_STAT_hypothesis}
		Let $\mathcal{L}_{n}(\bZ)=\mathcal{L}_{n}(\W,\h,\Beta,\bGamma)$ denote the objective in \eqref{eq:SDL_likelihood_BCD_filter}. Assume the hypothesis of Theorem \ref{thm:SDL_BCD_STAT_filt} holds. Then $\mathcal{L}_{n}$ is convex in each block coordinates $\W$, $\h$, $\Beta$, and $\bGamma$ for all $\nu\ge 0$. Furthermore, its gradient is continuous and $M$-Lipschitz for some $M>0$ on the admissible parameter space. 
	\end{lemma}
	
	\begin{proof}
		The argument is identical to the proof Lemma \ref{lem:BCD_hypothesis} using Lemma \ref{lem:SDL_filt_BCD_STAT_derivatives} instead of Lemma \ref{lem:SDL_filt_BCD_derivatives}. For the multi-convexity part, recall that $\ddot{\H}(y_{s},\a_{s})$'s are positive definite due to \ref{assumption:A4} and the commutation matrices $\C^{(a,b)}$ are also positive definite. 
	\end{proof}
	
	\begin{lemma}[Derivatives of the expected filter-based SDL objective]
		\label{lem:SDL_filt_BCD_STAT_derivatives}
		Suppose a single data $(\x,\x',y)$ is sampled according to the generative model \eqref{eq:SDL_prob_BCD_filt} and let $\bar{\mathcal{L}}$ be as in \eqref{eq:def_L_bar_STAT}. Assume the hypothesis of Theorem \ref{thm:SDL_BCD_STAT_filt} holds. 
		Recall $\dot{\h}$ and $\ddot{\H}$ defined in \eqref{eq:Hddot_def}. Let $\a:=\Beta^{T}\W^{T}\x+\bGamma^{T}\x'$ and $\xi:=(2\sigma^{2})^{-1}$. Then we have
		\begin{align}
			\nabla_{\W} \,  \bar{\mathcal{L}}(\bZ) &= \E[\x \, \dot{\h}(y,\a)^{T}] \Beta^{T} +  2\xi(\W\h-\W^{\star}\h^{\star})\h^{T} + 2\nu \W,\qquad
			\nabla_{\Beta} \, \bar{\mathcal{L}}(\bZ)  = \W^{T} \E[\x  \,\dot{\h}(y,\a)^{T}] + 2\nu \Beta\\
			\nabla_{\bGamma} \, \bar{\mathcal{L}}(\bZ) &=  \E[\x' \dot{\h}(y,\a)^{T}] + 2\nu \bGamma \qquad  \nabla_{\h} \, \bar{\mathcal{L}}(\bZ) =  2\xi \W^{T}(\W\h-\W^{\star}\h^{\star}) + 2\nu \h.
		\end{align}
		Furthermore, for diagonal terms in the Hessian, we have 
		\begin{align}
			\nabla_{\vect(\W)}\nabla_{\vect(\W)^{T}} \, \bar{\mathcal{L}}(\bZ) 
			&=\E\left[ (\Beta \otimes \x)  \ddot{\H}(y,\a) (\Beta \otimes \x)^{T} \right]+2\xi (\h\h^{T} \otimes \I_{p} ) + 2\nu \I_{pr} \, \\
			\nabla_{\h}\nabla_{\h^{T}} \, \bar{\mathcal{L}}(\bZ) & = 2\xi \W^{T}\W +  2\nu \I_{r},\\
			\nabla_{\vect(\Beta)}\nabla_{\vect(\Beta)^{T}} \, \bar{\mathcal{L}}(\bZ) 
			&=\E\left[(\I_{\kappa}\otimes \W^{T}\x)  \ddot{\H}(y,\a)  (\I_{\kappa}\otimes \W^{T}\x)^{T}\right] + 2\nu \I_{r\kappa}, \\
			\nabla_{\vect(\bGamma)}\nabla_{\vect(\bGamma)^{T}} \, \bar{\mathcal{L}}(\bZ)
			&= \E\left[(\I_{r}\otimes \x') \ddot{\H}(y,\a) (\I_{r}\otimes \x')^{T}\right]+ 2\nu \I_{q\kappa}.
		\end{align}
		Lastly, for the off-diagonal terms in the Hessian, we have 
		\begin{align}
			\nabla_{\vect(\Beta)}\nabla_{\vect(\W)^{T}} \, \bar{\mathcal{L}}(\bZ) 
			&= \E\left[  (\I_{\kappa}\otimes \W^{T}\x)  \ddot{\H}(y,\a) (\Beta \otimes \x)^{T}  +  \C^{(\kappa\times r)} (\I_{r}\otimes \x \, \dot{\h}(y, \a)^{T} )^{T} \right], \\
			\nabla_{\vect(\bGamma)}\nabla_{\vect(\W)^{T}} \, \bar{\mathcal{L}}(\bZ) &= O,\\
			\nabla_{\h}\nabla_{\vect(\W)^{T}} \, \bar{\mathcal{L}}(\bZ) &= 2\xi \left[ (\h^{T}\otimes \W^{T}) + (\I_{r}\otimes (\W\h - \W^{\star}\h^{\star})^{T})  \right],\\
			\nabla_{\vect(\Beta)}\nabla_{\h^{T}} \, \bar{\mathcal{L}}(\bZ) &= \nabla_{\vect(\bGamma)}\nabla_{\h^{T}} \, \bar{\mathcal{L}}(\bZ) =  O,\\
			\nabla_{\vect(\bGamma)}\nabla_{\vect(\Beta)^{T}} \, \bar{\mathcal{L}}(\bZ) &= \E\left[ (\I_{r}\otimes \x')  \ddot{\H}(y,\a) (\I_{\kappa}\otimes \W^{T}\x)^{T} \right].
		\end{align}
	\end{lemma}
	
	\begin{proof}
		According to Lemmas \ref{lem:SDL_filt_BCD_STAT_derivatives} and \ref{lem:SDL_BCD_STAT_smooth}, $\mathcal{L}$ is twice continuously differentiable and both $\nabla \mathcal{L}$ and $\nabla^{2} \mathcal{L}$ are bounded within the (compact) parameter space. Hence by the monotone convergence theorem, 
		\begin{align}
			\nabla \bar{\mathcal{L}} = \E\left[ \nabla \mathcal{L}\right], \qquad \nabla^{2} \bar{\mathcal{L}} = \E\left[ \nabla^{2} \mathcal{L}\right].
		\end{align}
		Hence we can simply specialize the derivatives of $\mathcal{L}$ we computed in Lemma \ref{lem:SDL_filt_BCD_STAT_derivatives} for the single sample case $n=1$ and then take the expectation. In doing so, we use the fact that $\C^{(1\times a)}=\I_{a}$, where $\C^{(a,b)}$ denotes the commutation matrix defined above the statement of Lemma \ref{lem:SDL_filt_BCD_derivatives}. 
	\end{proof}

	\begin{lemma}\label{lem:SDL_BCD_avg_smooth}
		\label{lem:BCD_STAT_hypothesis}
		Suppose a single data $(\x,\x',y)$ is sampled according to the generative model \eqref{eq:SDL_prob_BCD_filt} with true parameters $\bZ^{\star}=[\W^{\star}, \H^{\star}, \Beta^{\star}, \bGamma^{\star}]$ and $\blambda^{\star}$. Let $\bar{\mathcal{L}}$ be as in \eqref{eq:def_L_bar_STAT} and assume the hypothesis of Theorem \ref{thm:SDL_BCD_STAT_filt} holds. Then the following hold: 
		\begin{description}
			\item[(i)] $\bar{\mathcal{L}}$ is convex in each block coordinates $\W$, $\h$, $\Beta$, and $\bGamma$ for all $\nu\ge 0$.
			
			\item[(ii)] $\nabla \bar{\mathcal{L}}$ continuous and $M$-Lipschitz for some $M>0$ on the admissible parameter space. 
			
			\item[(iii)] $\nabla^{2} \bar{\mathcal{L}}(\bZ^{\star})$ is positive definite if $\nu > \lambda_{+}$, where
			\begin{align}
				\lambda_{+}:= \frac{1}{2}\max\left( \begin{matrix} 2\xi \lVert \h^{\star}  \rVert_{2} \, \lVert  \W^{\star} \rVert_{2} + \alpha^{+} \lVert \Beta^{\star} \rVert_{2}  \lVert \W^{\star} \rVert_{2} \E\left[  \Vert \x\x^{T} \rVert_{2} \right] +  \gamma_{\max}\sigma \sqrt{2p\pi}  \\
					\hspace{3cm} - \alpha^{-} \, \lambda_{\min}(\Beta^{\star} (\Beta^{\star})^{T} ) \E\left[   \lambda_{\min}( \x\x^{T}  ) \right] - 2\xi \lambda_{\min} (\h^{\star}(\h^{\star})^{T} ), \\[10pt] 
					2\xi \lVert \h^{\star}  \rVert_{2} \, \lVert  \W^{\star} \rVert_{2} - 2\xi (\W^{\star})^{T}\W^{\star} - 2\xi \lambda_{\min}( (\W^{\star})^{T}\W^{\star}), \\[10pt] 
					\alpha^{+} \lVert \Beta^{\star} \rVert_{2}  \lVert \W^{\star} \rVert_{2} \E\left[  \Vert \x\x^{T} \rVert_{2} \right] +  \gamma_{\max}\sigma \sqrt{2p\pi}  +\alpha^{+} \E\left[ \lVert \x'(\W^{\star})^{T}\x \rVert_{2}\right] \\
					\hspace{4cm} - \alpha^{-} \, \E\left[  \lambda_{\min}( (\W^{\star})^{T} \x\x^{T} \W^{\star} ) \right], \\[10pt] 
					\alpha^{+} \E\left[ \lVert \x'(\W^{\star})^{T}\x \rVert_{2}\right] - \lambda_{\min}(\x'(\x')^{T})
				\end{matrix}\right) \label{eq:nu_min_BCD_STAT_explicit}.
			\end{align}
		\end{description}
	\end{lemma}
	
	\begin{proof}
		Parts \textbf{(i)} and \textbf{(ii)} follow easily from Lemma \ref{lem:SDL_filt_BCD_STAT_derivatives} as in the proof of Lemma \ref{lem:BCD_STAT_hypothesis}. Now we argue for \textbf{(iii)}. Denote $\xi=(2\sigma^{2})^{-1}$ and  $\a^{\star}=(\W^{\star}\Beta^{\star})^{T}\x+\bGamma^{\star}\x' $. Recall that according to Lemma \ref{lem:SDL_filt_BCD_STAT_derivatives}, we can write the Hessian of $\bar{\mathcal{L}}$ at the true parameter $\bZ^{\star}$ as the $4\times 4$ block matrix $(A_{ij})_{1\le i,j\le 4} + 2\nu \I$ in \eqref{eq:expected_hessian_block}. Recall \ref{assumption:A4}. The diagonal blocks are given by 
		\begin{align}
			A_{11}
			&=\E\left[ (\Beta^{\star} \otimes \x)  \ddot{\H}(y,\a^{\star}) (\Beta^{\star} \otimes \x)^{T} \right]+2\xi (\h^{\star}(\h^{\star})^{T} \otimes \I_{p} ) \\
			&\qquad \succeq \left( \alpha^{-} \, \lambda_{\min}(\Beta^{\star} (\Beta^{\star})^{T} ) \E\left[   \lambda_{\min}( \x\x^{T}  ) \right] + 2\xi \lambda_{\min} (\h^{\star})(\h^{\star})^{T}  \right) \I_{pr}   \\ 
			A_{22} & = 2\xi (\W^{\star})^{T}\W^{\star} \succeq 2\xi \lambda_{\min}( (\W^{\star})^{T}\W^{\star}) \I_{r},\\
			A_{33}&=\E\left[(\I_{\kappa}\otimes (\W^{\star})^{T}\x)  \ddot{\H}(y,\a^{\star})  (\I_{\kappa}\otimes (\W^{\star})^{T}\x)^{T}\right] \succeq \alpha^{-} \E\left[ \lambda_{\min}( (\W^{\star})^{T}\x \x^{T} \W^{\star} ) \right] \I_{r\kappa}, \\
			A_{44} 
			&= \E\left[(\I_{r}\otimes \x') \ddot{\H}(y,\a^{\star}) (\I_{r}\otimes \x')^{T}\right] \succeq \alpha^{-} \lambda_{\min}(\x'(\x')^{T}) \I_{rq},
		\end{align}
		where the off-diagonal blocks are given by 
		\begin{align}
			A_{21} &=   2\xi  (\h^{\star})^{T}\otimes (\W^{\star})^{T} \\
			A_{31} &=  \E\left[  (\I_{\kappa}\otimes \W^{T}\x)  \ddot{\H}(y,\a^{\star}) (\Beta \otimes \x)^{T}  +  \C^{(\kappa\times r)} (\I_{r}\otimes \x \, \dot{\h}(y,\a^{\star})^{T} )^{T} \right]  \\
			A_{43} &=  \E\left[ (\I_{r}\otimes \x')  \ddot{\H}(y,\a^{\star}) (\I_{\kappa}\otimes \W^{T}\x)^{T} \right].
		\end{align}
		If $\nu$ is large enough so that the following condition is satisfied
		\begin{align}
			\lambda_{\min}(A_{ii}) + 2\nu > \sum_{j\ne i} \lVert A_{ij} \rVert_{2} \quad \forall 1\le i \le 4,
		\end{align}
		then the Hessian $\nabla^{2} \bar{\mathcal{L}}$ is block diagonally dominant and is positive definite (see \cite{feingold1962block}). Thus it suffices to take 
		\begin{align}\label{eq:lambda_lower_bd0}
			\nu &> \frac{1}{2}\max\left( \begin{matrix} \lVert A_{12} \rVert_{2} +\lVert A_{13} \rVert_{2}   - \lambda_{\min}(A_{11}), \\
				\lVert A_{12} \rVert_{2}- \lambda_{\min}(A_{22}), \\
				\lVert A_{13} \rVert_{2} +\lVert A_{34} \rVert_{2} - \lambda_{\min}(A_{33}), \\
				\lVert A_{34} \rVert_{2} - \lambda_{\min}(A_{44})
			\end{matrix}\right).
		\end{align}
		Note that (see \ref{assumption:A4} for the definition of $\gamma_{\max}$ and $\alpha^{\pm}$)
		\begin{align}
			\lVert \E\left[  \x \, \dot{\h}(y,\a^{\star})^{T} \right] \rVert_{2} &\le \lVert \E\left[  \W^{\star} \h^{\star} \dot{\h}(y,\a^{\star})^{T} \right] \rVert_{2} + \lVert \E\left[ \beps \dot{\h}(y,\a^{\star})^{T} \right] \rVert_{2} \\
			&\le  \lVert \W^{\star} \h^{\star} \E\left[  \dot{\h}(y,\a^{\star})^{T} \right] \rVert_{2} + \lVert \E\left[ \beps \dot{\h}(y,\a^{\star})^{T} \right] \rVert_{2} \\ 
			&= \lVert \E\left[ \beps \,\dot{\h}(y, \a^{\star})^{T} \right] \rVert_{2} \\
			&\le \gamma_{\max}\sigma \sqrt{2p\pi}.
		\end{align}
		Using this, we get 
		\begin{align}
			\lVert A_{12} \rVert_{2} &= \lVert A_{21} \rVert_{2} =2\xi \lVert \h^{\star}  \rVert_{2} \, \lVert  \W^{\star} \rVert_{2}, \\
			\lVert A_{13} \rVert_{2} &\le \alpha^{+} \lVert \Beta^{\star} \rVert_{2}  \lVert \W^{\star} \rVert_{2} \E\left[  \Vert \x\x^{T} \rVert_{2} \right] +  \gamma_{\max}\sigma \sqrt{2p\pi}, \\
			\lVert A_{43} \rVert_{2} &\le \alpha^{+} \E\left[ \lVert \x'(\W^{\star})^{T}\x \rVert_{2}\right].
		\end{align}
		Using these upper bounds on the operator norm of the off-diagonal blocks and the lower bounds on the eigenvalues of the diagonal blocks above, the lower bound in \eqref{eq:lambda_lower_bd0} can be lower bounded by $\lambda_{+}$ defined in the assertion. 
	\end{proof}

	\section{Proof of Theorems \ref{thm:SDL_LPGD_STAT} and \ref{thm:SDL_LPGD_STAT_feat}}
	
	We first recall the following standard concentration bounds:

	\begin{lemma}[Generalized Hoeffding's inequality for sub-gaussian variables]
		\label{lem:hoeffding}
		Let $X_{1},\dots,X_{n}$ denote i.i.d. random vectors in $\R^{d}$ such that $\E[X_{k}[i]^{2}/K^{2}]\le 2$ for some constant $K>0$ for all $1\le k \le n$ and $1\le i \le d$. Fix a vector $\a=(a_{1},\dots,a_{n})^{T}\in \R^{n}$. Then for each $t>0$, 
		\begin{align}
			\P\left( \left\lVert \sum_{k=1}^{n} a_{k}X_{k} \right\rVert_{1} > t \right) \le 2d \exp\left( \frac{-t^{2}}{K^{2} d^{2}\lVert \a  \rVert_{2}^{2}} \right)
		\end{align}
	\end{lemma}
	
	\begin{proof}
		Follows from \cite[Thm 2.6.2]{vershynin2018high} and using a union bound over $d$ coordinates. 
	\end{proof}

	\begin{lemma}(2-norm of matrices with independent sub-gaussian entries)
		\label{lem:matrix_norm_bd}
		Let $\A$ be an $m\times n$ random matrix with independent subgaussian entries $\A_{ij}$ of mean zero. Denote $K$ to be the maximum subgaussian norm of $\A_{ij}$, that is, $K>0$ is the smallest number such that $\E[\exp(\A_{ij})^{2}/K^{2}  ]\le 2$. Then for each $t>0$,
		\begin{align}
			\P\left( \lVert \A \rVert_{2} \ge 3K(\sqrt{m}+\sqrt{n}+t) \right) \le 2\exp(-t^{2}).
		\end{align}
	\end{lemma}
	
	\begin{proof}
		See \cite[Thm. 4.4.5]{vershynin2018high}
	\end{proof}

	\begin{proof}[\textbf{Proof of Theorem \ref{thm:SDL_LPGD_STAT}}]
		Let $\mathcal{L}_{n}$ denote the $L_{2}$-regularized negative joint negative log likelihood function in \eqref{eq:SDL_likelihood_conv_filter} without the last three terms, and define the  expected loss function $\bar{\mathcal{L}}_{n}(\bZ):= \E_{\beps_{i},\beps_{i}',1\le i \le n}\left[ \mathcal{L}_{n}(\bZ) \right]$. We omit the constant terms in these functions. Define the following gradient mappings of $\bZ^{\star}$ with respect to the empirical $f_{n}$ and the expected $\bar{f}_{n}$ loss functions: 
		\begin{align}
			G(\bZ^{\star}, \tau)=\frac{1}{\tau}\left( \bZ^{\star} - \Pi_{\Param} \left( \bZ^{\star} - \tau \nabla \mathcal{L}_{n}(\bZ^{\star}) \right)\right), \qquad  \bar{G}(\bZ^{\star},\tau):=\frac{1}{\tau}\left( \bZ^{\star} - \Pi_{\Param} \left( \bZ^{\star} - \tau \nabla \bar{\mathcal{L}}_{n}(\bZ^{\star}) \right)\right).
		\end{align}
		It is elementary to show that the true parameter $\bZ^{\star}$ is a stationary point of $\bar{\mathcal{L}}- \nu (\lVert \A\rVert_{F}^{2} + \lVert \bGamma \rVert_{F}^{2}) $ over $\Param\subseteq \R^{p\times (\kappa+n)}\times \R^{q\times \kappa}$. Hence we have $\bar{G}(\bZ^{\star},\tau)= 2\nu[\A^{\star},O,\bGamma^{\star}]$, so we may write 
		\begin{align}\label{eq:grad_mapping_compare_stationary}
			G(\bZ^{\star}, \tau) &= G(\bZ^{\star}, \tau) - \bar{G}(\bZ^{\star}, \tau) + 2\nu[\A^{\star},O,\bGamma^{\star}] \\
			&= \frac{1}{\tau}\left[  \Pi_{\Param}\left( \bZ^{\star}-\tau\nabla \mathcal{L}_{n}(\bZ^{\star}) \right) - \Pi_{\Param}\left( \bZ^{\star}-\tau\nabla \bar{\mathcal{L}}_{n}(\bZ^{\star}) \right) \right] +  2\nu[\A^{\star},O,\bGamma^{\star}]
		\end{align}
		
		First, suppose $\bZ^{\star}-\tau \nabla  \mathcal{L}_{n} (\bZ^{\star})\in \Param$ (In particular, this is the case whe $\Param$ equals the whole space). Then we can disregard the projection $\Pi_{\Param}$ in the above display so we get 
		\begin{align}
			G(\bZ^{\star}, \tau) - 2\nu[\A^{\star},O,\bGamma^{\star}] = \nabla \mathcal{L}_{n}(\bZ^{\star}) - \nabla \bar{\mathcal{L}}(\bZ^{\star}) =: [\Delta \X^{\star}, \Delta \bGamma^{\star}].
		\end{align}
		According to Theorem \ref{thm:SDL_LPGD}, it now suffices show that $G(\bZ^{\star}, \tau)$ above is small with high probability. We use the notation $\U=[\A^{T}, \bGamma^{T}]^{T}$, $\U^{\star}=[(\A^{\star})^{T}, (\bGamma^{\star})^{T}]^{T}$, $\bPhi=[\bphi_{1},\dots,\bphi_{n}]=[\X_{\textup{data}}^{T}, \X_{\textup{aux}}^{T}]^{T}$ (see also the proof of Theorem \ref{thm:SDL_LPGD}). Denote $\a_{s}=\U^{T}\bphi_{s}$ and $\a_{s}^{\star}=(\U^{\star})^{T}\bphi_{s}$ for $s=1,\dots,n$ and introduce the following random quantities 
		\begin{align}
			\mathtt{Q}_{1}:= \sum_{s=1}^{n} \dot{\h}(y_{s},\a_{s}^{\star})\in \R^{\kappa} ,\quad \mathtt{Q}_{2}:=\sum_{s=1}^{n}  \beps_{s}\in \R^{p} ,\quad \mathtt{Q}_{3}:=\sum_{s=1}^{n}  \beps_{s}'\in \R^{q}, \quad \mathtt{Q}_{4}:= [ \beps_{1},\dots,\beps_{n}]  \in \R^{p\times n}.
		\end{align}

		Recall that 
		\begin{align}
			&\nabla_{\vect(\U)}  \mathcal{L}_{n}(\U,\B) = \left( \sum_{s=1}^{n} \dot{\h}(y_{s},\a_{s}) \otimes \bphi_{s} \right) + 2\nu \vect(\U), \quad \nabla_{\B}  \mathcal{L}_{n}(\U,\B) = \frac{2}{2\sigma^{2}} (\B-\X_{\textup{data}}), \\ 	
			&\nabla_{\vect(\U)}  \bar{\mathcal{L}}_{n}(\U,\B) = \left(  \sum_{s=1}^{n} \E\left[ \dot{\h}(y_{s},\a_{s}) \otimes \bphi_{s} \right] \right) + 2\nu \vect(\U), \quad \nabla_{\B}  \bar{\mathcal{L}}_{n}(\U,\B) = \frac{2}{2\sigma^{2}} (\B-\B^{\star}),
		\end{align}
		where $\dot{\h}$ is defined in \eqref{eq:Hddot_def}. Note that 
		\begin{align}
			\E\left[ \dot{\h}(y_{s},\a_{s}) \,\bigg|\, \bphi_{s} \right]  		&= \left[ \left(\frac{h'(\a[j])}{1+\sum_{c=1}^{\kappa} h(\a[c])} - g_{j}(\a_{s}^{\star})\frac{h'(\a[j])}{h(\a[j])}  \right)_{\a=\a_{s}} \, ; \, j=1,\dots,\kappa \right]\\
			&= \left[ \left(\frac{h'(\a[j])}{1+\sum_{c=1}^{\kappa} h(\a[c])} - \frac{h(\a_{s}^{\star}[j])}{1+\sum_{c=1}^{\kappa} h(\a_{s}^{\star}[c])} \frac{h'(\a[j])}{h(\a[j])}  \right)_{\a=\a_{s}} \, ; \, j=1,\dots,\kappa \right],
		\end{align}
		so the above vanishes when $\a_{s}=\a_{s}^{\star}$. Hence 
		\begin{align}\label{eq:dot_h_exp_vanish}
			\E\left[ \dot{\h}(y_{s},\a_{s}^{\star}) \otimes \bphi_{s} \right] = \E\left[ \E\left[ \dot{\h}(y_{s},\a_{s}^{\star}) \otimes \bphi_{s}\,\bigg|\, \bphi_{s} \right]  \right] =\mathbf{0},
		\end{align}
		Hence we can compute the following gradients 
		\begin{align}
			\nabla_{\vect(\A)} (\mathcal{L}_{n} - \bar{\mathcal{L}}_{n})(\A,\B,\bGamma)   &=\left(  \sum_{s=1}^{n} \dot{\h}(y_{s},\a_{s}) \otimes \x_{s} \right)  \\
			\nabla_{\vect(\bGamma)} (\mathcal{L}_{n} - \bar{\mathcal{L}}_{n})(\A,\B,\bGamma) &=\left( \sum_{s=1}^{n} \dot{\h}(y_{s},\a_{s}) \otimes \x_{s}' \right) \\
			\nabla_{\B}  (\mathcal{L}_{n} - \bar{\mathcal{L}}_{n})(\A,\B,\bGamma) &=\frac{2}{2\sigma^{2})} (\B^{\star}-\X_{\textup{data}}) = \frac{2}{2\sigma^{2}} [ \beps_{1},\dots,\beps_{n}] \\
			\nabla_{\blambda}  (\mathcal{L}_{n} - \bar{\mathcal{L}}_{n})(\A,\B,\bGamma)  &=\frac{2}{2\sigma^{2}} \sum_{s=1}^{n}\beps_{s}' .
		\end{align}
		It follows that (recall the definition of $\gamma_{\max}$in \ref{assumption:A4}) 
		\begin{align}
			\lVert \nabla_{\A} (\mathcal{L}_{n} - \bar{\mathcal{L}}_{n})(\A^{\star},\B^{\star},\bGamma^{\star})  \rVert_{2}  &= \left\rVert\sum_{s=1}^{n} (\B^{\star}[:,s]+\beps_{s}) \dot{\h}(y_{s},\a_{s}^{\star})^{T} \right\rVert_{2}  \\
			& \le \left\rVert  \sum_{s=1}^{n} \B^{\star}[:,s] \dot{\h}(y_{s},\a_{s}^{\star})^{T} \right\rVert_{2}  +\left\rVert \sum_{s=1}^{n} \beps_{s} \dot{\h}(y_{s},\a_{s}^{\star})^{T} \right\rVert_{2}  \\
			&\le  \lVert \B^{\star}\rVert_{\infty} \left\rVert  \mathtt{Q}_{1} \right\rVert_{2}  +  \gamma_{\max}  \left\rVert \mathtt{Q}_{2}  \right\rVert_{2}. 
		\end{align}
		Similarly, we have 
		\begin{align}
			\lVert \Delta \bGamma^{\star} \rVert_{F}=	\lVert \nabla_{\bGamma} (\mathcal{L}_{n} - \bar{\mathcal{L}}_{n})(\A^{\star},\B^{\star},\bGamma^{\star})  \rVert_{F} &= \lVert \nabla_{\vect(\bGamma)} (\mathcal{L}_{n} - \bar{\mathcal{L}}_{n})(\A^{\star},\B^{\star},\bGamma^{\star})  \rVert_{2} \\
			&\le  q\lVert \blambda^{\star}\rVert_{\infty}  \left\rVert \mathtt{Q}_{1}  \right\rVert_{2}  + q\gamma_{\max}  \left\rVert \mathtt{Q}_{3}  \right\rVert_{2}
		\end{align}
		Using the fact that $\lVert [A,B] \rVert_{2}\le \lVert A \rVert_{2} +  \lVert B \rVert_{2} $ for two matrices $A,B$ with the same number of rows, we have 
		\begin{align}\label{eq:SDL_MLE_pf_bd_Q}
			\left\rVert \Delta \X^{\star}  \right\rVert_{2} &= \left\lVert  \nabla_{\A} (\mathcal{L}_{n} - \bar{\mathcal{L}}_{n})(\A^{\star},\B^{\star},\bGamma^{\star})   \right\rVert_{2}  + \left\lVert  \nabla_{\bGamma} (\mathcal{L}_{n} - \bar{\mathcal{L}}_{n})(\A^{\star},\B^{\star},\bGamma^{\star})   \right\rVert_{2}  \\
			&\le  \lVert \B^{\star}\rVert_{\infty} \left\rVert \mathtt{Q}_{1} \right\rVert_{2}  + n\gamma_{\max}  \left\rVert \mathtt{Q}_{2} \right\rVert_{2} +\frac{2}{2\sigma^{2}} \left\lVert  \mathtt{Q}_{4} \right\rVert_{2}. 
		\end{align}
		Thus, combining the above bounds, we obtain 
		\begin{align}\label{eq:SDL_MLE_pf_bd_Q2}
			S:=  \sqrt{3r} \lVert \Delta \X^{\star} \rVert_{2} + \lVert \Delta \bGamma^{\star} \rVert_{F}  \le  \sum_{i=1}^{4} c_{i} \lVert \mathtt{Q}_{i} \rVert_{2},
		\end{align}
		where the constants $c_{1},\dots,c_{4}>0$ are given by 
		\begin{align}\label{eq:c_constants_Q}
			c_{1}=\left( \sqrt{3r} \lVert \B^{\star}\rVert_{\infty} + q\lVert \blambda^{\star} \rVert_{\infty} \right), \quad c_{2}=\gamma_{\max}\left( q + \sqrt{3r} \right),\quad c_{3}=q\gamma_{\max}, \quad c_{4}=\frac{2\sqrt{3}r}{2\sigma^{2}}.
		\end{align}

		Next, we will use concentration inequalities to argue that the right hand side in \eqref{eq:SDL_MLE_pf_bd_Q2} is small with high probability and obtain the following tail bound on $S$: 
		\begin{align}\label{eq:S_tail_bound}
			\P\left(S>c \sqrt{n} \log n  + 3C\sigma(\sqrt{p} + \sqrt{n}+ c\sqrt{\log n})   \right)  \le \frac{1}{n},
		\end{align}
		where $C>0$ is an absolute constant and $c>0$ can be written explicitly in terms of the constants we use in this proof. Recall that for a random variable $Z$, its sub-Gaussian norm, denoted as $\lVert Z \rVert_{\psi_{2}}$, is the smalleset number $K>0$ such that $\E[\exp(Z^{2}/K^{2})]\le 2$. The constant $C>$ above is the sub-gaussian norm of the standard normal variable, which can be taken as $C\le 36e/\log 2$. Using union bound with Lemmas \ref{lem:hoeffding} and \ref{lem:matrix_norm_bd}, for each $t,t>0$, we get 
		\begin{align}\label{eq:S_bd_pf}
			&\P\left( S > (c_{1}+c_{2}+c_{3}+c_{4})t  + 3C\sigma(\sqrt{p} + \sqrt{n}+ t')   \right)  \\
			&\qquad \le \left( \sum_{i=1}^{3} \P\left( \lVert \mathtt{Q}_{i} \rVert_{2}>t\right)  \right) + \P\left( \lVert n\mathtt{Q}_{4} \rVert_{2}> 3C\sigma(\sqrt{p} + \sqrt{n}+ t')    \right)  \\ 
			&\qquad \le 2\kappa \exp\left( \frac{-t^{2} }{C_{1}^{2} \kappa^{2} n} \right) + 2p \exp\left( \frac{-t^{2}  }{ (C \sigma)^{2} p^{2} n } \right) + 2q \exp\left( \frac{-t^{2}  }{ (C \sigma')^{2}  q^{2} n } \right) + \exp(-(t')^{2}).
		\end{align}
		Indeed, for bounding $\P(\mathtt{Q}_{1}>t)$, we used Lemma \ref{lem:hoeffding} with sub-Gaussian norm $C_{1}=K=\gamma_{\max}/\sqrt{\log 2} $ for the bounded random vector $\dot{\h}(y_{s},\a_{s})$ (see \cite[Ex. 2.5.8]{vershynin2018high}); for $\P(\mathtt{Q}_{2}>t)$ and $\P(\mathtt{Q}_{3}>t)$, we used Lemma \ref{lem:hoeffding} with $K=C \sigma$ and $K=C \sigma'$, respectively; for the last term involving $\mathtt{Q}_{4}$, we used Lemma \ref{lem:matrix_norm_bd} with $K=C/\sigma$. Observe that in order to make the last expression in \eqref{eq:S_bd_pf} small, we will chose $t=c_{5}\sqrt{n}\log n$ and $t'=c_{5}\sqrt{\log n}$, where $c_{5}>0$ is a constant to be determined. This yields
		\begin{align}
			\P\left( S> c \sqrt{n} \log n  + 3C\sigma(\sqrt{p} + \sqrt{n}+ c\sqrt{\log n})   \right) \le  n^{-c_{6}},
		\end{align}
		where $c=c_{5}\sum_{i=1}^{4} c_{i}$ and $c_{6}>0$ is an explicit constant that grows in $c_{5}$. We assume $c_{5}>0$ is such that $c_{6}\ge 1$. This shows \eqref{eq:S_tail_bound}.
		
		To finish, we use Theorem \ref{thm:SDL_LPGD} to deduce that with probability at least $1/n$,
		\begin{align}
			\lVert \bZ_{t} - \bZ^{\star}  \rVert_{F}  \le  \rho^{t}  \, \lVert  \bZ_{0} - \bZ^{\star}\rVert_{F} + \frac{\tau}{1-\rho}\left( c \sqrt{n} \log n  + 3C\sigma(\sqrt{p} + \sqrt{n}+ c\sqrt{\log n})  \right) + \frac{2\nu\tau}{1-\rho}\left( \lVert \A^{\star} \rVert_{2} + \lVert \bGamma^{\star} \rVert_{F} \right)
		\end{align}
		Note that $\tau<\frac{3}{2L}$ with $L = \max(2\xi,\, 2\nu+ n L^{*}) \ge n L^{*}$, so $\tau<\frac{3}{2nL^{*}}$. So this yields the desired result.

		Second, suppose $\bZ^{\star}-\tau \nabla f_{\textup{SDL-filt}}(\bZ^{\star})\notin \Param$. Then we cannot direcly simplify the expression \eqref{eq:grad_mapping_compare_stationary}. In this case, we take the Frobenius norm and use non-expansiveness of the projection operator (onto convex set $\Param$): 
		\begin{align}\label{eq:grad_mapping_compare_stationary2}
			\lVert G(\bZ^{\star}, \tau) \rVert_{F} &=  \frac{1}{\tau} \left\rVert\left[  \Pi_{\Param}\left( \bZ^{\star}-\tau\nabla \mathcal{L}_{n}(\bZ^{\star}) \right) - \Pi_{\Param}\left( \bZ^{\star}-\tau\nabla \bar{\mathcal{L}}_{n}(\bZ^{\star}) \right) \right] \right\rVert_{F}   \\
			&\le \lVert \nabla \mathcal{L}_{n}(\bZ^{\star})-  \nabla \bar{\mathcal{L}}_{n}(\bZ^{\star}) \rVert_{F} \\
			& \le  \lVert \Delta \X^{\star} \rVert_{F} + \lVert \Delta \bGamma^{\star} \rVert_{F}. 
		\end{align}
		According to Remark \ref{rmk:pf_thm_LPGD}, we also have Theorem \ref{thm:CALE_LPGD} (and hence Theorem \ref{thm:SDL_LPGD}) with $\sqrt{3r} \lVert \Delta \X^{\star} \rVert_{2}$ replaced with $\lVert \Delta \X^{\star} \rVert_{F}$. Then an identical argument shows 
		\begin{align}\label{eq:SDL_MLE_pf_bd_Q4}
			S':= \lVert \Delta \X^{\star} \rVert_{F} + \lVert \Delta \bGamma \rVert_{F} \le c_{1} \lVert \mathtt{Q}_{1} \rVert_{2}  + c_{2}\lVert \mathtt{Q}_{2} \rVert_{2} + c_{3}\lVert \mathtt{Q}_{3} \rVert_{2} + c_{4}\lVert \mathtt{Q}_{4} \rVert_{F},
		\end{align}
		where the constants $c_{1},\dots,c_{4}>0$ are the same as in \eqref{eq:c_constants_Q}. So we have 
		\begin{align}
			\lVert \bZ_{t} - \bZ^{\star} \rVert_{F} \le \rho^{t}  \, \lVert  \bZ_{0} - \bZ^{\star}\rVert_{F} + \frac{\tau}{1-\rho}(S' + 2\nu (\lVert \A^{\star} \rVert_{2}+\lVert \A^{\star} \rVert_{F})). 
		\end{align}
		Then an identical argument with the inequality $\lVert \mathtt{Q}_{4}\rVert_{F} \le \sqrt{\min(p,n)} \lVert \mathtt{Q}_{4}\rVert_{2}$ shows 
		\begin{align}\label{eq:S_bd_pf2}
			&\P\left( S' > (c_{1}+c_{2}+c_{3}+c_{4})t  + 3C\sigma(\sqrt{p} + \sqrt{n}+ t')  \sqrt{\min(p,n)}  \right) \\
			&\qquad \le \left( \sum_{i=1}^{3} \P\left( \lVert \mathtt{Q}_{i} \rVert_{2} >t\right)  \right) + \P\left( \lVert \mathtt{Q}_{4} \rVert_{2} > \frac{3C(\sqrt{p} + \sqrt{n}+ t')  }{\sigma}  \right),
		\end{align}
		and the assertion follows similarly as before. 
	\end{proof}
	
	\begin{proof}[\textbf{Proof of Theorem \ref{thm:SDL_LPGD_STAT_feat}}]
		The argument is entirely similar to the proof of Theorem \ref{thm:SDL_LPGD_STAT}. Indeed, denoting $\a_{s}=\A[:,s]+\bGamma^{T}\x_{s}'$ for $s=1,\dots,n$ and keeping the other notations the same as in the proof of Theorem \ref{thm:SDL_LPGD_STAT}, we can compute the following gradients 
		\begin{align}
			\nabla_{\A} (\mathcal{L}_{n} - \bar{\mathcal{L}}_{n})(\A,\B,\bGamma)   &=  \begin{bmatrix} 
				\dot{\h}(y_{1},\a_{1}),\dots,
				\dot{\h}(y_{n},\a_{n})
			\end{bmatrix}  \\
			\nabla_{\vect(\bGamma)} (\mathcal{L}_{n} - \bar{\mathcal{L}}_{n})(\A,\B,\bGamma) &=\left( \sum_{s=1}^{n} \dot{\h}(y_{s},\a_{s}) \otimes \x_{s}' \right) \\
			\nabla_{\B}  (\mathcal{L}_{n} - \bar{\mathcal{L}}_{n})(\A,\B,\bGamma) &=\frac{2}{2\sigma^{2}} (\B^{\star}-\X_{\textup{data}}) = \frac{2}{2\sigma^{2}} [ \beps_{1},\dots,\beps_{n}] \\
			\nabla_{\blambda}  (\mathcal{L}_{n} - \bar{\mathcal{L}}_{n})(\A,\B,\bGamma)  &=\frac{2}{2\sigma^{2}} \sum_{s=1}^{n}\beps_{s}' .
		\end{align}
		Hence repeating the same argument as before, using concentration inequalities for the following random quantities 
		\begin{align}
			\mathtt{Q}_{1}:= \begin{bmatrix} 
				\dot{\h}(y_{1},\a_{1}),\dots,
				\dot{\h}(y_{n},\a_{n})
			\end{bmatrix}  \in \R^{p\times n} ,\quad \mathtt{Q}_{2}:=\sum_{s=1}^{n}  \beps_{s}\in \R^{p} ,\quad \mathtt{Q}_{3}:=\sum_{s=1}^{n}  \beps_{s}'\in \R^{q}, \quad \mathtt{Q}_{4}:= [ \beps_{1},\dots,\beps_{n}]  \in \R^{p\times n},
		\end{align}
		one can bound the size of $G(\bZ^{\star},\tau)$ with high probability. The rest of the details are omitted. 
	\end{proof}

	\section{Proof of Theorem \ref{thm:MLE_local_consistency}}
	
	In this section, we prove the non-asymptotic local consistency of constrained and regularized MLE, stated in Theorem \ref{thm:MLE_local_consistency}. We combine a classical approach in \cite{fan2001variable} with concentration inequalities, namely, a classical Berry-Esseen bound for deviations from standard normal distribution and  a uniform McDirmid bound (Lemma \ref{lem:uniform_McDirmid}). The former is used to control the linear term in the second-order Taylor expansion of the log-likelihood function, and the latter is used to control the second-order term. By using an $\eps$-net argument, the latter  concentration inequality can be extended to a setting where the random variables are parameterized within a compact set.

	\begin{lemma}[A uniform McDirmid's inequality]
		\label{lem:uniform_McDirmid}
		Let $(X_{n})_{n\ge 1}$ be i.i.d. random vectors in $\R^{d}$ from a distribution $\pi$. Fix a compact parameter space $\Param\subseteq \R^{p}$ and $f_{\param}:\R^{d}\rightarrow \R$ is a bounded functional for each $\param\in \Param$ such that 
		\begin{align}\label{eq:McDirmid_cond}
			\lVert f_{\param} - f_{\param'} \rVert_{\infty} \le L\lVert\param-\param' \rVert ,\qquad \forall \param,\param'\in \Param
		\end{align}
		for some constant $L>0$. Further assume that $\E_{X\sim \pi}[f_{\param}(X)]=0$ for all $\param\in \Param$. Then there exists constants $C,M>0$ such that for each $n \ge 0$, and $\eta>0$, 
		\begin{align}
			\P\left(  \sup_{\param\in \Param} \left| \frac{1}{n}\sum_{k=1}^{n} f_{\param}(X_{k}) \right| \ge \eta \right) \le C \exp\left( p\log (1/\eta) + \frac{-\eta^{2}n}{ 2M^{2} }\right),
		\end{align}
		where $C'=4R(\varphi(\Omega))^{2}\tau_{\textup{min}}$.
	\end{lemma}
	
	\begin{proof}	
		Since $\Param\subseteq \R^{p}$ is compact, it can be covered by a finite number of $L^{2}$-balls of any given radius $\eps>0$. Let $\mathcal{U}_{\eps}$ be such an open cover using the least number of balls of radius $\eps>0$. Let $N(\eps) = |\mathcal{U}_{\eps}|$ denote the least number of such balls to cover $\Param$. Moreover, let $\textup{diam}(\Param)$ denote the diameter of $\Param$, which is finite since $\Param$ is compact. Then $\Param$ is contained in a $p$-dimensional box of side length $\textup{diam}(\Param)$. It follows that there exists a constant $K>0$, depending only on $\textup{diam}(\Param)$ and $rd$, such that 
		\begin{align}\label{eq:bd_N(eps)}
			N(\eps)\le K\left( \frac{\textup{diam}(\Param)}{\eps} \right)^{p}.
		\end{align}
		
		Next, fix $\eta>0$, $\param\in \Param$, and $\eps>0$. Let $\param_{1},\cdots,\param_{N(\eps)}$ be the centers of balls in the open cover $\mathcal{U}_{\eps}$. Then there exists $1\le j \le N(\eps)$ such that $\lVert \param - \param_{j}\rVert <\eps$. By the hypothesis, $f_{\param}$ depends on $\param$ uniformly continuously with respect to the supremum norm. Hence there exists $\delta=\delta(\eps)>0$ such that 
		\begin{align}
			\lVert f_{\param} - f_{\param_{j}} \rVert_{\infty} \le L\eps.
		\end{align} 
		Denote $H_{n}(\param):= n^{-1}\sum_{k=1}^{n} f_{\param}(X_{k})$.  Then this yields, almost surely, 
		\begin{align}
			| H_{n}(\param) - H_{n}(\param_{j}) | \le L\eps.
		\end{align} 
		Furthermore, since $\Param$ is compact and each $f_{\param}$ is bounded, \eqref{eq:McDirmid_cond} implies that $\lVert f_{\param} \rVert_{\infty}$ is uniformly bounded in $\param$ by some constant, say, $M>0$. It follows that for each $\param\in \Param$, $H_{n}(\param)$ changes its value at most by $M$ when one of $X_{1},\dots,X_{n}$ is replaced arbitrarily. Therefore by the standard McDirmid's inequality (see, e.g., \cite[Thm. 2.9.1]{vershynin2018high}) and a union bound, with choosing $ \eps=\eta/(2L)$, 
		\begin{align}
			\P\left(  |H_{n}(\param)| \ge \eta \right) &\le \sum_{j=1}^{N(\eta/2L)} \P\left(  \left| H_{n}(\param_{j}) \right| \ge \eta/2 \right) 	\le K\left( \frac{2L \, \textup{diam}(\Param)}{\eta} \right)^{p}  \exp\left( - \frac{n \eta^{2}}{2M^{2}} \right).
		\end{align}
		The above holds for all $n\ge 1$ and $\eta>0$. This shows the assertion.
	\end{proof}

	Now we prove Theorem \ref{thm:MLE_local_consistency}.

	\begin{proof}[\textbf{Proof of Theorem \ref{thm:MLE_local_consistency}}]
		Let $X_{1},\dots,X_{n}$ denote i.i.d. samples from $\pi_{\param^{\star}}$ and write $\X=[X_{1},\dots,X_{n}]\in \R^{d\times n}$. Recall that 
		\begin{align}
			\mathcal{L}(\X;\param) := \mathcal{L}_{0}(\X;\param) + n R_{n}(\param), \qquad \mathcal{L}_{0}(\X;\param) := -\sum_{i=1}^{n} \log \pi_{\param}(X_{i}),
		\end{align}
		where $R(\cdot)$ is the regularizer used in \eqref{eq:def_MLE_problem}. By the hypothesis, $\mathcal{L}_{0}$ is twice continuously differentiable, so $\E[\nabla \mathcal{L}_{0}]=\nabla \E[\mathcal{L}_{0}]$ and $\E[\nabla^{2} \mathcal{L}_{0}]=\nabla^{2} \E[\mathcal{L}_{0}]$ by the dominated convergence theorem.

		Fix a constant $C>0$ and let $\alpha_{n}:=n^{-1/2}+\lVert \nabla R_{n}(\param^{\star}) \rVert$. We wish to show the probability bound in \eqref{eq:local_consistency_thm}. We first introduce two random variables that we will bound to be small by using some concentration inequalities: 
		\begin{align}\label{eq:def_T_S_pf_consistency}
			T_{n}(\param) &:=\E\left[ \left\langle \nabla_{\param} \mathcal{L}_{0}(X_{i};\param^{\star}),\, \frac{\param-\param^{\star}}{\lVert\param - \param^{\star} \rVert} \right\rangle \right]- \frac{1}{\sqrt{n}}\sum_{i=1}^{n}  \left\langle \nabla_{\param} \mathcal{L}_{0}(X_{i};\param^{\star}),\, \frac{\param-\param^{\star}}{\lVert\param - \param^{\star} \rVert} \right\rangle \\
				&=  \left\langle  \frac{1}{\sqrt{n}}\sum_{i=1}^{n} \left(  \nabla_{\param} \mathcal{L}_{0}(X_{i};\param^{\star}) - \E\left[  \nabla_{\param} \mathcal{L}_{0}(X_{i};\param^{\star}) \right] \right) ,\, \frac{\param-\param^{\star}}{\lVert\param - \param^{\star} \rVert} \right\rangle,  \\
			S_{n}(\param) &:= 
			\frac{(\param-\param_{0})^{T}}{\lVert \param-\param^{\star} \rVert} \left( \frac{1}{n}\nabla_{\param}\nabla_{\param^{T}} 
			\mathcal{L}(\X \,;\,  \param^{\star})  - \nabla_{\param}\nabla_{\param^{T}} 
			\left( \E\left[ \mathcal{L}(X_{1} \,;\,  \param^{\star})\right]  \right) \right) \frac{\param-\param_{0}}{\lVert \param-\param_{0}\rVert}.
		\end{align}
		Fix $\param\in \Param$ such that $\lVert \param-\param^{\star} \rVert=C\alpha_{n}$. Since $\param\mapsto \mathcal{L}(\X\,;\, \param)$  is assumed to be three-times continuously differentiable, the quantity $M\ge 0$ in the assertion is well-defined and is finite. Then using a Taylor expansion, we  may write 
		\begin{align}\label{eq:thm_local_consistency_taylor}
			\mathcal{L}(\X \, ;\,  \param)  -  \mathcal{L}(\X \, ;\,  \param^{\star})   \ge \langle \nabla_{\param} \mathcal{L}(\X;\param^{\star}),\, \param-\param^{\star} \rangle   + \frac{1}{2} (\param-\param_{0})^{T} \nabla_{\param}\nabla_{\param^{T}} \mathcal{L}(\X \,;\,  \param^{\star})   (\param-\param_{0}) -  \frac{M}{6} (C\alpha_{n})^{3}.
		\end{align}
		We will lower bound the first two terms in the right hand side above. Note that 
		\begin{align}
			\langle \nabla_{\param} \mathcal{L}(\X;\param^{\star}),\, \param-\param^{\star} \rangle &= \left[  \langle \nabla_{\param} \mathcal{L}(\X;\param^{\star}),\, \param-\param^{\star} \rangle- \E\left[ \langle \nabla_{\param} \mathcal{L}(\X;\param^{\star}),\, \param-\param^{\star} \rangle \right] \right] \\
			&\hspace{3cm} +n \langle \nabla_{\param} \E[\mathcal{L}_{0}(X_{1};\param^{\star})] ,\, \param-\param^{\star} \rangle  + n\langle \nabla R_{n}(\param^{\star}),\, \param-\param^{\star} \rangle \\
			&\overset{(a)}{\ge}  \langle \nabla_{\param} \mathcal{L}_{0}(\X;\param^{\star}),\, \param-\param^{\star} \rangle- \E\left[ \langle \nabla_{\param} \mathcal{L}_{0}(\X;\param^{\star}),\, \param-\param^{\star} \rangle \right] - n \lVert \nabla R_{n}(\param^{\star}) \rVert \, \lVert \param-\param^{\star} \rVert\\
			&= -\sqrt{n} \lVert \param-\param^{\star} \rVert T_{n}(\param) - n \lVert \nabla R_{n}(\param^{\star}) \rVert \, \lVert \param-\param^{\star} \rVert \\
			&\overset{(b)}{\ge} -C \sqrt{n} \alpha_{n} T_{n}(\param) - Cn\alpha_{n}^{2},
		\end{align}
		where for (a) we use the fact that $\param^{\star}$ is a stationary point of $\E[\mathcal{L}_{0}(X_{1};\param)]$ over $\Param$ and Cauchy-Schwarz inequality; 
		for (b) we used that $\max(n^{-1/2},\lVert \nabla R_{n}(\param^{\star})  \rVert) \le \alpha_{n}$.

		Next, we turn our attention to the second order term in the Taylor expansion \eqref{eq:thm_local_consistency_taylor}. 
		Under assuming $\lVert \param-\param^{\star} \rVert=C\alpha_{n}$, note that 
		\begin{align}
			(\param-\param^{\star})^{T} \nabla_{\param}\nabla_{\param^{T}} \mathcal{L}(\X \,;\,  \param^{\star})   (\param-\param^{\star}) &= C^{2}n \alpha^{2} \frac{(\param-\param^{\star})^{T}}{\lVert \param-\param^{\star} \rVert} \left( \frac{1}{n}\nabla_{\param}\nabla_{\param^{T}} 
			\mathcal{L}(\X \,;\,  \param^{\star})   \right) \frac{\param-\param^{\star}}{\lVert \param-\param^{\star}\rVert} \\
			&\ge C^{2}n\alpha_{n}^{2}\left( S_{n}(\param) + \lambda_{*}\right),
		\end{align}
		where the inequality follows from the hypothesis, which implies
		\begin{align}
			\E\left[ \frac{1}{n}\nabla_{\param}\nabla_{\param^{T}} 
			\mathcal{L}(\X \,;\,  \param^{\star}) \right] = \nabla_{\param}\nabla_{\param^{T}} 
			\left( \E\left[ \mathcal{L}(X_{1} \,;\,  \param^{\star})\right]  \right) \succeq \lambda_{*} \I_{p},
		\end{align}
		where $\lambda_{*}>0$ is a constant. Combining the above inequalities, we obtain  
		\begin{align}\label{eq:local_consistency_pf_1}
			&\hspace{-0.5cm} \inf_{\substack{\param\in \Param \\ \lVert \param-\param^{\star} \rVert =C\alpha_{n}} }  	\frac{1}{C^{2}n\alpha_{n}^{2}}\left( \mathcal{L}(\X \, ;\,  \param)  -  \mathcal{L}(\X \, ;\,  \param^{\star})  \right) \\
			&\qquad \ge \underbrace{\frac{1}{C}\left(  \inf_{\substack{\param\in \Param \\ \lVert \param-\param^{\star} \rVert =C\alpha_{n}} }   \frac{-T_{n}(\param)}{\sqrt{n} \alpha_{n}}  - 1 \right)}_{=:A} + \underbrace{\left(  \inf_{\substack{\param\in \Param \\ \lVert \param-\param^{\star} \rVert =C\alpha_{n}} }  S_{n}(\param) + \lambda_{*} \right)}_{=:B} - \frac{M C\alpha_{n}}{6n}. 
		\end{align}
		
		According to the hypothesis, we have 
		\begin{align}
			(MC/6)(n^{-3/2} + n^{-1} \rVert \nabla R_{n}(\param^{\star}) \rVert)  =\frac{M C\alpha_{n}}{6n} \le \lambda_{*}/4.
		\end{align}
		Then the last expression in \eqref{eq:local_consistency_pf_1} is at least $\lambda_{*}/4$ if $A\ge -\lambda_{*}/4$ and $B\ge 3\lambda_{*}/4$. Moreover, Note that for any two events $E_{1},E_{2}$ defined on the same probability space, $\P(E_{1}\cap E_{2})\ge \P(E_{1})+\P(E_{2})-1$. Notice that $A$ and $B$ above are random variables defined on the same probability space, which are deterministic functions of the common random vectors $X_{1},\dots,X_{n}$. Thus 
		\begin{align}\label{eq:local_consistency_pf_11}
			\P\left( \inf_{\substack{\param\in \Param \\ \lVert \param-\param^{\star} \rVert =C\alpha_{n}} }  	\frac{1}{C^{2}n\alpha_{n}^{2}}\left( \mathcal{L}(\X \, ;\,  \param)  -  \mathcal{L}(\X \, ;\,  \param^{\star}) \right)  \ge \lambda_{*}/4  \right)  \ge \P(A\ge -\lambda_{*}/4) + \P( B\ge 3\lambda_{*}/4) -1.
		\end{align}
		Note that by the uniform McDirmid's inequality in Lemma \ref{lem:uniform_McDirmid}, there exists constants $C',C''>0$ such that 
		\begin{align}\label{eq:local_consistency_pf_111}
			\P( B <  3\lambda_{*}/4) &\le \P\left(   \inf_{\substack{\param\in \Param \\ \lVert \param-\param^{\star} \rVert =C\alpha_{n}} }  S_{n}(\param) < -\lambda_{*}/4  \right) \le  C'\exp(-C'' n).
		\end{align}
		On the other hand, we will show the following inequalities: 
		\begin{align}
			\P(A< -\lambda_{*}/4) \overset{(c)}{\le}  \P\left(  \sup_{\substack{\param\in \Param \\ \lVert \param-\param^{\star} \rVert =C\alpha_{n}} }  T_{n}(\param)    > \frac{C\lambda_{*}}{4}-1\right) &\overset{(d)}{\le} M \left( \P\left(Z \ge  p^{-1/2}\left(\frac{C\lambda_{*}}{4}-1\right) \right) +   \frac{K}{\sqrt{n}}\right) \\
			&\overset{(e)}{\le} M \left( \exp\left(-\left( \frac{C\lambda_{*}}{4}-1 \right)^{2}\bigg/2\sqrt{p} \right)  +   \frac{K}{\sqrt{n}}\right), \label{eq:local_consistency_pf_2}
		\end{align}
		where $Z\sim N(0,1)$ is an independent standard normal random variable and  $M, K$ are constants that does not depend on $n$ and $\eta$. Then the assertion will follow by combining  \eqref{eq:local_consistency_pf_1},  \eqref{eq:local_consistency_pf_11},  \eqref{eq:local_consistency_pf_111}, and  \eqref{eq:local_consistency_pf_2}. Note that (c) in \eqref{eq:local_consistency_pf_2} follows from the definition of $A$ in \eqref{eq:local_consistency_pf_1} and the fact that $\sqrt{n}\alpha_{n} = 1+n^{-1/2} \lVert \nabla R_{n}(\param^{\star}) \rVert \ge 1$. Also note that (e) above is a simple consequence of Hoeffding's bound for the standard normal tail.

		It remains to verify (d) in \eqref{eq:local_consistency_pf_2}. To this end, we first wright $R_{n}:= \frac{1}{\sqrt{n}}\sum_{i=1}^{n} \left(  \nabla_{\param} \mathcal{L}_{0}(X_{i};\param^{\star}) - \E\left[  \nabla_{\param} \mathcal{L}_{0}(X_{i};\param^{\star}) \right] \right)  = [R_{n}^{(1)},\dots, R_{n}^{(p)}]\in \R^{p}$. By Cauchy-Schwarz inequality, 
		\begin{align}
				| T_{n}(\param)  |^{2} \le \lVert R_{n} \rVert^{2}=  \sum_{i=1}^{p}  \left( R_{n}^{(i)} \right)^{2},
		\end{align}
		where the upper bound  $R_{n}$ does not depend on $\param$. Hence by union bound, 
		\begin{align}\label{eq:pf_sup_T_union_bd}
			\P\left(   \sup_{\substack{\param\in \Param \\ \lVert \param-\param^{\star} \rVert =C\alpha_{n}} }  T_{n}(\param)   \ge t  \right) & \le \P\left(  \lVert R_{n} \rVert \ge t  \right) \le \sum_{i=1}^{p} \P\left( |R_{n}^{(i)}|  \ge  \frac{t}{\sqrt{p}} \right). 
		\end{align}
		Note that if $R^{(i)}_{n}=0$ a.s., then the coordsponding tail probability in the last expression above is zero whenever $t\ne 0$. Since $\E[ R^{(i)}_{n} ]=0$, this is the case if $\var( R^{(i)}_{n}  )=0$. So we may assume without loss of generality that $\var( R^{(i)}_{n}  )>0$ for $1\le i \le p$. Then, by the definition of $R_{n}^{(i)}$ and the classical Berry-Esseen theorem (see, e.g., \cite[Thm. 3.4.17]{Durrett}), for all $z\in \R$,
		\begin{align}
			\sup_{z\in \R} \left| \P\left( R_{n}^{(i)} \le  z\right)  - \P\left( Z\le z  \right) \right| \le \frac{3 \, \E[ | R_{n}^{(i)} |^{3} ]  }{ \var (R_{n}^{(i)} )^{3/2} \sqrt{ n}}.
		\end{align} 
		 Note that $\E[ | R_{n}^{(i)} |^{3} ]  <\infty$ by the hypothesis and $\var (R_{n}^{(i)} ) >0$  by the assumption we just made above. Combining with \eqref{eq:pf_sup_T_union_bd}, we obtain 
		\begin{align}\label{eq:pf_sup_T_union_bd}
			\P\left(   \sup_{\substack{\param\in \Param \\ \lVert \param-\param^{\star} \rVert =C\alpha_{n}} }  T_{n}(\param)   \ge t  \right) & \le  2p\left(   \P( Z \ge t/\sqrt{p} ) +  \frac{K}{\sqrt{n}} \right),
		\end{align}
		where $K\in (0,\infty)$ is the maximum of $3 \, \E[ | R_{n}^{(i)} |^{3} ]  \bigg/  \E[ |R_{n}^{(i)}|^{2} ]^{3/2} $ for $i=1,\dots,p$. Thus (d) in  \eqref{eq:local_consistency_pf_2} follows. 
	\end{proof}

	\section{Generative model for strongly constrained feature-based SDL} 
	\label{section:statistical_estimation_BCD_feat}
	
	In Section \ref{subsection:statistical_estimation_BCD}, we have considered a generative model for the filter-based SDL with strong constraints and derived a local consistency result. In this section, we discuss a parallel model for a generative feature-based SDL with strong constraints. Unlike for the filter-based case, we would need to formulate a latent-variable model where the `code matrix' $\H=[h_{1},\dots,\h_{n}]$ is the latent variable, and our general theory of non-asymptotic local consistency of constrained and regularized MLE applies only approximately.

	Suppose that the data, auxiliariy covariate, and label triples $(\x_{i}, \x_{i}', y_{i})$ are drawn i.i.d. according to the following generative model:
	\begin{align}\label{eq:SDL_prob_BCD_feat}
		&\h_{i}\sim N(\h^{\star}, \sigma_{h}^{2}\I_{r}),\quad \x_{i} \sim  N\left( \W^{\star}\h_{i}, \sigma^{2}\I_{p} \right), \quad \x_{i}'\sim N(\blambda^{\star}, (\sigma')^{2}\I_{q}),\\
		& y_{i}\,|\, \x_{i}, \x_{i}' \sim \text{Multinomial}\big(1, \g\left( (\Beta^{\star})^{T} \h_{i}+ (\bGamma^{\star})^{T}\x_{i}'\right) \big), \\
		&\textup{where}\quad  \W^{\star}\in \R^{p\times r},\,  \h^{\star}\in \R^{r\times 1},\,  \Beta^{\star}\in  \R^{r\times \kappa},\,  \bGamma^{\star}\in \R^{q\times \kappa},\,\blambda^{\star}\in \R^{q\times 1} \,\, \text{s.t. $[\W^{\star},\h^{\star},\Beta^{\star},\bGamma^{\star}]\in \Param$}.
	\end{align}
	As before, $\Param= \mathcal{C}^{\textup{dict}}\times \mathcal{C}^{\textup{code}}\times  \mathcal{C}^{\textup{beta}} \times \mathcal{C}^{\textup{aux}}$ is the product of convex constraint sets on individual factors. We assume $(\x_{i},\x_{i}',y_{i})$ for $i=1,\dots,n$ are i.i.d. observed data and $\h_{i}$ for $i=1,\dots,n$ are i.i.d. latent variables that we do not observe. However, we assume that the mean latent variable $\h^{\star}$ is known\footnote{In \cite{mairal2008supervised}, a similar model was considered with $\h^{\star}=\mathbf{0}$.}. We also assume $\x_{i}$ and $\x_{i}'$ are independent for each $1\le i \le n$. We call the above the \textit{generative feature-based SDL model}. Assuming that $\sigma_{n},\sigma,$ and $\sigma'$ are known, our goal is to estimate the true factors $\W^{\star}$, $\Beta^{\star}, \bGamma^{\star}$, and $\blambda^{\star}$ from an observed sample $(\x_{i},\x'_{i},y_{i})$, $i=1,\dots,n$ of size $n$. 
	
	We consider the maximum likelihood estimation framework with $L_{2}$-regularization of the parameters. Namely, define
	\begin{align}\label{eq:SDL_likelihood_BCD_feature_cond}
		\mathcal{\ell}_{n}(\W,\H,\Beta, \bGamma)&:=-\sum_{j=0}^{\kappa}\mathbf{1}(y_{i}=j)g_{j}(\Beta^{T}\h_{i} +  \bGamma^{T} \x_{i}' ) )   +  \frac{1}{2\sigma^{2}} \lVert \X_{\textup{data}} - \W \H \rVert_{F}^{2},
	\end{align}
	which is the negative log likelihood of observing the samples conditional on the hidden variable $\H=[h_{1},\dots,\h_{n}]$ and also on the auxiliary covarite $\x_{1}',\dots,\x_{n}'$. Integrating out the hidden variable $\H$, we obtain the negative log likelihood of observing the data conditional on the auxiliary covariate: 
	\begin{align}\label{eq:feat_SDL_approx_MAP0}
		\mathcal{L}_{n}(\W,\Beta, \bGamma) = -\log \left(  \int_{\R^{r\times n}} \exp\left( -\mathcal{\ell}_{n}(\W,\H,\Beta, \bGamma)\right) \exp\left( -\frac{1}{(2\sigma_{h})^{r}} \left( \sum_{i=1}^{n}\lVert \h_{i}-\h^{\star} \rVert^{2}\right)   \right) \,d \H \right).
	\end{align}
	Note that when the variance of the hidden variable $\sigma_{h}^{2}$ is assumed to be small, the leading contribution to the integral over $\H$ above comes when $\H$ is set to the maximizer of the integrand. Hence in this case, the above can be approximate as
	\begin{align}\label{eq:feat_SDL_approx_MAP1}
		\mathcal{L}_{n}(\W,\Beta, \bGamma) \propto \min_{\H} \,\,  \ell_{n}(\W,\H,\Beta,\bGamma) +\frac{1}{(2\sigma_{h})^{r}} \left( \sum_{i=1}^{n}\lVert \h_{i}-\h^{\star} \rVert^{2}\right) .
	\end{align}
	Consequently, we may estimate the unknown true parameters $\W^{\star}, \Beta^{\star},$ and $\bGamma^{\star}$ as the minimizer of the following loss function 
	\begin{align}\label{eq:SDL_likelihood_BCD_filter_full}
		L(\W, \H, \Beta, \bGamma,\blambda) &  :=  \ell_{n}(\W, \H,\Beta, \bGamma) +\frac{1}{(2\sigma_{h})^{r}} \left( \sum_{i=1}^{n}\lVert \h_{i}-\h^{\star} \rVert^{2}\right) + \frac{pn\log \sigma}{2} +   \frac{qn\log \sigma'}{2} +  \frac{1}{(2(\sigma')^{2})^{q}} \sum_{i=1}^{n} \lVert \x_{i}'-\blambda \rVert^{2},
	\end{align}
	where we also simultaneously estimate the hidden variable $\H$. This is analogous to the `MAP approximation' approach and `generative training' employed in \cite{mairal2008supervised}. 
	
	Note that the first two terms on the right-hand side above are equivalent to the feature-based SDL loss in \eqref{eq:ASDL_1} with an additional quadratic regularizer on $\H$. Hence we can compute its derivate and Hessian using a similar computation as for the filter-based model in Subsection \ref{subsection:statistical_estimation_BCD}. By adding a suitable $L_{2}$-regularizer for $\W,\Beta$, and $\bGamma$ (also possibly for $\H$ with different regularization coefficient), we can obtain a similar non-asymptotic local consistency result as we established in Theorem \ref{thm:SDL_BCD_STAT_filt} for the generative filter-based SDL. However, in order to obtain a precise error bound as in Theorem \ref{thm:SDL_BCD_STAT_filt} for the generative feature-based SDL we consider here, we need to account for the `MAP approximation' of the likelihood function in \eqref{eq:feat_SDL_approx_MAP0}, where we replaced integrating over $\H$ by maximizing over $\H$ assuming the variance of $\h_{i}$'s are small. It should be straightforward to estimate this approximation error and relate the Hessian of $\mathcal{L}_{n}(\W,\Beta,\bGamma)$ in \eqref{eq:feat_SDL_approx_MAP1} with the Hessian of $\ell_{n}(\W,\H,\Beta,\bGamma)$. Then we can apply Theorem \ref{thm:MLE_local_consistency} with an explicit lower bound on the added $L_{2}$-regularization coefficients to make the Fisher information positive definite. We omit the details of this sketch in this work.

	\section{Auxiliary lemmas from optimization}

	\begin{lemma}\label{lem:gradient_mapping}
		Fix a differentiable function $f:\R^{p}\times \R$ and a convex set $\Param\subseteq \R^{p}$. Fix $\tau>0$ and \begin{align}
			G(\bZ, \tau):= \frac{1}{\tau}(\bZ - \Pi_{\Param}(\param - \tau\nabla f(\param))). 
		\end{align} 
		Then for each $\param\in \Param$, $\lVert G(\param,\tau) \rVert\le \lVert \nabla f(\param) \rVert$. 
	\end{lemma}
	
	\begin{proof}
		The assertion is clear if $\lVert G(\param,\tau) \rVert=0$, so we may assume $\lVert G(\param,\tau) \rVert>0$. 
		Denote $\hat{\param}:=\Pi_{\Param}(\param- \tau\nabla f(\param)))$. Note that 
		\begin{align}
			\hat{\param}= \argmin_{\param'} \, \lVert \param - \tau\nabla f(\param) - \param' \rVert^{2},
		\end{align}
		so by the first-order optimality condition, 
		\begin{align}
			\langle \hat{\param} - \param + \tau\nabla f(\param) ,\, \param' -\hat{\param}  \rangle \ge 0 \quad \forall \param'\in \Param. 
		\end{align}
		Plugging in $\param'=\param$ and using Cauchy-Schwarz inequality, 
		\begin{align}
			\tau^{2} \lVert G(\param,\tau) \rVert^{2} =  \lVert \param - \hat{\param} \rVert^{2} \le \tau \langle \nabla f(\param),\, \param-\hat{\param} \rangle \le \tau \lVert \nabla f(\param) \rVert \, \tau \lVert G(\param, \tau)  \rVert.
		\end{align}
		Hence the assertion follows by dividing both sides by $\tau^{2}\lVert G(\param,\tau) \rVert>0$. 
	\end{proof}

	\section{Generalized multinomial logistic Regression}\label{sec:MLR}
	In this section, we provide some background on a generalized multinomial logistic regression and record some useful computations. (See \cite{bohning1992multinomial} for backgrounds on multinomial logistic regression.)
	Without loss of generality, we can assume that the $\kappa$ classes are the integers in $\{1,2,\dots, \kappa \}$. Say we have training examples $(\bphi(\x_{1}),y_{1}),\dots,(\bphi(\x_{N}),y_{N})$, where 
	\begin{enumerate}[label={\textbf{$\bullet$}}]
		\item $\x_{1},\dots,\x_{N}$: Input data (e.g., collection of all medical records of each patient) 		
		\item $\bphi_{i}:=\bphi(\x_{1}),\dots,\bphi_{N}:=\bphi(\x_{N})\in \R^{p}:$ Features (e.g., some useful (derived) information for each patient)
		\item $y_{1},\dots,y_{n}\in \{0, 1,\dots, \kappa \}$: $\kappa$ class labels (e.g., digits from 0 to 9). 
	\end{enumerate}
	The basic idea of multinomial logistic regression is to model the output $y$ as a discrete random variable  $Y$ with probability mass function  $\p=[p_{0}, p_{1},\dots,p_{\kappa}]$ that depends on the observed feature $\bphi(\x)$, link function $h:\R \rightarrow \R$, and a parameter $\W=[\w_{1},\dots,\w_{\kappa}]\in \R^{p\times \kappa}$ through the following relation: 
	\begin{align}\label{eq:MLR}
		p_{0}= \frac{1}{1+\sum_{c=1}^{\kappa} h(\langle \bphi(\x), \w_{c} \rangle)},\qquad 	p_{j}=\frac{h(\langle \bphi(\x), \w_{i} \rangle)}{1+\sum_{c=1}^{\kappa} h(\langle \bphi(\x), \w_{c} \rangle)},\qquad \text{for $j=1,\dots, \kappa$}.
	\end{align}
	That is, given the feature vector $\bphi(\x)$, the probability $p_{i}$ of $\x$ having label $i$ is proportional to $h$ evaluated at the `linear activation' $\langle \bphi(\x), \w_{i} \rangle$. Note that using $h(x)=\exp(x)$, the above multiclass classification model reduces to the classical multinomial logistic regression. In this case, the corresponding predictive probability distribution $\p$ is called  the \textit{softmax distribution} with activation $\mathbf{a}=[a_{1},\dots,a_{\kappa}]$ with $a_{i}=\langle \bphi(\x), \w_{i} \rangle$ for $i=1,\dots,\kappa$. Notice that this model has parameter vectors $\w_{1},\dots,\w_{\kappa}\in \R^{p}$, one for each of the $\kappa$ nonzero class labels.

	Next, we derive the maximum log likelihood formulation for finding optimal parameter $\W$ for the given training set $(\bphi_{i},y_{i})_{i=1,\dots,N}$. For each $1\le i \le N$ and $1\le j \le \kappa$, denote $p_{ij} := h(\langle \bphi_{i}, \w_{j} \rangle) \big/ \sum_{c=1}^{\kappa} h(\langle \bphi_{i}, \w_{c} \rangle)$, the predictive probability of the $y_{i}$ given $\bphi_{i}$ being $j$. We introduce the following matrix notations 
	\begin{align}
		\begin{matrix}
			\Y:=
			\begin{bmatrix}
				\mathbf{1}(y_{1}=1) & \cdots & \mathbf{1}(y_{1}=\kappa) \\ \vdots & & \vdots  \\ \mathbf{1}(y_{N}=1)  & \cdots & \mathbf{1}(y_{N}=\kappa)
			\end{bmatrix} 
			\\ \\ 
			\quad \in \{0,1\}^{N\times \kappa}
		\end{matrix}
		,\,\, 
		\begin{matrix}
			\bP:=
			\begin{bmatrix}
				p_{11} & \cdots & p_{1\kappa} \\ \vdots & & \vdots  \\ p_{N 1}  & \cdots & p_{N \kappa}
			\end{bmatrix} 
			\\ \\ 
			\quad \in \{0,1\}^{N\times \kappa}
		\end{matrix}
		,\,\, 
		\begin{matrix}
			\bPhi:= 	
			\begin{bmatrix}
				\uparrow & & \uparrow \\
				\bphi(\x_{1})  & \cdots &  \bphi(\x_{N}) \\
				\downarrow &  &\downarrow
			\end{bmatrix}
			\\ \\
			\quad \in \R^{p\times N}
		\end{matrix}
		,\,\, 
		\begin{matrix}
			\W := 
			\begin{bmatrix}
				\uparrow & & \uparrow \\
				\w_{1}  & \cdots &  \w_{\kappa}\\
				\downarrow &  &\downarrow
			\end{bmatrix}
			.
			\\ \\
			\quad \in \R^{p\times \kappa}
		\end{matrix}
	\end{align}
	Note that the $s$th row of $\Y$ is a one-hot encoding of the lable $y_{s}$ and the corresponding row of $\Q$ is its predictive probability distribution. Then the joint likelihood function of observing labels $(y_{1},\dots,y_{N})$ given input data $(\x_{1},\dots,\x_{N})$ under the above probabilistic model is 
	\begin{align}
		L(y_{1},\dots,y_{N}\,;\, \W) = \P(Y_{1}=y_{1},\dots,Y_{N}=y_{N}\, ;\, \W)=  \prod_{s=1}^{N} \prod_{j=1}^{\kappa}  (p_{sj})^{\mathbf{1}(y_{s}=j)}.
	\end{align}
	We can derive the negative log likelihood function $\ell(\bPhi, \W)
	:= - \sum_{s=1}^{N} \sum_{j=1}^{\kappa} \mathbf{1}(y_{s}=j) \log p_{sj}$ in a matrix form as follows: 
	\begin{align}
		\ell(\bPhi, \W)
		& = \sum_{s=1}^{N}  \log \left( \sum_{c=1}^{\kappa} h(\langle \bphi(\x_{s}), \w_{c} \rangle ) \right) - \sum_{s=1}^{N}  \sum_{j=1}^{\kappa} \mathbf{1}(y_{s}=j) \log h\left(  \langle \bphi(\x_{s}), \w_{j}\rangle \right)  \\
		&= \left( \sum_{s=1}^{N} \log \left( \sum_{q=1}^{\kappa} h(\langle \bphi(\x_{s}), \w_{q} \rangle ) \right) \right)- \tr\left( \Y^{T} h(\bPhi^{T} \W) \right).
	\end{align}
	Then the maximum likelihood estimate $\hat{\W}$ is defined as the minimizer of the above loss function in $\W$ while fixing the feature matrix $\bPhi$. 
	
	Both the maps $\W\mapsto \ell(\bPhi, \W)$ and $\bPhi\mapsto \ell(\bPhi, \W)$ are convex and we can compute their gradients as well as the Hessian explicitly as follows. For each $y\in \{0,1,\dots\kappa\}$, $\bphi\in \R^{p}$, and $\W\in \R^{p\times \kappa}$, define vector and matrix functions 
	\begin{align}\label{eq:MNL_h_def_appendix}
		\dot{\h}(y,\bphi, \W)& :=(\dot{h}_{1},\dots, \dot{h}_{\kappa})^{T}\in \R^{\kappa\times 1}, \qquad \dot{h}_{j} :=\left( \frac{h'(\langle \bphi,\w_{j} \rangle)}{1+\sum_{c=1}^{\kappa} h(\langle \bphi,\w_{c} \rangle)} - \mathbf{1}(y=j)\frac{h'(\langle \bphi,\w_{j} \rangle)}{h(\langle \bphi,\w_{j} \rangle)} \right)\\
		\ddot{\H}(y,\bphi,\W) & := \left( \ddot{\H}_{ij}  \right)_{i,j} \in \R^{\kappa\times \kappa}, \qquad \ddot{\H}_{ij} = \left( \begin{matrix} \hspace{-2cm} \frac{ h''(\langle \bphi,\w_{j} \rangle) \mathbf{1}(i=j)  }{  1+\sum_{c=1}^{\kappa} h(\langle \bphi,\w_{c} \rangle)  } - \frac{ h'(\langle \bphi,\w_{i} \rangle) h'(\langle \bphi,\w_{j} \rangle)   }{  \left( 1+\sum_{c=1}^{\kappa} h(\langle \bphi,\w_{c} \rangle) \right)^{2} }  \\
			\hspace{2cm} - \mathbf{1}(y=i=j) \left( \frac{h''(\langle \bphi,\w_{j} \rangle)}{h(\langle \bphi,\w_{j} \rangle)} - \frac{\left( h'(\langle \bphi,\w_{j} \rangle) \right)^{2} }{\left( h(\langle \bphi,\w_{j} \rangle) \right)^{2}}   \right) \end{matrix}  \right).
	\end{align}
	For each $\W=[\w_{1},\dots,\w_{\kappa}]\in \R^{p\times \kappa}$, let $\W^{\textup{vec}}:=[\w_{1}^{T},\dots, \w_{\kappa}^{T}]^{T}\in \R^{p\kappa}$ denote its vectorization.
	Then a straightforward computation shows 
	\begin{align}
		\nabla_{\vect(\W)}   \ell(\bPhi, \W) = \sum_{s=1}^{N} \dot{\h}(y_{s},\bphi_{i}, \W) \otimes \bphi_{s},\qquad 	\H:= \nabla_{\vect(\W)}  \nabla_{\vect(\W)^{T}}   \ell(\bPhi, \W) = \sum_{s=1}^{N} \ddot{\H}(y_{s},\bphi_{s}, \W) \otimes \bphi_{s}\bphi_{s}^{T},
	\end{align}
	where $\otimes$ above denotes the Kronecker product. 	Recall that the eigenvalues of $\A\times \B$, where $\A$ and $\B$ are two square matrices, are given by $\lambda_{i}\mu_{j}$, where $\lambda_{i}$ and $\mu_{j}$ run over all eigenvalues of $\A$ and $\B$, respectively. Hence we can deduce 
	\begin{align}\label{eq:MNL_evals_bounds}	
		\lambda_{\min}\left( \bPhi \bPhi^{T} \right) 	\min_{1\le s \le N,\, \W} \lambda_{\min}\left( \ddot{\H}(y_{s},\bphi_{s}, \W) \right) 	&\le \lambda_{\min}(\H) \\
		&\le  \lambda_{\max}(\H)\le \lambda_{\max}\left( \bPhi \bPhi^{T} \right) 	\max_{1\le s \le N,\, \W} \lambda_{\min}\left( \ddot{\H}(y_{s},\bphi_{s}, \W) \right).
	\end{align}

	There are some particular cases worth noting. First, suppose binary classification case, $\kappa=1$. Then the Hessian $\H$ above reduces to 
	\begin{align}
		\H = \sum_{s=1}^{N} \ddot{\H}_{11}(y_{s}, \bphi_{s},\W)  \bphi_{s}\bphi_{s}^{T}.
	\end{align}
	Second, let $h(x)=\exp(x)$ and consider the multinomial logistic regression case. Then $h=h'=h''$ so the above yields the following concise matrix expression
	\begin{align}
		&\nabla_{\W} \, \ell(\bPhi, \W) = \bPhi (\bP - \Y) \in \R^{p\times \kappa},\qquad \nabla_{\bPhi} \, \ell(\bPhi, \W) = \W (\bP - \Y)^{T} \in \R^{p\times N}, \\
		& \H= \sum_{s=1}^{N} \begin{bmatrix}
			p_{s1}(1-p_{s1}) & - p_{s1}p_{s2} & \dots & -p_{s1}p_{s\kappa} \\
			-p_{s2}p_{s1} & p_{s2}(1-p_{s2}) & \dots & -p_{s2}p_{s\kappa} \\
			\vdots & \vdots & \ddots & \vdots \\
			-p_{s\kappa}p_{s1} & -p_{s\kappa}p_{s2} & \dots & p_{s\kappa}(1-p_{s\kappa}) 
		\end{bmatrix}
		\otimes \bphi_{s}\bphi_{s}^{T}.
	\end{align}
	It follows that  eigenvalues of $\H$ are bounded above by $1/4$. The lower bound on the eigenvalues depend on the range of linear activation $\langle \bphi_{i}, \w_{j} \rangle$ may take. For instance, if we restrict the norms of the input feature vector $\phi_{i}$ and parameter $\w_{j}$, then we can find a suitable positive uniform lower bound on the eigenvalues of $\H$.

	\begin{lemma}\label{lem:MNL}
		Supose $h(\cdot)=\exp(\cdot)$. Then 
		\begin{align}
			&\lambda_{\min}\left( \ddot{\H}(\bphi_{s},\W) \right)\ge \min_{1\le i \le \kappa} \frac{ \exp(\langle \bphi_{s}, \w_{i} \rangle) }{1+\sum_{c=1}^{\kappa} \exp(\langle \bphi_{s}, \w_{c} \rangle) }, \\
			&\lambda_{\max}\left( \ddot{\H}(\bphi_{s},\W) \right)\le \max_{1\le i \le \kappa} \frac{ \exp(\langle \bphi_{s}, \w_{i} \rangle)  }{\left( 1+\sum_{c=1}^{\kappa} \exp(\langle \bphi_{s}, \w_{c} \rangle) \right)^{2} }\left(1+2\sum_{c=2}^{\kappa}  \exp(\langle \bphi_{s}, \w_{c} \rangle)  \right).
		\end{align}
	\end{lemma}
	
	\begin{proof}
		For the lower bound on the minimum eigenvalue, we note that
		\begin{align}
			\lambda_{\min}\left( \ddot{\H}(\bphi_{s},\W) \right) \ge \min_{1\le i\le \kappa}  \sum_{j=1}^{\kappa}  \ddot{H}_{ij} = \min_{1\le i \le \kappa} p_{si}p_{s0} =\min_{1\le i \le \kappa}  \frac{ \exp(\langle \bphi_{s}, \w_{i} \rangle) }{1+\sum_{c=1}^{\kappa} \exp(\langle \bphi_{s}, \w_{c} \rangle) }
		\end{align}
		where the first inequality was shown in  \cite[Sec. 3]{amani2021ucb} using the fact that $\ddot{\H}(\bphi_{s},\W)$ is a diagonally dominant $M$-matrix (see \cite{tian2010inequalities}). The following equalities can be verified easily. 
		
		For the upper bound on the maximum eigenvalue, we use the Gershgorin circle theorem (see, e.g., \cite{horn2012matrix}) to bound 
		\begin{align}
			\lambda_{\max}\left( \ddot{\H}(\bphi_{s},\W) \right) &\le \max_{1\le i \le \kappa} \left( p_{si}(1-p_{si}) + \sum_{c=2}^{\kappa} p_{si}p_{sc}\right) \le \max_{1\le i \le \kappa}  p_{si} \left(2-p_{s0}-2p_{si}\right).
		\end{align}
		Then simplifying the last expression gives the assertion.
	\end{proof}

	${}$
	\newpage 
	
	\section{Additioanl figures and tables}

	\begin{table}[h!]
		\centering
		\includegraphics[width=1\linewidth]{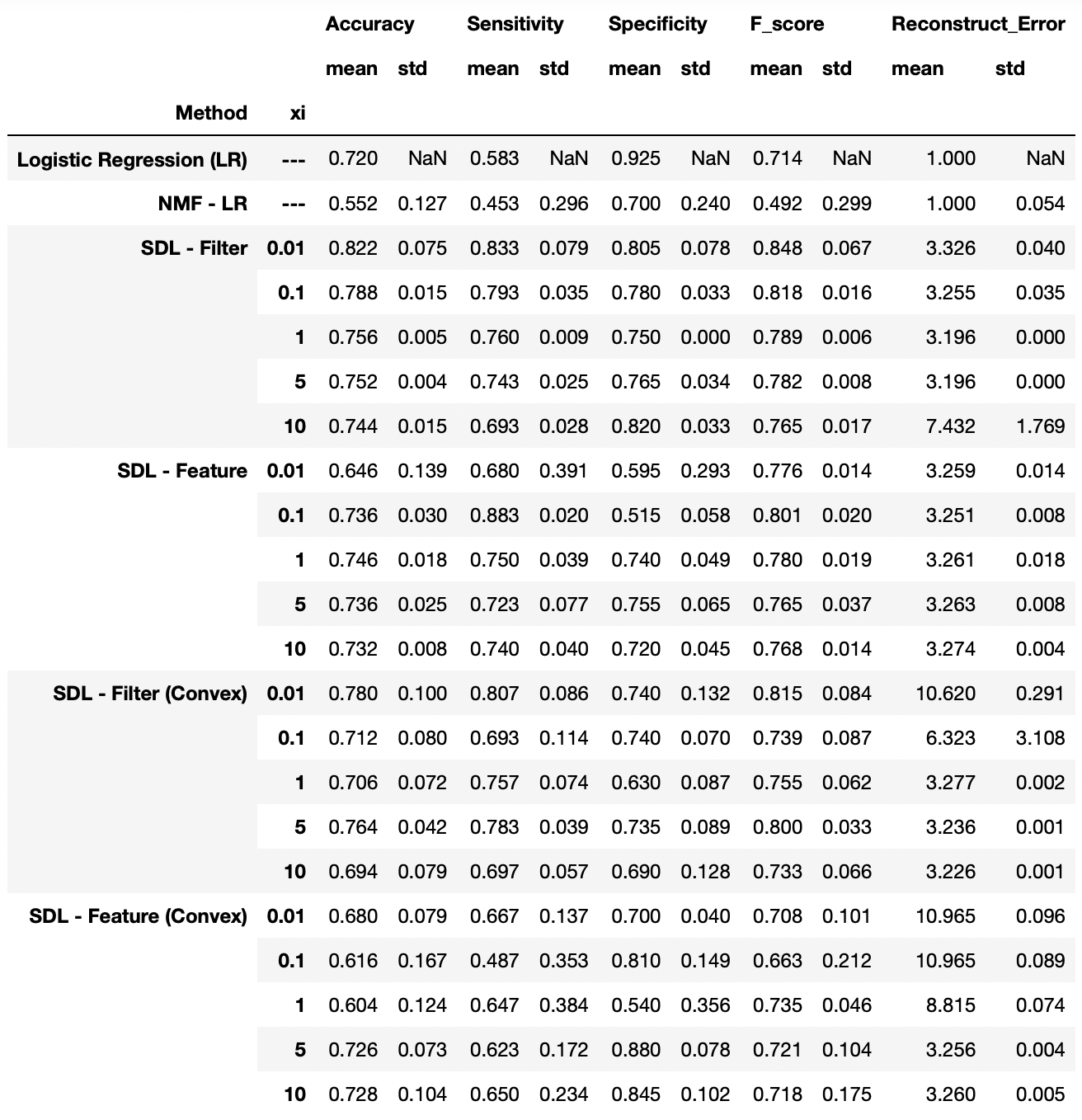} 
		\caption{ Tables of all results from each of the six methods for the semi-sythetic MNIST data in Section \ref{sec:MNIST_setup}. Mean and standard deviations of five runs are reported. We report the reconstruction error of each method after normalizing it by the reconstruction error from NMF. }
		\label{table:MNIST_table}
	\end{table}

	\begin{table}[h!]
		\centering
		\includegraphics[width=1\linewidth]{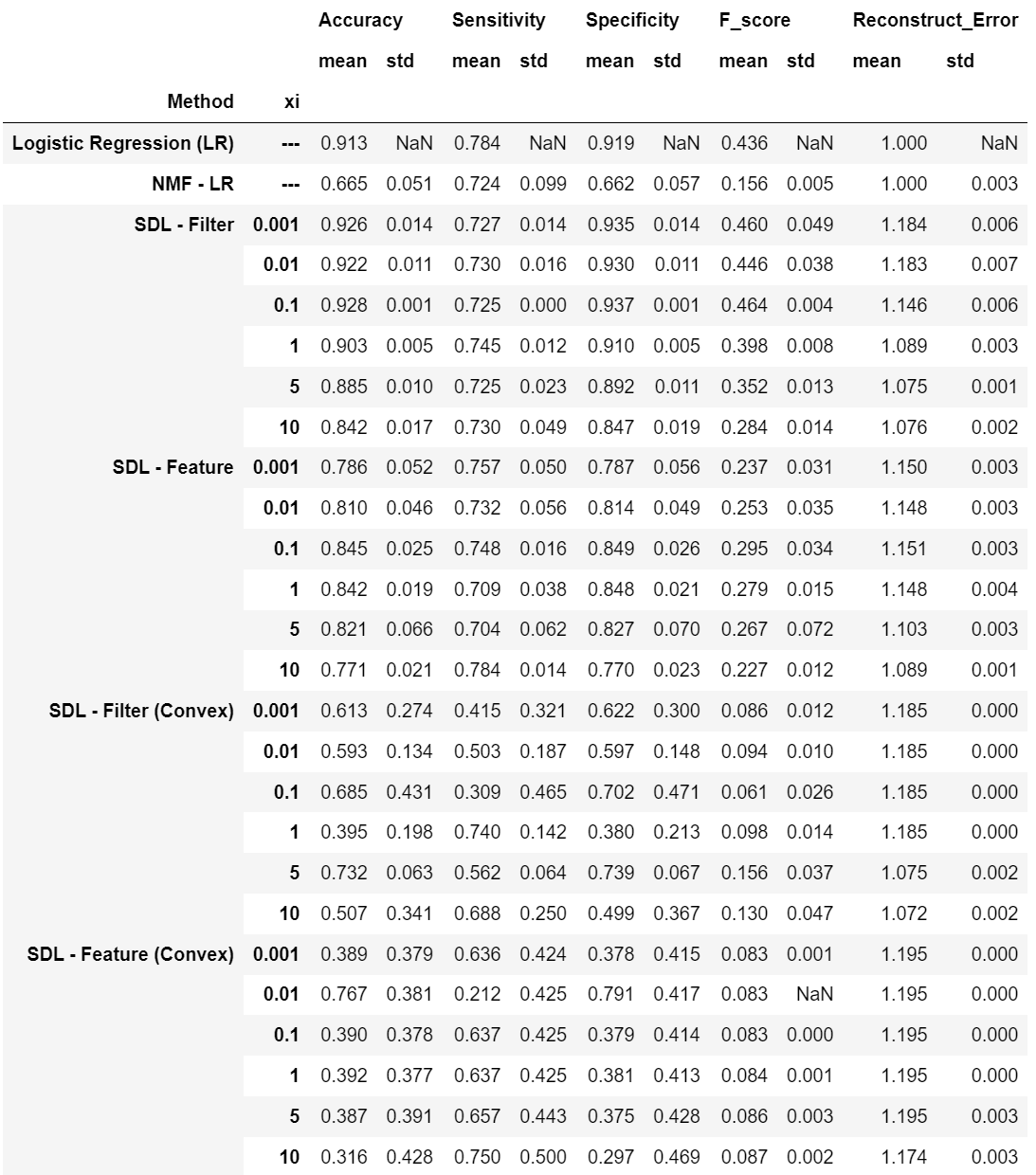} 
		\caption{ Tables of all results from each of the six methods for the fake job postings data in Section \ref{sec:fakejob_setup}. without the auxiliary covariates. Mean and standard deviations of five runs are reported. We report the reconstruction error of each method after normalizing it by the reconstruction error from NMF. }
		\label{table:fakejob_table1}
	\end{table}

	\begin{table}[h!]
		\centering
		\includegraphics[width=1\linewidth]{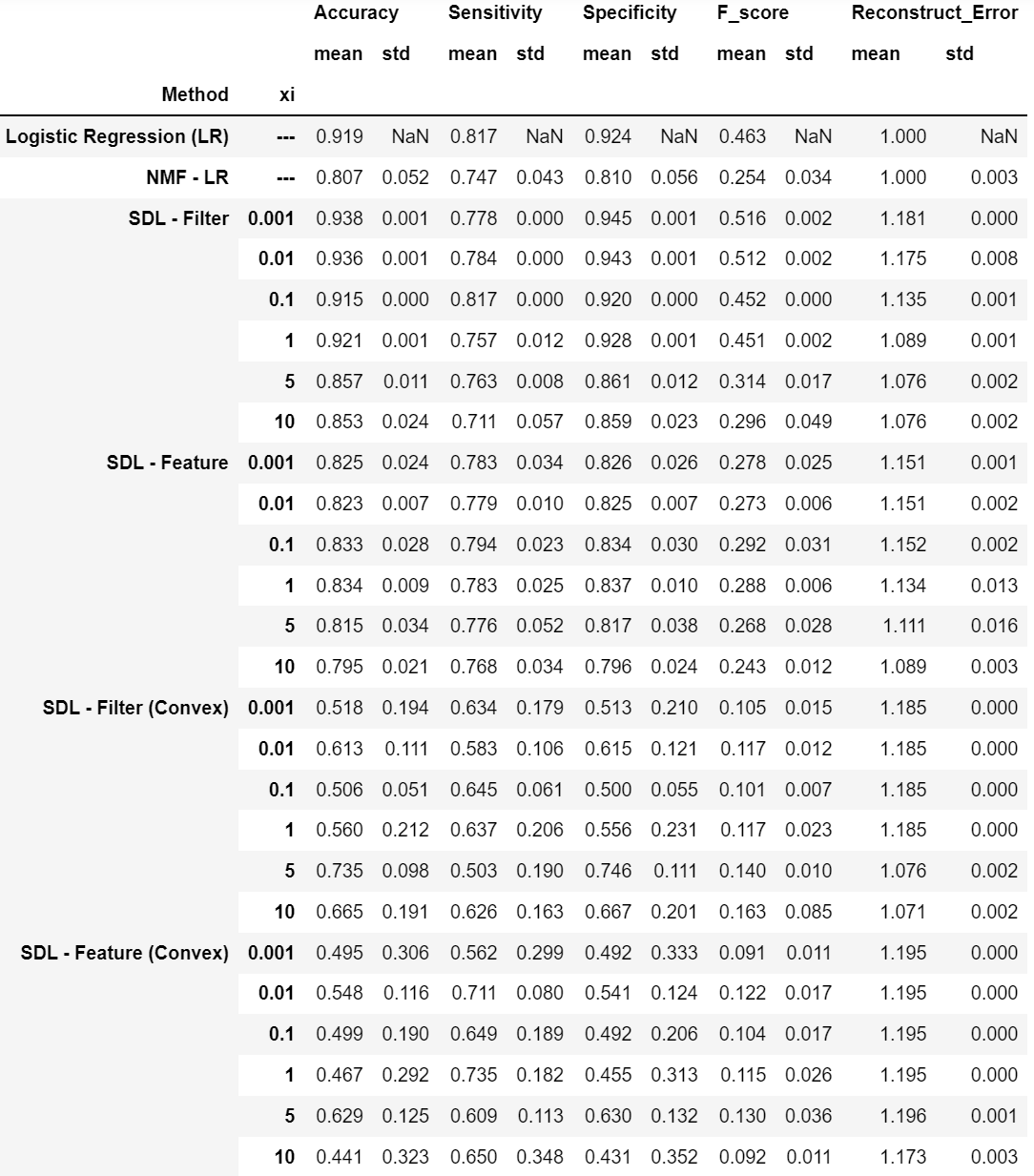} 
		\caption{ Tables of all results from each of the six methods for the fake job postings data in Section \ref{sec:fakejob_setup} with the 72 auxiliary covariates. Mean and standard deviations of five runs are reported. We report the reconstruction error of each method after normalizing it by the reconstruction error from NMF. }
		\label{table:fakejob_table2}
	\end{table}

\end{document}